\documentclass[10pt,reqno]{amsart}
\usepackage[utf8]{inputenc}  
\usepackage{graphicx}
\usepackage{pgfplots}
\usepackage{float}
\usepackage{bm}
\usepackage{stmaryrd}
\usepackage{subfig}
\usepackage{graphicx}
\usepackage{xcolor}
\usepackage{amsmath,amssymb, mathtools}
\usepackage{multicol}
\usepackage{hyperref}
\usepackage{cleveref}
\usepackage{dsfont}
\usepackage{ mathrsfs }

\usepackage[a4paper,top=3 cm,bottom=3 cm,left=2.5 cm,right=2.5 cm]{geometry}
\newtheorem{theorem}{Theorem}[section]
\newtheorem{proposition}{Proposition}[section]
\newtheorem{lemma}{Lemma}[section]
\newtheorem{corollary}{Corollary}[section]
\newtheorem{remark}{Remark}[section]
\newtheorem*{maintheorem*}{Main Theorem}
\allowdisplaybreaks
\numberwithin{equation}{section}

\def\R{\mathbb R}

\def\R{\mathbb R}

\def\to{\rightarrow}
\newcommand{\vecsigma}{\bm{\sigma}}
\DeclareMathOperator{\Imag}{Im}
\DeclareMathOperator*{\esssup}{ess\,sup}
\DeclareMathOperator{\loss}{loss}
\DeclareMathOperator{\erf}{erf}
\DeclareMathOperator{\GELU}{GELU}
\DeclareMathOperator{\argmin}{argmin}
\DeclareMathOperator{\ReLU}{ReLU}

\newcommand{\tnorm}[1]{\left\vert\!\left\vert\!\left\vert #1 \right\vert\!\right\vert\!\right\vert}

\def\R{\mathbb R}

\def\to{\rightarrow}

\usepackage{titletoc}

\makeatletter
\newcommand\thankssymb[1]{\textsuperscript{\@fnsymbol{#1}}}

\makeatother

\begin{document}

\title[Constructive Universal Approximation and Memorization by  Deep Networks]{Constructive Universal Approximation and Finite Sample Memorization by Narrow Deep ReLU Networks}

\author[M. Hern\'{a}ndez]{Mart\' in Hern\' andez\thankssymb{2}}
\email{martin.hernandez@fau.de}

\author[E. Zuazua]{Enrique Zuazua \thankssymb{1}\thankssymb{2}\thankssymb{3}}

\thanks{\thankssymb{2}
 Chair for Dynamics, Control, Machine Learning, and Numerics, Alexander von Humboldt-Professorship, Department of Mathematics,  Friedrich-Alexander-Universit\"at Erlangen-N\"urnberg,
91058 Erlangen, Germany.}

\thanks{\thankssymb{1} 
 Departamento de Matem\'{a}ticas,
Universidad Aut\'{o}noma de Madrid,
28049 Madrid, Spain.
}

\thanks{\thankssymb{3} 
Chair of Computational Mathematics, University of Deusto. Av. de las Universidades, 24,
48007 Bilbao, Basque Country, Spain.
}
\email{\texttt{enrique.zuazua@fau.de}}


\begin{abstract}
We present a fully constructive analysis of deep ReLU neural networks for classification and function approximation tasks. First, we prove that any dataset with $N$ distinct points in $\mathbb{R}^d$ and $M$ output classes can be exactly classified using a multilayer perceptron (MLP) of width $2$ and depth at most $2N + 4M - 1$, with all network parameters constructed explicitly. This result is sharp with respect to width and is interpreted through the lens of simultaneous or ensemble controllability in discrete nonlinear dynamics.

Second, we show that these explicit constructions yield uniform bounds on the parameter norms and, in particular, provide upper estimates for minimizers of standard regularized training loss functionals in supervised learning. As the regularization parameter vanishes, the trained networks converge to exact classifiers with bounded norm, explaining the effectiveness of overparameterized training in the small-regularization regime.

We also prove a universal approximation theorem in $L^p(\Omega; \R_+)$ for any bounded domain $\Omega \subset \mathbb{R}^d$ and $p \in [1, \infty)$, using MLPs of fixed width $d + 1$. The proof is constructive, geometrically motivated, and provides explicit estimates on the network depth when the target function belongs to the Sobolev space $W^{1,p}$. We also extend the approximation and depth estimation results to $L^p(\Omega; \R^m)$ for any $m \geq 1$. 

Our results offer a unified and interpretable framework connecting controllability, expressivity, and training dynamics in deep neural networks.

\end{abstract}

\maketitle
\setcounter{tocdepth}{1}

\tableofcontents

\section{Introduction and main results}
\subsection{Motivation and summary of the main results} Given a training dataset $\{x_i, y_i\}_{i=1}^N \subset \mathcal{X} \times \mathcal{Y}$, where each $x_i$ represents an input data point and $y_i$ its corresponding label or output, and a model $\phi(x,\theta)$ parameterized by $\theta$, the property of \emph{finite sample memorization}~\cite{yamasaki1993lower,yun2019small} holds if the model $\phi$ can correctly assign the label $y_i$ to each training instance $x_i$, i.e., 
\begin{align*}
   \phi(x_i,\theta) = y_i, \qquad \text{for every } i \in \{1, \dots, N\}.
\end{align*}

We analyze the finite sample memorization property when $\phi(\cdot, \theta)$ corresponds to the output of a neural network, with $\mathcal{X} = \mathbb{R}^d$ for $d \geq 1$, and $\mathcal{Y} = \mathbb{R}^m$ for $m \geq 1$.

When $\phi(x, \theta)$ is determined as the output of a continuous or discrete dynamical system, the problem can also be interpreted as an \emph{ensemble} or \emph{simultaneous controllability} problem, ensuring that the initial data $\{x_i\}$ are mapped simultaneously to the corresponding targets $\{y_i\}$~\cite{dom_zuauza_neuralode,MR3564127,MR1766422,MR3479654}.

This memorization property is particularly valuable for classification and interpolation tasks involving an unknown function $f: \mathcal{X} \to \mathcal{Y}$,  as it guarantees exact fitting at the data points. Once pointwise interpolation is achieved, one can extend this construction to approximate $f$ in $L^p$-norms for  $p \in [1, \infty)$, aligning with universal approximation theorems. These results establish that various neural network architectures are dense in functional spaces, such as $L^p$ or Sobolev spaces, thus enabling global approximation from finite data~\cite{Cybenko1989, devore2021neural, hardt2016identity, Hornik1989MultilayerFN, park2021provable, park2020minimum, vardi2021optimal, yun2019small}.

\begin{figure}
\centering
\includegraphics[width=0.6\textwidth]{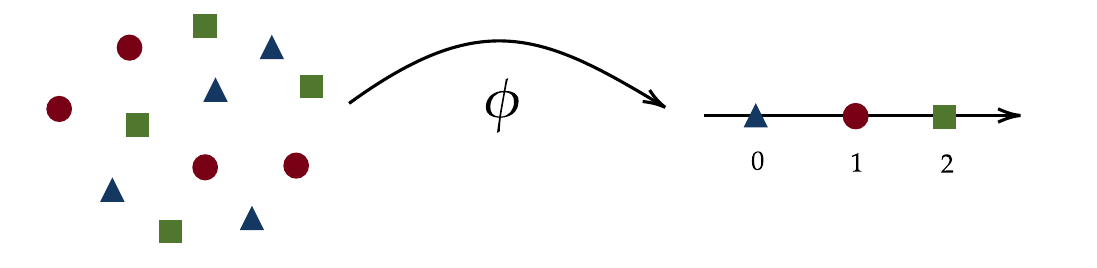}
\caption{Classification of a two-dimensional dataset via the neural network map $\phi$, with input dimension $d=2$ and and scalar output $m=1$.}
\label{fig:my_label}
\end{figure}

In this article, we present the following main results. 
\begin{itemize}
\item The first shows one that a ReLU multilayer perceptron with a width  $2$ and at most $2N + 4M - 1$ layers satisfies the finite sample memorization (or universal interpolation) property for $N$ input points and $M$ classes. Equivalently, the discrete dynamics generated by this MLP fulfils the property of simultaneous and/or ensemble controllable. This result is sharp, as memorization cannot be achieved by MLPs with width  $1$.

Our proof constructs the network parameters in a systematic manner, based on a geometric and dynamical interpretation of the neural network's architecture. Specifically, the parameters at each layer define hyperplanes that partition the state space into regions where the nonlinear activation function induces distinct dynamical behaviors. By strategically selecting these parameters and iteratively applying them across layers, we ensure that the network satisfies the memorization property (see Section~\ref{section_preliminars}). To the best of our knowledge, such a constructive and purely geometric interpretation of how narrow MLPs achieve finite sample memorization has not been previously explored in the literature (see Section~\ref{sec:related_work}).

\item We then examine the implications of our constructive approach in the context of supervised training with $\ell^2$-regularization. Although the networks we construct are not obtained through optimization (but rather through geometric and dynamical considerations), we demonstrate that the explicit interpolating parameters from our first result can be used to establish upper bounds on the optimal value of the regularized empirical loss. In particular, we prove that minimizing a standard training objective with $\ell^2$-regularization produces parameter norms that are uniformly bounded by those of our explicit construction. Moreover, in the vanishing regularization limit, the resulting parameters converge to a minimal-norm interpolating network. This result offers a theoretical explanation for the behavior of trained networks in the small-$\lambda$ regime and reinforces the idea that exact data fitting can be achieved without uncontrolled growth in parameter magnitude.

\item Our final contribution establishes a universal approximation theorem in $L^p(\Omega; \mathbb{R}_+)$, where $\Omega \subset \mathbb{R}^d$ is bounded and $p \in [1, \infty)$. This result is obtained using an MLP of fixed width $d + 1$. As in our first result, the network parameters are constructed explicitly, without relying on any optimization procedure. The proof is geometrically motivated and nonlinear, marking a departure from prior approaches to universal approximation with fixed-width networks~\cite{cai2022achieve, kidger2020universal, kim2024minimum, article2, park2020minimum}. Crucially, our explicit construction allows for quantitative estimates on the required depth when the target function belongs to $W^{1,p}(\Omega; \mathbb{R}_+)$. As a corollary, we extend this approximation result to the vector-valued settings $L^p(\Omega; \mathbb{R}_+^m)$ and $L^p(\Omega; \mathbb{R}^m)$ for any $m \geq 1$, with corresponding estimates on the necessary network width.

\end{itemize}

\subsection{Problem formulation}
Let $x\in\R$ and define the ReLU activation function as $\sigma(x)=\max\{0,x\}$. We consider a sequence of positive integers $\{d_j\}_{j=1}^L$. For each $j\in \{1,\dots, L\}$ and $x=(x^{(1)},\dots,x^{(d_j)})^{\top
}$ in $\R^{d_j}$, we introduce the vector-valued version of $\sigma$, defined by
\begin{align}\label{eq:vector_valuated_relu}
    \vecsigma_j:\R^{d_j}\to \R^{d_j},\qquad\vecsigma_{j}(x)=\left(\sigma(x^{(1)})\ ,\dots ,\sigma(x^{(d_j)})\right).
\end{align}
Given positive integers $L,\,d,\,N,\,M$, and the dataset $\lbrace x_i,y_i\rbrace_{i=1}^N\subset\R^{d}\times \{0,\dots,M-1\}$. We consider the following multilayer perceptron:
\begin{align}\label{discrete_dynamics}
\begin{cases}
x^{j}_i=\vecsigma_{j}(W_j x^{j-1}_i+b_j),\quad &\text{for } j\in \{1,\dots,L\},\\
x^0_i=x_i,
\end{cases}
\end{align}
where $i\in \{1,\dots,N\}$. In this context, $W_j\in \R^{d_{j}\times d_{j-1}}$ and $b_j\in\R^{d_{j}}$, for $j\in \{1,\dots,L\}$, represent the weight matrices and biases, respectively. Each $d_j$ determines the width of the $j$-layer, i.e., the dimension of the Euclidean space where the data reside at each iteration. The depth or the number of hidden layers of the neural network is represented by $L$, which is the number of iterations in the discrete dynamical system \eqref{discrete_dynamics},  referred to as a $L-$hidden layer neural network
 (see Figure \ref{fig:example_nn}). \begin{figure}[H]
\centering
\includegraphics[height=4.2cm]{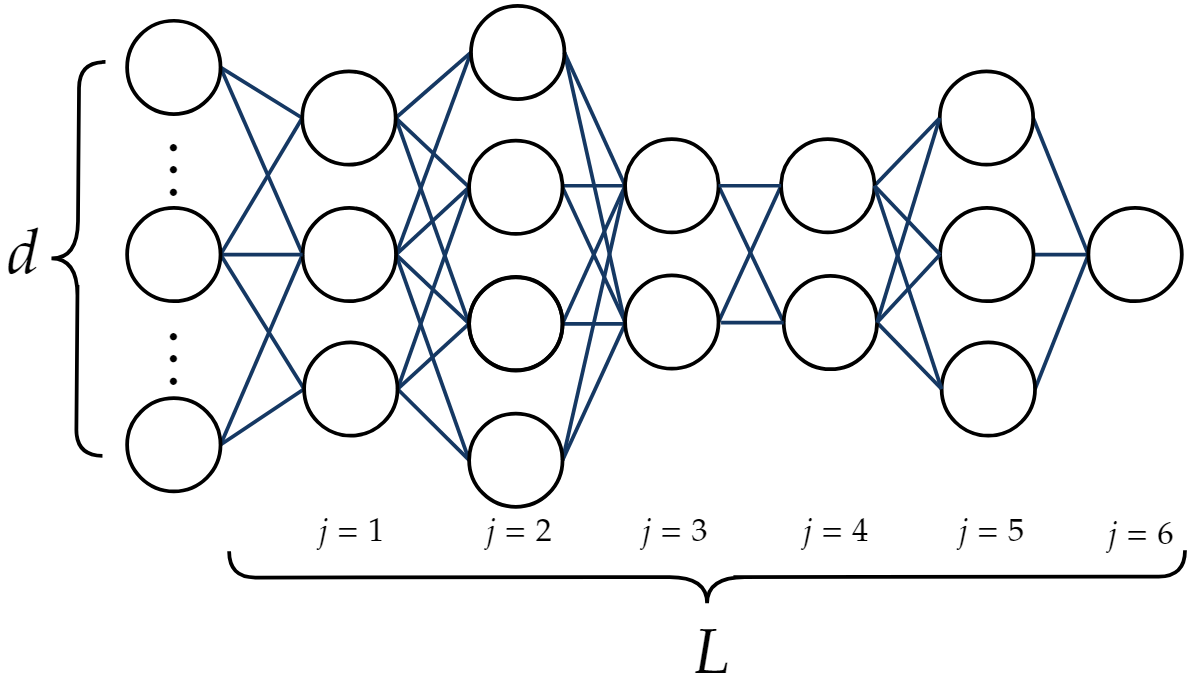}
\caption{Example of a deep neural network defined by the architecture \eqref{discrete_dynamics}. Here, $d$ indicates the dimension of the input data, while $j$ is the index of the layer, $L=6$ being the total number of layers, i.e., the depth of the neural network. In this particular example, we have $d_1=3$, $d_2=4$, $d_3=2$, etc. Moreover, the maximum width of the neural network is $4$, determined by the second layer.}
\label{fig:example_nn}
\end{figure}

In the following, we denote the sequences of weights and biases defining the neural network \eqref{discrete_dynamics} as $\mathcal{W}^L=\{W_j\}_{j=1}^L$ and $\mathcal{B}^L=\{b_j\}_{j=1}^L$, respectively. The \emph{width} of the neural network is defined as $w_{max}=\max_{j\in \{1,\dots,L\}}\{d_j\}$, i.e., the number of neurons in the widest layer. System \eqref{discrete_dynamics} is said to be a \emph{$w_{max}$-wide deep neural network}. The width and depth of a neural network are determined by its architecture and serve as an intrinsic measure of its complexity and approximation capacity.

\subsection{Statement of the main results and strategies of the proofs.}\label{sec:main_results}
\subsubsection{Finite Sample Memorization.}  Let us define the input-output map $\phi:\R^d\to\R$ of the neural network \eqref{discrete_dynamics} as
\begin{align*}
\phi^L(x_i):= \phi(\mathcal{W}^L,\mathcal{B}^L,x_i)=x^{L}_i, \qquad\text{for every }i\in\{1,\dots,N\},
\end{align*}
where $x^{L}$ represents the output of \eqref{discrete_dynamics}. 

Note that we are considering the particular case in which the output lies in $\R$, which means that $W_L\in \R^{1\times d_{L-1}}$.

With this definition, we present our first main theorem.

\begin{theorem}[Finite Sample Memorization]\label{multiclass_theorem}
Let the integers $d,\,N,\,M\geq 1$ and consider the dataset $\lbrace x_i,y_i\rbrace_{i=1}^N\subset \R^{d}\times\{0,\dots,M-1\}$. Assume that $x_i\neq x_j$ if $i\neq j$. Then, there exist parameters $\mathcal{W}^L$ and $\mathcal{B}^L$  with  width  $w_{max}=2$ and depth $L=2N+4M-1$  such that the input-output map of \eqref{discrete_dynamics} satisfies
\begin{align}\label{finite_sample_equality}
\phi(\mathcal{W}^L,\mathcal{B}^L,x_i)=y_i,\qquad\text{for every }i\in\{1,\dots,N\}.
\end{align}
Moreover, this result is sharp since the memorization property cannot be achieved with width $1$.
\end{theorem}

\begin{remark}\label{remark_theorem1} Some comments are in order.
\begin{itemize}
    \item Theorem \ref{multiclass_theorem} guarantees that there exists a $2-$wide deep neural network satisfying the finite sample memorization property, or equivalently, that system  \eqref{discrete_dynamics} is simultaneously or ensemble controllable.

\item Obviously, Theorem \ref{multiclass_theorem} also ensures finite sample memorization with $w_{max}\geq 2$.

\item Theorem \ref{multiclass_theorem} provides an estimate for the number of layers sufficient for the neural network to exhibit the finite sample memorization property. The depth of the network is directly related to both the number of data points $N$ and distinct labels $M$. However, it is independent of the dimension $d$, to which the data set belongs.  

\item  The neural network depth estimation is obtained from the constructive proof of \Cref{multiclass_theorem}, which is based on the worst-case scenario. This construction does not guarantee optimality in the estimated depth $L$, and for specific datasets, memorization could be achieved with fewer layers.

\item Although the width of the neural network is $2$, some layers have only one neuron, as in \Cref{fig:NN_DIAGRAM}. Namely, the total number of neurons and parameters in our neural network is $4N+6M+d-2$ and  $8N+12M+2d-4$, respectively. 

\item In Theorem \ref{multiclass_theorem}, we considered $\{y_i\}_{i=1}^N\subset\{0,\dots,M-1\}$ to simplify the exposition. However, the labels $\{0,\dots,M-1\}$ could be replaced by any other choice of $M$ distinct values in $\R_+$. This does not impact the width and depth of the neural network needed for memorization.
\end{itemize}
\end{remark}
\begin{figure}
\centering
\includegraphics[height=3.2cm]{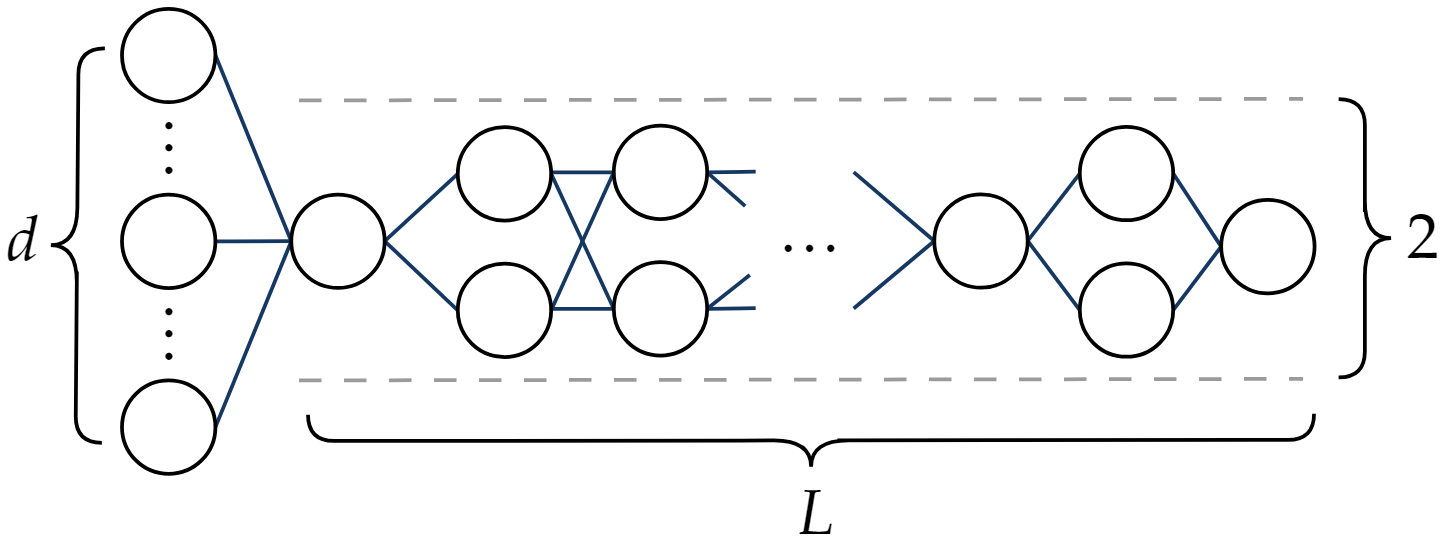}
\caption{
Deep neural network   of width  $w_{max}=2$ as in  Theorem \ref{multiclass_theorem}.}
\label{fig:NN_DIAGRAM}
\end{figure}

\emph{Strategy of proof of Theorem \ref{multiclass_theorem}.} The proof of this result is grounded in three fundamental tools, all of which are derived from the geometrical properties of the system \eqref{discrete_dynamics}:
\smallbreak
\begin{itemize}
\item {\bf Dimension Reduction:} Given a family of distinct points $\{x_k\}_{k=1}^N\subset\R^{d}$, $d\geq 1$, we can construct a projection $\phi^1:\R^{d}\to \R$ so that its images  are  all  different.
\smallbreak

\item { \bf Distance Scaling:} Let $w\in\R^d$ and $b\in\R$. For $x^0\in\R^d$, if $w\cdot x^0+b>0$, the value of $\sigma(w\cdot x^0+b)\in\R$ corresponds to $\|w\|d(x_0,H)$, where $\|w\|$ is the Euclidean norm of $w$ and $d(x_0,H)$ is the distance between $x_0$ and the hyperplane 
    \begin{align}\label{hyperplane_example}
        H:=\{ x\in\R^d\,:\,w\cdot x+b=0\}.
    \end{align}
\smallbreak

\item { \bf Collapse:} The hyperplane \eqref{hyperplane_example} divides the space into two half-spaces and the function $\sigma(w\cdot x+b)$ collapses the half-space  $w\cdot x+b\leq 0$ into the null point.
\medbreak
\end{itemize}

A more detailed discussion of these tools can be found in Section \ref{section_preliminars}. The map $\phi$ in Theorem \ref{multiclass_theorem} is built in four steps: 
\smallbreak
\begin{enumerate}
    \item \textbf{Preconditioning of the data:} Data is driven from the $d-$dimensional space to the one-dimensional one to reduce complexity.

\item \textbf{Compression process:} A recursive process is built to drive the $N$ data points to $M$ representative elements, according to their labels.

\item \textbf{Data sorting:} Data are mapped to ordered one-dimensional points according to their labels.

\item \textbf{Mapping to the respective labels:} Finally, each data point is mapped to its corresponding label.
\end{enumerate}

In each of these steps, we employ an input-output map of the neuronal network \eqref{discrete_dynamics}, utilizing at most two neurons per layer. As this is a purely constructive process, we can determine the number of layers required at each stage, and therefore, we can estimate the depth of the neural network. Further details on these key steps can be found in Section \ref{summary_proof}.
\medbreak
Let us consider the norms
\begin{align*}
    \tnorm{(\mathcal{W}^L,\mathcal{B}^L)}_2 := \left(\sum_{j=1}^L\|W_j\|_F^2 + \|b_j\|_2^2\right)^{1/2},
\end{align*}
and
\begin{align*}
    \tnorm{(\mathcal{W}^L,\mathcal{B}^L)}_\infty := \sup_{j\in\{1,\dots,L\}} \left\{ \|W_j\|_\infty, \|b_j\|_\infty \right\},
\end{align*}
where $\|\cdot\|_F$ denotes the Frobenius norm, $\|\cdot\|_2$ the $\ell^2$-norm, and $\|\cdot\|_\infty$ the $\ell^\infty$-norm. Then, due to the explicit construction of the parameters in \Cref{multiclass_theorem}, we obtain the following estimate for their norms.

\begin{corollary}\label{coro:estimation_norms}
Let $d,\,N,\,M > 1$ be integers, and consider the dataset $\{(x_i,y_i)\}_{i=1}^N \subset B_{R_x}^d(0)\times B_{R_y}^1(0)$, where $B_{R_x}^d(0) \subset \R^d$ and $B_{R_y}^1(0) \subset \R$ denote balls of radius $R_x > 0$ and $R_y > 0$ centered at the origin. Then, the parameters $\mathcal{W}^L$ and $\mathcal{B}^L$ provided by \Cref{multiclass_theorem} satisfy
\begin{align}
\tnorm{(\mathcal{W}^L,\mathcal{B}^L)}_2 &\leq C(1 + R_x\sqrt{N} + R_x N \sqrt{M} + R_y M), \\
\tnorm{(\mathcal{W}^L,\mathcal{B}^L)}_\infty &\leq C(R_x N + M + R_y),
\end{align}
where $C > 0$ is a constant independent of $M$, $N$, $R_x$, $R_y$, and $d$, but depends on $\min_{i \neq j \in \{1,\dots,N\}} \|x_i - x_j\|$.
\end{corollary}

\indent\emph{Strategy of the proof of Corollary~\ref{coro:estimation_norms}.} The result follows directly from the explicit construction given in \Cref{multiclass_theorem}. Since the parameters $\mathcal{W}^L$ and $\mathcal{B}^L$ are specified layer by layer, the norm bounds are obtained by computing their contributions at each step and applying standard estimates.

\begin{remark}
    As mentioned in Remark~\ref{remark_theorem1}, Theorem~\ref{multiclass_theorem} ensures exact classification for any $w_{\max} \geq 2$ by zero-padding the parameters. Consequently, Corollary~\ref{coro:estimation_norms} also holds for all $w_{\max} \geq 2$.
\end{remark}

\begin{remark}[Other activation functions]\label{remark:other_af}
Our techniques can also be applied to other activation functions $\sigma$ that satisfy the following conditions:
\smallbreak
\begin{itemize}
    \item  $\sigma$ is monotonically non-decreasing on $\R_+$, being strictly monotonic in a subinterval $T$ of $\R_+$. This permits scaling distances between different points.

\item There exists and open subset $S$ of $\R_-$ in which $\sigma$ vanishes. This allows the collapse of different points to merge them according to their labels.

\end{itemize}
The essential features of the activation function employed are described below in \Cref{summary_proof} (see Figure \ref{fig:general-activation}).
\begin{figure}[H]
\centering
\includegraphics[height=3.2cm]{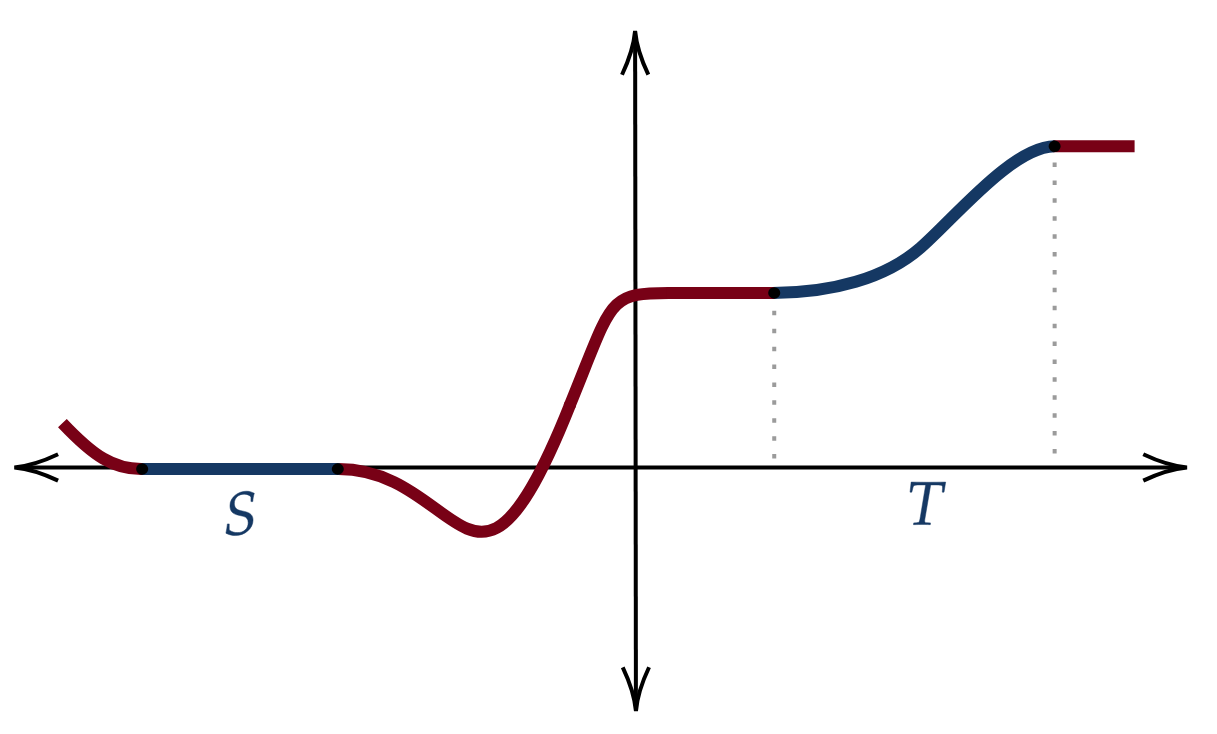}
\caption{
Activation functions for which the results of this paper can be generalized.}
\label{fig:general-activation}
\end{figure}

Note, however, that despite this generalization being possible, its practical interest is limited. The weights and biases will be more difficult to construct and will generally have larger norms compared to those provided by Corollary \ref{coro:estimation_norms}, which uses the ReLU activation function to ensure minimal complexity.

\end{remark}

As a consequence of Theorem~\ref{multiclass_theorem}, we can address the case of $m$-dimensional labels for $m \geq 1$, that is, when $\{y_i\}_{i=1}^N$ is a subset of $M$ distinct points in $\mathbb{R}^m$. This leads to the following corollary.

\begin{corollary}[Finite Sample Memorization for $m-$dimensional labels]\label{corollary_multi_classification}
Let the integers $d,\,N,\,M,\, m\geq 1$ and a dataset $\lbrace x_i,y_i\rbrace_{i=1}^N$ so that $\lbrace x_i\rbrace_{i=1}^N\subset \R^{d}$, $x_i\neq x_j$ if $i\neq j$, and $\lbrace y_i\rbrace_{i=1}^N\subset \lbrace\ell_k\rbrace_{k=0}^{M-1}$ with $ \ell_k \in \R^{m}_{+}$. 

Then, there exist parameters $\mathcal{W}^L$ and $\mathcal{B}^L$ with $L=2N+4M-1$ and $w_{max}=2m$, such that the input-output map of \eqref{discrete_dynamics} satisfies
\begin{align*}
\phi(\mathcal{W}^L,\mathcal{B}^L,x_i)=y_i,\qquad\text{for every }i\in
\{1,\dots,N\}.
\end{align*}
\end{corollary}
\indent\emph{Strategy of proof of Corollary \ref{corollary_multi_classification}.}
Corollary~\ref{corollary_multi_classification} follows directly from Theorem~\ref{multiclass_theorem} by considering $m$ independent neural networks, each constructed as in Theorem~\ref{multiclass_theorem} to handle one component of the $m$-dimensional labels, and combining them in parallel to obtain a single network with vector-valued outputs (see Figure \ref{fig:figure_net_d_y}).
\medbreak
\begin{figure}[H]
\centering
\includegraphics[height=5.7cm]{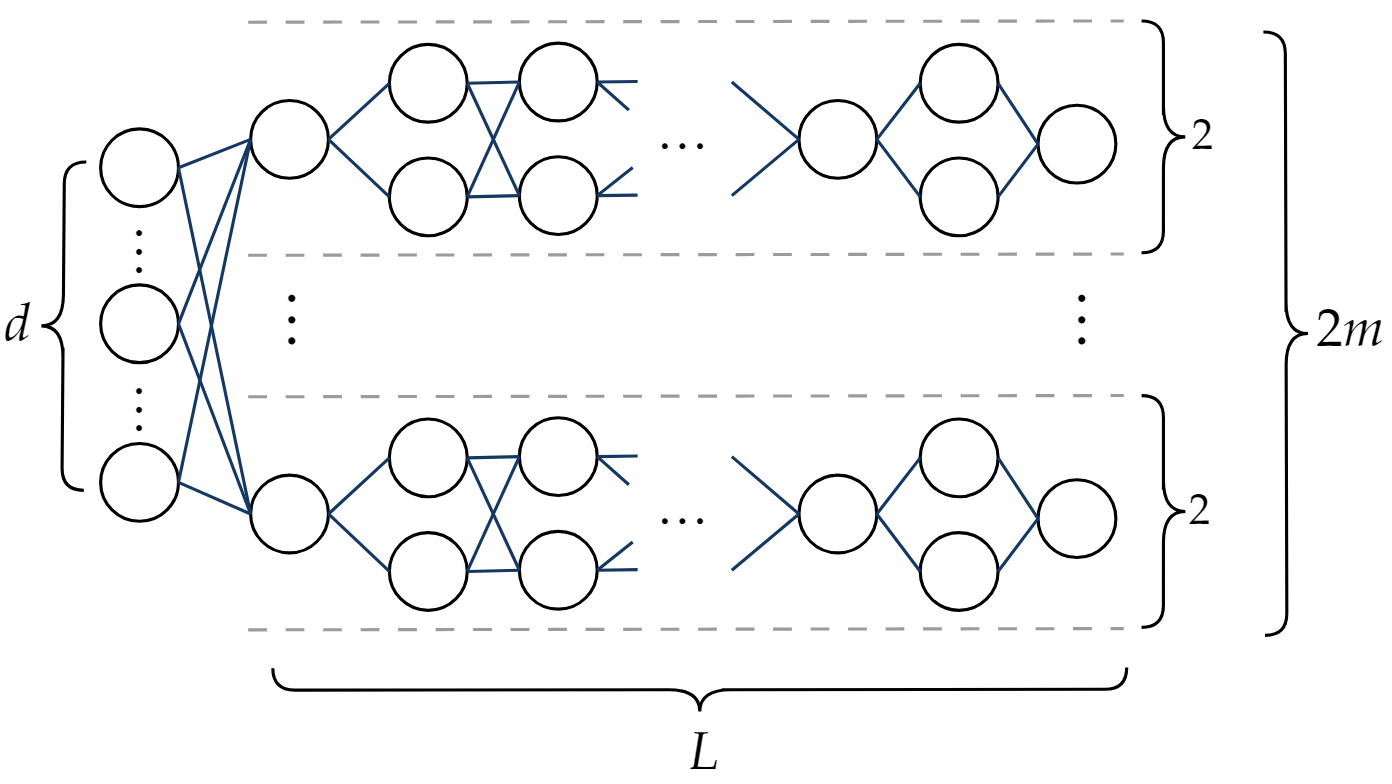}
\caption{Deep neural network of width $w_{max}=2m$ as  in Corollary \ref{corollary_multi_classification}.}
\label{fig:figure_net_d_y}
\end{figure}

\medbreak

In the previous results, all the labels were assumed to be non-negative. However, it is reasonable to consider the case where the signs on the labels may vary. Let $\Imag(\phi^{L})$ denote the image of the input-output map defined by \eqref{discrete_dynamics}. Since $\Imag(\phi^L)\subset\R_+$,  data cannot be mapped to negative labels by the architecture defined in \eqref{discrete_dynamics}. However, we can consider the following MLP 
\begin{align}\label{discrete_dynamics_2}
\begin{cases}
x^{j}_i=A_j\vecsigma_{j}(W_j x^{j-1}_i+b_j),\quad &\text{for } j\in \{1,\dots,L\},\\
x^0_i=x_i,
\end{cases}
\end{align}
where the extra parameters $A_j\in \R^{d_j\times d_j}$ for $ j\in \{1,\dots,L\}$ allow the neural network, in particular, to map data to negative values. We denote by $\mathcal{A}^L=\{A_j\}_{j=1}^L$ this new sequence of parameters.
\smallbreak

The following corollary states that the neural network architecture \eqref{discrete_dynamics_2} satisfies the finite sample memorization for labels in $\R$. 

\begin{corollary}[Finite Sample Memorization for real labels]\label{corollary_clasification_R}
Consider the integers $d,\,N,\,M\geq 1$ and a dataset $\lbrace x_i,y_i\rbrace_{i=1}^N\subset \R^{d}\times\{\alpha_0,\dots,\alpha_{M-1}\}$ with $\{\alpha_k\}_{k=0}^{M-1}\subset\R$. Assume that $x_i\neq x_j$ if $i\neq j$. Then, for $L=2N+4M$ and $w_{max}=2$, there exist parameters $\mathcal{A}^L$, $\mathcal{W}^L$ and $\mathcal{B}^L$ such that the input-output map of \eqref{discrete_dynamics_2} satisfies
\begin{align}\label{finite_sample_equality_2}
\phi(\mathcal{A}^L,\mathcal{W}^L,\mathcal{B}^L,x_i)=y_i,\qquad\text{for every }i\in\{1,\dots,N\}.
\end{align}
\end{corollary}

\indent\emph{Strategy of Proof for Corollary \ref{corollary_clasification_R}.} The proof hinges on two key observations: first, that the architecture \eqref{discrete_dynamics_2} can drive data to negative labels; and second, that this architecture coincides with \eqref{discrete_dynamics} when \(A_j = \text{Id}_{d_j}\) (the identity matrix in $\mathbb{R}^{d_j \times d_j}$). Initially, we define a set of auxiliary positive labels. Then, taking \(A_j = \text{Id}_{d_j}\), we apply Theorem \ref{multiclass_theorem} to map the data points to these auxiliary labels by using $2N+4M-1$ layers. Finally, by using a particular matrix \(A\), we construct a $2$-wide, one-layer neural network that maps the auxiliary labels to the original ones. The proof concludes by composing these two neural networks, obtaining a 2-wide neural network with $2N+4M$ layers. 

\begin{remark}
Following the same approach used in the proof of Corollary~\ref{corollary_multi_classification}, we can extend Corollary~\ref{corollary_clasification_R} to establish the finite sample memorization property for labels in $\mathbb{R}^m$. In this case, the construction yields a neural network of width $2m$ and depth $2N + 4M$.
\end{remark}

\subsubsection{Implications for Neural Network Training}

The training of neural networks is typically formulated as the minimization of a regularized empirical risk functional. Given a dataset $\{x_i,y_i\}_{i=1}^N\subset\R^d\times \R^m$, the standard training objective reads
\begin{align}\label{eq:training_problem}
    \min_{(\mathcal{W}^L, \mathcal{B}^L)}\left\{\mathcal{J}_{\lambda}(\mathcal{W}^L,\mathcal{B}^L)= \lambda\tnorm{(\mathcal{W}^L,\mathcal{B}^L)}^2_2+\frac{1}{N}\sum_{i=1}^N\loss\left(\phi(\mathcal{W}^L,\mathcal{B}^L,x_i), y_i\right)\right\},
\end{align}
where the first term penalizes large weights with a parameter $\lambda>0$, and the second term measures the prediction error on the training dataset through the input-output map $\phi(\mathcal{W}^L,\mathcal{B}^L,\cdot)$ defined in \eqref{discrete_dynamics}. Here, $\loss(\cdot, \cdot)$ is a given continuous and nonnegative function such that $\loss(x,x)=0$. For instance, $\loss(x, y) := \|x-y\|^2_{p}$ with $p\in\{1, 2\}$ is commonly used for regression tasks, while $\loss(x,y)= \log(1+e^{x}) - yx$, with $y\in\{0,1\}$ and $x\in \R$, is used for binary classification (binary logistic loss). The regularization parameter $\lambda$ controls the trade-off between data fidelity and model complexity. 

Note that for every $\lambda > 0$, the functional $\mathcal{J}_\lambda$ is coercive. Moreover, since both $\loss$ and the activation functions $\vecsigma_j$ are continuous, the Bolzano–Weierstrass theorem ensures the existence of optimal parameters for every $\lambda > 0$.

Recall that Theorem \ref{multiclass_theorem} establishes the existence of $(\mathcal{W}_*^L, \mathcal{B}_*^L)$ such that exact interpolation of arbitrary finite datasets is possible. In particular, this guarantees that
\begin{align}\label{eq:functional_zero}
   \mathcal{J}_{0}(\mathcal{W}_*^L, \mathcal{B}_*^L) = 0,
\end{align}
and since $\mathcal{J}_{0} \geq 0$, we deduce that $(\mathcal{W}_*^L, \mathcal{B}_*^L)$ is a minimizer of $ \mathcal{J}_{0}$. Moreover, we have the following result:

\begin{theorem}\label{th:estima_optimal _control_Jlambda}
Let $d,\,N,\,m \in \mathbb{N}$ with $d,\,N,\,m > 1$, and let $\{(x_i,y_i)\}_{i=1}^N \subset \mathbb{R}^d \times \mathbb{R}^m$. Let $(\mathcal{W}_\lambda^L, \mathcal{B}_\lambda^L)$ be a minimizer of $\mathcal{J}_\lambda$, and let $(\mathcal{W}_*^L,\mathcal{B}_*^L)$ be a minimizer of $\mathcal{J}_0$ for the ReLU activation function. Then, we have
\begin{align}\label{eq:estimation_functional_1}
    \mathcal{J}_{\lambda}(\mathcal{W}^L_\lambda,\mathcal{B}^L_\lambda) \leq \lambda \tnorm{(\mathcal{W}_*^L,\mathcal{B}_*^L)}_2^2, \quad \text{for all } \lambda > 0.
\end{align}
In particular, we have
\begin{align}\label{eq:limit_lambda_to_zero}
    \frac{1}{N} \sum_{i=1}^N \loss\left(\phi(\mathcal{W}^L_\lambda,\mathcal{B}^L_\lambda,x_i), y_i\right) \to 0 \quad \text{as } \lambda \to 0,
\end{align}
and
\begin{align}\label{eq:uniform_estimatin_control_lambda}
    \tnorm{(\mathcal{W}_\lambda^L,\mathcal{B}_\lambda^L)}_2 \leq \tnorm{(\mathcal{W}_*^L,\mathcal{B}_*^L)}_2, \quad \text{for all } \lambda > 0.
\end{align}
Moreover, the family of parameters $\{(\mathcal{W}_\lambda^L,\mathcal{B}_\lambda^L)\}_{\lambda>0}$ admits a subsequence that converges to a minimal-norm minimizer of $\mathcal{J}_0$.
\end{theorem}

\begin{remark}
Several remarks are in order:
\begin{itemize} 
\item  The proof of Theorem \ref{th:estima_optimal _control_Jlambda} can be found in \Cref{sec:proo_training}. It relies on the fact that a minimizer of $\mathcal{J}_0$ is known explicitly, which is possible due to the interpolation result from Theorem~\ref{multiclass_theorem}. However, Theorem \ref{th:estima_optimal _control_Jlambda} is not tight to the 2-width MLP used in Theorem~\ref{multiclass_theorem}; it can be applied to any architecture (or discrete dynamical system) whose input-output map is continuous with respect to the parameters. 

    \item This result shows that, as the regularization parameter $\lambda$ tends to zero, the trained network tends to interpolate the dataset while remaining uniformly bounded in parameter norm. In particular, it guarantees that any accumulation point of the sequence of minimizers corresponds to a  minimum-norm interpolant. Indeed, this holds given that the arguments of Theorem \ref{th:estima_optimal _control_Jlambda} apply to any minimizer of  $\mathcal{J}_0$, and in particular to those of minimal norm.This has practical implications in training: it justifies the use of small values of $\lambda$  for training since this pushes the model to fit the training data exactly, with finite values of the parameters, close to the optimal ones. 

    \item In general, we cannot guarantee that $(\mathcal{W}_0^L,\mathcal{B}_0^L) = (\mathcal{W}_*^L,\mathcal{B}_*^L)$. This is a consequence of the lack of convexity of the functionals $\mathcal{J}_\lambda$ for $\lambda \geq 0$.

    \item Observe that we can replace the norm $\tnorm{\cdot}$ in \eqref{eq:training_problem} by $\tnorm{\cdot}_\infty$ and derive \Cref{th:estima_optimal _control_Jlambda} again, but now involving the $\ell^\infty$-norm.

    \item Due to Corollary~\ref{coro:estimation_norms}, when we consider the dataset $\{(x_i,y_i)\}_{i=1}^N \subset B_{R_x}^d(0)\times B_{R_y}^m(0)$, we can provide an explicit estimate for the optimal value of $\mathcal{J}_\lambda$ in terms of $N$, $M$, $R_x$, and $R_y$. Namely, in the case $m = 1$, combining Corollary~\ref{coro:estimation_norms} with Theorem~\ref{th:estima_optimal _control_Jlambda}, we obtain
\begin{align}\label{eq:estimatio_J_lambda_paramters_explicit}
    \mathcal{J}_{\lambda}(\mathcal{W}^L_\lambda,\mathcal{B}^L_\lambda)\leq \lambda C(1 + R_x^2N + R_x^2 N^2 M + R_y^2 M^2).
\end{align}
However, since the parameters provided in Theorem~\ref{multiclass_theorem} are constructed to handle worst-case scenarios, the bound in \eqref{eq:estimatio_J_lambda_paramters_explicit} may be conservative. In practical situations, significantly tighter estimates may be achieved.
\end{itemize}
\end{remark}

\begin{remark}[Practical application of Theorem \ref{th:estima_optimal _control_Jlambda}]
  These results, and in particular  \eqref{eq:estimatio_J_lambda_paramters_explicit}, facilitate defining  stopping criteria for the training process. Indeed, the optimal value of $\mathcal{J}_{\lambda}$ is always bounded above by $\lambda\tnorm{(\mathcal{W}^L_*,\mathcal{B}^L_*)}_2^2$. Therefore, for any fixed $\lambda>0$ -- even for small values of $\lambda$ -- the optimizer can be designed to terminate once the value of $\mathcal{J}_{\lambda}$ falls close to this explicit threshold. \end{remark}

In practice, neural networks are frequently trained using activation functions that do not satisfy the properties  in Remark~\ref{remark:other_af}. Then, constructing or estimating the  parameters assuring classification becomes challenging.

This motivates the following question: Can the knowledge of parameter norms that ensure exact classification for a given activation function, for instance the ReLU,  still provide useful information to the regularized functional for a perturbed activation function? 

To analyze this point, we consider the following neural network 
\begin{align}\label{eq:discrete_dynamics_different_activation}
\begin{cases}
\hat{x}^{j}_i=\hat{\vecsigma}_{j}( W_j \hat{x}^{j-1}_i+ b_j),\quad &\text{for } j\in \{1,\dots,L\},\\
\hat{x}^0_i=x_i,
\end{cases}
\end{align}
where each $\hat{\vecsigma}_{j}:\R^{ d_j}\to \R^{ d_j}$ is a continuous activation function, not necessarily the ReLU. Denote by $\hat{\phi}({\mathcal{W}}^L,{\mathcal{B}}^L,\cdot)$ the input-output map associated with \eqref{eq:discrete_dynamics_different_activation}. We then train this new model:
\begin{align}\label{eq:training_problem2}
    \min_{({\mathcal{W}}^L,{\mathcal{B}}^L)}\left\{\hat{\mathcal{J}}_{\lambda}({\mathcal{W}}^L,{\mathcal{B}}^L)= \lambda\tnorm{({\mathcal{W}}^L,{\mathcal{B}}^L)}^2_2+\frac{1}{N}\sum_{i=1}^N\loss\left(\hat\phi({\mathcal{W}}^L,{\mathcal{B}}^L,x_i), y_i\right)\right\},
\end{align}
where $\loss(\cdot, \cdot)$ is a continuous loss function, as before. Thanks to the continuity of $\loss$ and $\hat\vecsigma_j$, and the coercivity of $\hat{\mathcal{J}}_\lambda$, we can guarantee the existence of a minimizer $(\hat{\mathcal{W}}^L_\lambda,\hat{\mathcal{B}}^L_\lambda)$ for every $\lambda>0$. Moreover, the following result holds.

\begin{theorem}\label{th:estimation_functional_dif_activation}
Let $d,\,N,\,m \in \mathbb{N}$ with $d,\,N,\,m > 1$, and let $\{(x_i,y_i)\}_{i=1}^N \subset \mathbb{R}^d \times \mathbb{R}^m$. Let $(\hat{\mathcal{W}}_\lambda^L, \hat{\mathcal{B}}_\lambda^L)$ be a minimizer of $\hat{\mathcal{J}}_\lambda$, and let $(\mathcal{W}_*^L,\mathcal{B}_*^L)$ be a minimizer of $\mathcal{J}_0$ for the ReLU. Define $R_0 := \max_{i \in \{1,\dots,N\}} \|x_i\|$, and assume that
\begin{align}\label{eq:difinition_radios_error}
    \nu_j := \sup_{z \in B^{d_j}_{R_j}(0)} \|\hat{\vecsigma}_j(z) - \vecsigma_j(z)\| < \infty, \quad \text{with} \quad R_j := \|W_j\|_2 R_{j-1} + \|b_j\|_2, \quad \text{for } j \in \{1,\dots,L\},
\end{align}
$R_j$ being defined through the optimal parameters of the ReLU model whose values we estimated.
Set $\nu := (\nu_1, \dots, \nu_L)$. 

Then, for every $\lambda > 0$, we have
\begin{align*}
    \hat{\mathcal{J}}_{\lambda}(\hat{\mathcal{W}}^L_\lambda, \hat{\mathcal{B}}^L_\lambda) \leq \lambda \tnorm{(\mathcal{W}_*^L, \mathcal{B}_*^L)}_2^2 + \mathscr{A}_{\loss}(\nu, \mathcal{W}_*^L, \mathcal{B}_*^L),
\end{align*}
where $\mathscr{A}_{\loss}$ is a nonnegative function depending on the loss function $\loss$, and satisfies $\mathscr{A}_{\loss}(\nu, \cdot, \cdot) \to 0$ as $\|\nu\|_2 \to 0$.
\end{theorem}
The proof of Theorem \ref{th:estimation_functional_dif_activation} can be found in \Cref{sec:proo_training}.

\begin{remark}
Several observations are worth:
\begin{itemize}
    \item The proof consists on analyzing the deviation between the input-output maps for both activation functions, using arguments similar to those used in Theorem~\ref{th:estima_optimal _control_Jlambda}. The radii $R_j$ introduced in \eqref{eq:difinition_radios_error} are key to estimating the deviation between activations at each layer. 
    \item Observe that the only assumption imposed on $\hat{\vecsigma}_j$ is continuity. No further structural properties are required.

    \item We have assumed that both architectures are identical, i.e.,   data evolve through layers of the same dimensions. This assumption is not strictly necessary. When the dynamics \eqref{eq:discrete_dynamics_different_activation} evolve along a different architecture,   with $\hat d_j \neq d_j$, one can always extend the networks by zero-padding to the maximal dimension $\max\{d_j,\hat d_j\}$ and proceed to a stability analysis. A similar estimate then follows, now involving a modified term $\hat{\mathscr{A}}_{\loss}$, depending on the padded architecture.

    \item In the particular case $\loss(x,y) = \|x - y\|^2_2$, we obtain the explicit estimate
    \begin{align*}
        \mathscr{A}_{\loss}(\nu,\mathcal{W}_*^L,\mathcal{B}_*^L) = \displaystyle\begin{cases}
            \displaystyle2\|\nu\|_2^2\left(\frac{\tnorm{(\mathcal{W}_*^L,\mathcal{B}_*^L)}_2^{2L} - 1}{\tnorm{(\mathcal{W}_*^L,\mathcal{B}_*^L)}_2^2 - 1}\right) &\text{if } \tnorm{(\mathcal{W}_*^L,\mathcal{B}_*^L)}_2^2 \neq 1,\\
            2\|\nu\|_2^2 L &\text{if } \tnorm{(\mathcal{W}_*^L,\mathcal{B}_*^L)}_2^2 = 1.
        \end{cases}
    \end{align*}
\end{itemize}
\end{remark}

The previous result is particularly relevant when, instead of considering a fixed activation function $\hat{\vecsigma}$, for instance the ReLU, we study a family of continuous activation functions $\hat{\vecsigma}_\varepsilon$ parametrized by $\varepsilon > 0$, for instance a regularization of the ReLU. In this case, we can again formulate a training problem, now depending on the parameter $\varepsilon > 0$. Namely, we consider
\begin{align}\label{eq:training_problem3}
    \min_{({\mathcal{W}}^L,{\mathcal{B}}^L)}\left\{ {\mathcal{J}}_{\lambda,\varepsilon}({\mathcal{W}}^L,{\mathcal{B}}^L) = \lambda\tnorm{({\mathcal{W}}^L,{\mathcal{B}}^L)}^2_2 + \frac{1}{N} \sum_{i=1}^N \loss\left(\phi_\varepsilon({\mathcal{W}}^L,{\mathcal{B}}^L,x_i), y_i\right) \right\},
\end{align}
where $\phi_\varepsilon$ denotes the input-output map of the dynamics \eqref{eq:discrete_dynamics_different_activation} when $\hat{\vecsigma}$ is replaced by $\hat{\vecsigma}_\varepsilon$.

By similar arguments, one can ensure the existence of a minimizer of \eqref{eq:training_problem3} for every $\varepsilon,\,\lambda > 0$. We denote by $(\hat{\mathcal{W}}^L_{\lambda,\varepsilon},\hat{\mathcal{B}}^L_{\lambda,\varepsilon})$ the corresponding minimizer. 

A particularly interesting case arises when we define the family $\hat\sigma_\varepsilon:\R \to \R$ by
\begin{align}\label{eq:ineterpol_gelu}
    \hat\sigma_\varepsilon(x) = \frac{x}{2} \left(1 + \erf\left(\frac{x}{\varepsilon\sqrt{2}}\right)\right),
\end{align}
so that $\hat\sigma_1(x) = \GELU(x)$, the so-called Gaussian Error Linear Unit (see \Cref{fig:interpolation_relu_gelu}). Let $\hat\vecsigma_\varepsilon:\R^{d_j} \to \R^{d_j}$ denote the vector-valued extension of $\hat\sigma_\varepsilon$ (as in \eqref{eq:vector_valuated_relu}). Then, we obtain the following corollary.
\begin{figure}
    \centering
    \includegraphics[width=0.5\linewidth]{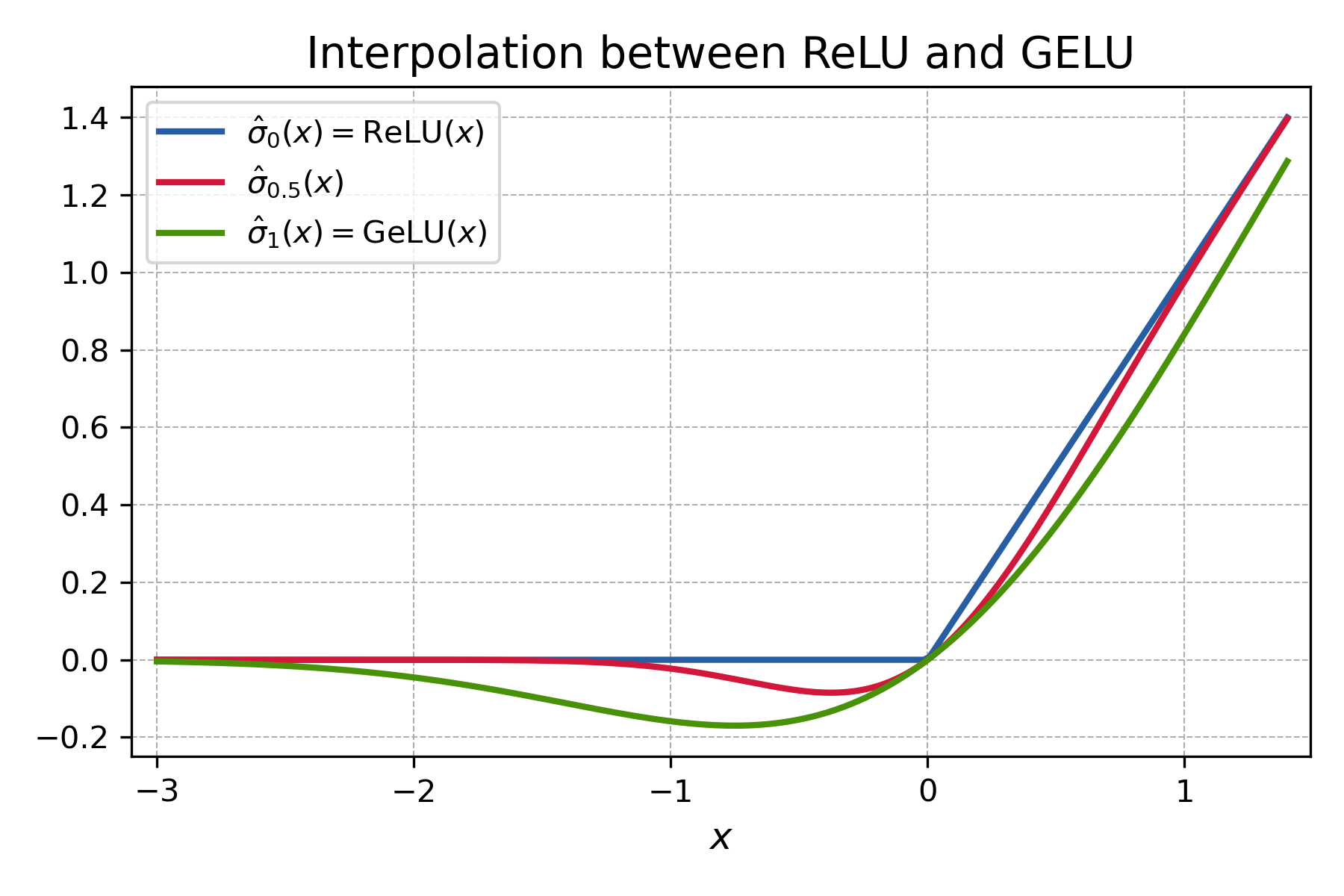}
    \caption{Illustration of the interpolation between ReLU and GELU via $\hat\sigma_\varepsilon(x)$ for different values of $\varepsilon$. The case $\varepsilon = 0$ corresponds to ReLU, and $\varepsilon = 1$ to GELU.}
    \label{fig:interpolation_relu_gelu}
\end{figure}

\begin{corollary}\label{coro:convergence_gelu_relu}
Let $d,\,N,\,m \in \mathbb{N}$ with $d,\,N,\,m > 1$, and let $\{(x_i,y_i)\}_{i=1}^N \subset \mathbb{R}^d \times \mathbb{R}^m$. Let $({\mathcal{W}}_{\lambda,\varepsilon}^L, {\mathcal{B}}_{\lambda,\varepsilon}^L)$ be a minimizer of ${\mathcal{J}}_{\lambda,\varepsilon}$, and let $(\mathcal{W}_*^L, \mathcal{B}_*^L)$ be a minimizer of $\mathcal{J}_0$. Then, under the assumptions of Theorem~\ref{th:estimation_functional_dif_activation}, we have
\begin{align*}
\limsup_{\varepsilon \to 0} \mathcal{J}_{\lambda,\varepsilon}({\mathcal{W}}^L_{\lambda,\varepsilon}, {\mathcal{B}}^L_{\lambda,\varepsilon}) \leq \lambda \tnorm{(\mathcal{W}_*^L, \mathcal{B}_*^L)}_2^2,
\end{align*}
for every $\lambda > 0$.
\end{corollary}

\begin{remark}
Corollary~\ref{coro:convergence_gelu_relu} follows directly from Theorem~\ref{th:estimation_functional_dif_activation}, observing that $\sigma_\varepsilon \to \sigma$ in $C(\R)$ as $\varepsilon \to 0$, where $\sigma_\varepsilon$ is defined in \eqref{eq:ineterpol_gelu}, and $\sigma$ denotes the ReLU activation function. The proof is provided in \Cref{sec:proo_training}.
\end{remark}

\subsubsection{Universal Approximation Theorem.}
We now analyze the property of universal approximation.

\begin{theorem}[Universal Approximation Theorem for ${L^{p}(\Omega;\R_{+})}$]\label{UAT_LP}
Let be $1\leq p< \infty$, $d\geq1$ an integer, and $\Omega\subset\R^d$ a bounded domain. For any $f\in L^p(\Omega;\R_+)$ and $\varepsilon>0$, there exist a depth $\mathcal{L}=\mathcal{L}(\varepsilon) \geq 1$ and parameters $\mathcal{W}^\mathcal{L}$ and $\mathcal{B}^\mathcal{L}$ such that the input-output map of \eqref{discrete_dynamics} with $w_{max}=d+1$ satisfies
\begin{align} \label{error_f_phi_intro}
\|\phi(\mathcal{W}^{\mathcal{L}},\mathcal{B}^{\mathcal{L}},\cdot)-f(\cdot)\|_{L^p(\Omega;\R_+)}<\varepsilon.
\end{align}
Additionally, for all  $f \in W^{1,p}(\Omega;\R_+)$, we have 
\begin{align}\label{upper_bound_L_introl}
 \mathcal{L}\leq C\|f\|_{W^{1,p}(\Omega;\R_+)}^{d}\varepsilon^{-d},
\end{align}
where $C$ is a positive constant depending on $m_d(\Omega)$, $d$ and $p$, $m_d(\cdot)$ being the Lebesgue measure in $\R^d$.

\end{theorem}

\begin{remark}
   \Cref{UAT_LP} ensures the existence of a neural network with a fixed width  $d+1$ neurons and sufficient depth approximating any function in $L^p(\Omega;\R_{+})$. As illustrated in the strategy of proof below, the neural network is constructed using geometrical arguments. Therefore, \Cref{UAT_LP} not only merely asserts the existence of parameters, but also provides explicit ones.
\end{remark}

 \emph{Strategy of Proof of Theorem \ref{UAT_LP}.} The proof is based on a two-step approximation procedure. First, a function $f \in L^p(\Omega; \mathbb{R}_{+})$ is approximated by a simple function supported in a family of hyperrectangles. Then, the simple function is approximated by a deep enough neural network. The proof is outlined as follows:
\medbreak

\noindent{\bf Step 1.} Let $f \in L^p(\Omega; \mathbb{R}_{+})$ be a given function and $\varepsilon > 0$. Denote by $\mathcal{C}$ the smallest hyperrectangle containing $\Omega$, oriented according to the axes of the canonical basis of $\mathbb{R}^d$. We extend $f$ by zero into $\mathcal{C}$. Then, we construct a particular simple function $f_h$ that approximates $f$. For its construction, we consider a family $\mathcal{H}_h$ of hyperrectangles of size $h > 0$ such that $\mathcal{H}_h \cup G_h^\delta = \mathcal{C}$, where $G_h^\delta$ is a grid with thickness $\delta\leq h^{1+p}$ that satisfies $m_d(G_h^\delta) \to 0$ as $\delta \to 0$. This is illustrated in \Cref{fig:compres_1}. Then, $f_h$ is defined as the average value of the function $f$ on each hyperrectangle of the family $\mathcal{H}_h$. With this simple function, we can guarantee that there exists $h_1 > 0$ such that for every $h < h_1$, we have $\|f - f_h\|_{L^p(\mathcal{C}; \mathbb{R}_{+})} < \varepsilon / 2$. We denote by $N_h$ the number of hyperrectangles in $\mathcal{H}_h$. 
\medbreak

\begin{figure}
\centering
\subfloat[]{
\includegraphics[width=0.45\textwidth]{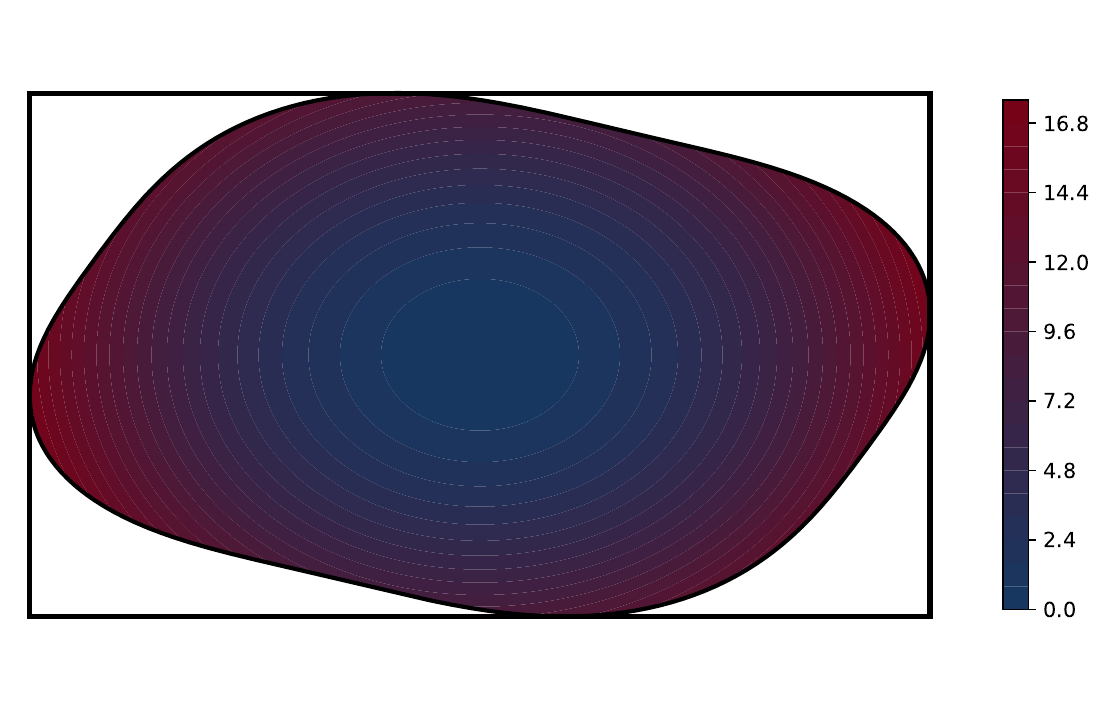}}
\subfloat[]{
\includegraphics[width=0.4\textwidth]{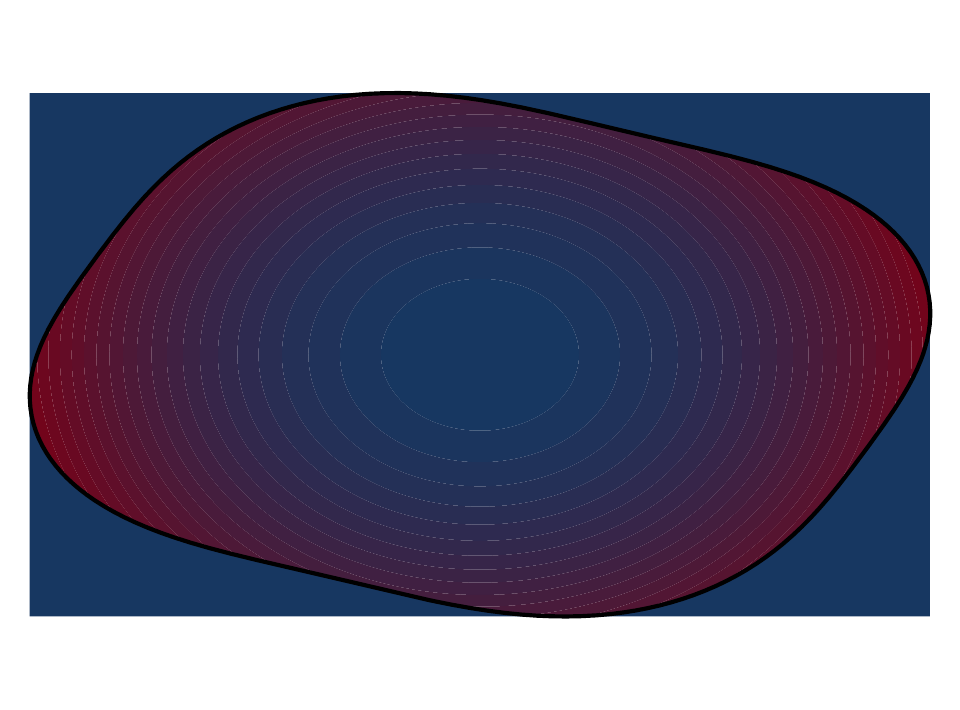}}\\
\subfloat[]{
\includegraphics[width=0.4\textwidth]{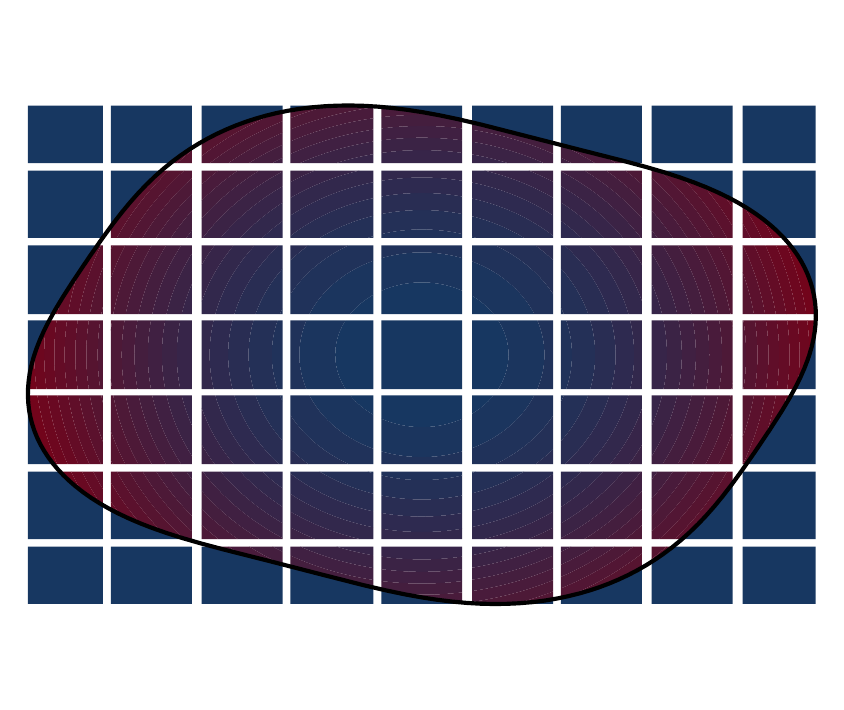}}
\caption{(A) Level sets of the paraboloid $f(x,y)=x^2+y^2$  on $\Omega\subset\R^2$. The rectangle represents the set $\mathcal{C}$. (B) The function $f$ extended by zero to $\mathcal{C}$. (C) Level sets of $f$ on $\mathcal{H}=\mathcal{C}\setminus G_\delta^h$, $G_\delta^{h}$ being the white mesh.}\label{fig:compres_1}
\end{figure}

 \noindent{\bf Step 2.} Then, we construct a neural network $\phi^{\mathcal{L}}$ with a width $d+1$ and depth $\mathcal{L}$, ensuring the existence of $h_2>0$ such that for all $h < h_2$, we have $\|f_h - \phi^{\mathcal{L}}\|_{L^p(\mathcal{C}; \mathbb{R}_{+})} < \varepsilon / 2$. This is done in two steps:
\medbreak

\noindent{\bf Step 2.1.} Let $\mathcal{H}^E_h$  be the subset of hyperrectangles in $\mathcal{H}_h$ that are closest to the edges of $\mathcal{C}$, with $N^E_h$ denoting the number of hyperrectangles in $\mathcal{H}^E_h$. For each hyperrectangle $H \in \mathcal{H}^E_h$, we construct a two-layer neural network with an input-output map $\phi_1^2: \mathbb{R}^d \to \mathbb{R}^d$. This map, which has a width $d+1$, drives $H$ to a single point in $\mathbb{R}^d$ while mapping the remaining hyperrectangles in $\mathcal{H}_h$ to distinct, non-overlapping hyperrectangles in different locations. Here, the chosen parameters ensure the injectivity of the neural network with respect to the hyperrectangles in $\mathcal{H}_h$.

\smallbreak
By iteratively applying these maps, we define a sequence of maps $\phi_i^2$ for $i \in \{1, \dots, N^E_h\}$, such that the composition $\phi^{2N^E_h} := (\phi_{N^E_h}^2 \circ \dots \circ \phi_1^2)$ eventually drives all hyperrectangles in $\mathcal{H}^E_h$ to distinct points. This process results in the dimensional reduction of the remaining hyperrectangles in $\mathcal{H}_h$, where each $\phi_i^2$ transforms $n$-dimensional hyperrectangles into $(n-1)$-dimensional hyperrectangles. Eventually, all hyperrectangles in $\mathcal{H}_h$ are mapped to distinct points $\{x_i\}_{i=1}^{N_h} \subset \mathbb{R}^d$. This approach leverages the specific choice of $\mathcal{H}^E_h$, the dimensionality reduction, and the injectivity of the map $\phi^{2N^E_h}$.

\medbreak

    \medbreak
\noindent{\bf Step 2.2 } Then, we apply Theorem \ref{multiclass_theorem} to construct an input-output map $\phi^{L}$ driving the points $\{x_i\}_{i=1}^{N_h}$ to the $M_h$ values of $f_h$, ensuring that the composition $\phi^{\mathcal{L}} := (\phi^L \circ \phi^{2N^E_h})$ equals $f_h$ on $\mathcal{H}_h$, and therefore $\|f_h-\phi^{\mathcal{L}}\|_{L^p(\mathcal{H}_h;\R_+)}=0$. Additionally, we estimate the error introduced by the neural network on $G_h^\delta$ and show that, for a fixed $h < 1$, the norm $\|\phi^{\mathcal{L}}\|_{L^\infty(G_h^\delta;\R_+)}$ is bounded. This estimation allows us to ensure that there exists a  $h_2>0$ such that for all $h < h_2$, we have $\|f_h - \phi^{\mathcal{L}}\|_{L^p(G_h^\delta;\R_+)} < \epsilon/2$.
\medbreak

\noindent{\bf Step 3.} As a consequence of the triangle inequality and by choosing $h < \min\{h_1, h_2\}$, we obtain $\|f - \phi^{\mathcal{L}}\|_{L^p(\Omega; \mathbb{R}_{+})} < \varepsilon$.
\medbreak

\noindent{\bf Step 4.} Due to the explicit construction of the neural network $\phi^{\mathcal{L}}$, the depth $\mathcal{L}$ can be estimated in terms of $h$. Additionally, assuming $f$ to be in $W^{1,p}$, explicit estimates for $h_1$ and $h_2$ can be obtained in terms of $\|f\|_{W^{1,p}(\Omega; \mathbb{R}_{+})}$, $\varepsilon$, $p$ and $d$ (see \cite{davydov2010algorithms} and its application to isotropic partitions), and therefore conclude \eqref{upper_bound_L_introl}.

\begin{remark}
Several remarks concerning the strategy of the proof of \Cref{UAT_LP} are in order:
\begin{itemize}
    \item  {\bf Step 2}, involving the neural network approximation, is the most challenging one in the proof.
\medbreak

\item  As shown in \Cref{section_preliminars}, each neuron in the neural network represents a hyperplane. In {\bf step 2.1}, $d+1$ neurons are necessary because this is the number of hyperplanes required to separate a single edge hyperrectangle in $\mathcal{H}^E_h$. For instance, in the 2-dimensional case, a hyperrectangle  on the left edge of $\mathcal{C}$ has only one adjacent hyperrectangle above, one below, and one to its right (there are no hyperrectangles outside $\mathcal{C}$). Thus, only $d+1 = 3$ hyperplanes are necessary to separate this hyperrectangle from the others, as opposed to the case where a hyperrectangle is in the interior of $\mathcal{C}$, where four hyperrectangles would surround it, and it would be necessary to use $2d = 4$ hyperplanes. These $d+1$ hyperplanes, in particular, ensure the compression of the separated hyperrectangle into a point.
\medbreak

\item   The set $\mathcal{H}^E_h$ allows us to define a neural network with $2N^E_h$ layers, mapping the $N_h$ hyperrectangles of $\mathcal{H}_h$ to $N_h$ distinct points. Focusing on $\mathcal{H}^E_h$ is essential, as it reduces the number of required layers. Mapping each hyperrectangle of $\mathcal{H}_h$ individually would result in a neural network with significantly more layers, i.e., $N_h \gg N^E_h$. 
\medbreak

\item  The neural network in \Cref{UAT_LP}, defined as $\phi^{\mathcal{L}}=\phi^L\circ \phi^{2N_h^E}$ requires a width $d+1$. However, only the first $2N^E_h$ layers require this width. The width of the $L$ remaining layers can be reduced to $2$ (as a consequence of \Cref{multiclass_theorem}).
\medbreak

\item The neural network of width $2$ from \Cref{multiclass_theorem} maps the compressed hyperrectangles to their corresponding labels and requires $L=2N_h+4M_h-1$ layers, where $M_h$ is the number of distinct values taken by the approximating simple functions $f_h$ (which act as labels). In particular, $M_h \leq N_h$, and $N_h$ can be estimated in terms of $h$ and $\delta$. Moreover, since $N^E_h \leq N_h$ and $M_h \leq N_h$, the depth of the neural network $\phi^{\mathcal{L}}$ is bounded by $\mathcal{L} = 2N^E_h + 2N_h + 4M_h - 1 \leq 8N_h - 1$. Finally, following Step 4, we can estimate the total depth.

\end{itemize}
 \end{remark}
\medbreak
\begin{remark}[An equivalent statement of \Cref{UAT_LP}]
Observe that \Cref{UAT_LP} asserts that for every $\mathcal{L}>0$, there exist parameters $\mathcal{W}^\mathcal{L}$ and $\mathcal{B}^\mathcal{L}$ such that the input-output map of \eqref{discrete_dynamics} with $w_{\max}=d+1$ satisfies
\begin{align*}
    \|\phi^{\mathcal{L}}(\mathcal{W}^\mathcal{L},\mathcal{B}^\mathcal{L},\cdot)-f(\cdot)\|_{L^p(\Omega;\R_{+})}\leq C'\|f\|_{W^{1,p}(\Omega;\R_+)}\mathcal{L}^{-1/d},
\end{align*}
where $C'$ is a positive constant depending on $m_d(\Omega)$, $d$, and $p$.

It is important to mention that in the absence of a finite $W^{1,p}$-norm, a function $f \in L^p(\Omega;\R_{+})$ may exhibit arbitrarily rapid oscillations, making any nontrivial approximation rate unattainable. Therefore, the finiteness of the $W^{1,p}$-norm is essential.
\end{remark}

As a consequence of the preceding theorem and Corollary \ref{corollary_multi_classification}, we have the following universal approximation theorem for $L^p(\Omega;\R^{m}_+)$ functions.

\begin{corollary}[Universal Approximation Theorem in ${L^{p}(\Omega;\R_{+}^m)}$]\label{corollary_UAT}
Let us consider $1\leq p< \infty$, two integers $m,\,d\geq1$, and a bounded domain $\Omega\subset\R^d$. Then, for any $f\in L^p(\Omega;\R^m_+)$ and $\varepsilon>0$, there exist $\mathcal{L}=\mathcal{L}(\varepsilon)\geq 1$ and parameters $\mathcal{W}^\mathcal{L},$ and $\mathcal{B}^\mathcal{L}$ such that the input-output map of \eqref{discrete_dynamics} with width  $w_{max}=\max\{d+1,2m\}$, satisfies
\begin{align}\label{error_f_phi_intro_corollary}
\|\phi(\mathcal{W}^{\mathcal{L}},\mathcal{B}^{\mathcal{L}},\cdot)-f(\cdot)\|_{L^p(\Omega;\R^{m}_+)}<\varepsilon.
\end{align}
Moreover, estimate \eqref{upper_bound_L_introl} for the depth $\mathcal{L}$ is still valid for $f\in W^{1,p}(\Omega;\R^m_+)$.
\end{corollary}
\textit{Strategy of proof for Corollary \ref{corollary_UAT}.} The proof follows a methodology analogous to that of Theorem \ref{UAT_LP}, approximating, first,  the function $f$ by simple functions with support in hyperrectangles. Then, we construct a neural network of width $d+1$ approximating the second simple function. This is done by mapping the hyperrectangles to a set of different points in $\R^d$. Subsequently, to map these points to their $m$-dimensional targets, we utilize $2m$ hyperplanes by applying Corollary \ref{corollary_multi_classification} instead of Theorem \ref{multiclass_theorem}. The width of the network is then determined by the maximum between $d+1$ and $2m$. Furthermore, since Corollary \ref{corollary_multi_classification} also employs a depth of $2N+4M-1$, the estimated depth of the neural network of Corollary \ref{corollary_UAT} is that in Theorem \ref{UAT_LP}.
\medbreak

Clearly, the input-output map of \eqref{discrete_dynamics} cannot approximate functions with negative values. However, in view of Corollary \ref{corollary_clasification_R}, we know that this can be done by means of the more general architecture \eqref{discrete_dynamics_2}. Using  \eqref{discrete_dynamics_2} the universal approximation result can be extended to functions in $L^p(\Omega;\R^m)$.

\begin{corollary}[Universal Approximation Theorem in ${L^p(\Omega;\R^m)}$]\label{remark_UAT_R}
    Let us consider $1\leq p<\infty$,   integers $m,\,d\geq1$, and a bounded domain $\Omega\subset\R^d$. Then for any $f\in L^p(\Omega;\R^m)$ and $\varepsilon>0$, there exist parameters $\mathcal{A}^\mathcal{L},\,\mathcal{W}^\mathcal{L},\,\mathcal{B}^{\mathcal{L}}$, and $\mathcal{L}=\mathcal{L}(\varepsilon)\geq 1$ such that the input-output map of \eqref{discrete_dynamics_2}  with $w_{max}=\max\{d+1,2m\}$, satisfies
\begin{align}\label{error_f_phi_intro_corollary_R}
\|\phi(\mathcal{A}^{\mathcal{L}},\mathcal{W}^{\mathcal{L}},\mathcal{B}^{\mathcal{L}},\cdot)-f(\cdot)\|_{L^p(\Omega;\R^m)}<\varepsilon.
\end{align} 
Moreover, estimate \eqref{upper_bound_L_introl} for the depth $\mathcal{L}$ is still valid for $f\in W^{1,p}(\Omega;\R^m)$.\end{corollary}
\indent\emph{Strategy of Proof of Corollary \ref{remark_UAT_R}.} The proof is analogous to the proof of Corollary \ref{corollary_UAT}, and it concludes by replacing Corollary \ref{corollary_multi_classification} with Corollary \ref{corollary_clasification_R}.

\subsection{Related Work}\label{sec:related_work}
Deep learning has gained popularity due to its state-of-the-art performance in various machine learning applications \cite{Jumper2021,MR3888768}. In practice, neural networks are typically trained using optimization methods minimizing a least-squares error functional, with stochastic gradient descent algorithms serving as an essential tool to search for minimizers. While this numerical approach, combined with backpropagation techniques to compute the gradients, often leads to solutions that outperform human experts, we still lack a solid mathematical understanding of why deep learning works so well. The results in this paper aim to contribute to explain such performance by explicit constructions, which yield concrete estimates of the complexity required for neural networks to achieve the desired goals -- in particular memorization, and universal approximation -- and explicit bounds for the trained parameters.

\smallbreak
In this section, we describe  some recent related advancements in this broad field.
\smallbreak

\noindent{\bf Finite sample memorization:} The literature on the memorization capacity of linear threshold networks, employing a step function $\sigma$, dates back to the 1960s \cite{cover1965geometrical,baum1988capabilities,kowalczyk1997estimates}. In the 1990s, the analysis of single-hidden layer neural networks (FNNs) with more general nonlinear bounded activation functions, such as sigmoids, was conducted (\cite{huang1998upper,huang1990bounds}). These studies show that a single-hidden layer neural network of width $N$ can memorize $N$ points with $N$ classes. A similar result was proven in \cite{zhang2017understanding}, showing that a single hidden layer ReLU network with $N$ neurons can memorize $N$ arbitrary real points, see also \cite{1189626}. In \cite{yun2019small}, it was proved that a 2-hidden layer ReLU network with widths $d_1$ and $d_2$ can memorize a dataset with $N$ points with $N$ classes if $d_1d_2 \geq 4N m$, where $m$ is the dimension of the labels. Therefore, for $m=1$, the width of the neural network in \cite{yun2019small} is $d_1=d_2=2\sqrt{N}$. The above shows that a 2-hidden layer ReLU network can memorize $N$ points with $O(\sqrt{N})$ neurons.

In the context of deep neural networks, \cite{yamasaki1993lower} constitutes one of the first attempts with sigmoid activation functions. For the ReLU activation function, \cite{yun2019small} demonstrated that for a fixed depth $L$, a neural network with a width depending on $N$ and satisfying a technical assumption can memorize $N$ data. In particular, this result holds if there exists $l>1$ such that $ d_jd_j+l =O(Nm)$ for some $j>1$. In \cite{park2021provable}, it is shown that ReLU networks with a width greater than $3$ and \(O(N^{2/3}log(M))\) neurons suffice to approximately memorize \(N\) points with $M$ one-dimensional classes, in the sense that the data can be driven to be \(\varepsilon\)-close to the labels, where \(\varepsilon > 0\) (this constitutes an approximate simultaneous or ensemble controllability result). The same article also establishes that a 3-wide neural network with \(O(N^{2/3}log(M))\) layers (or neurons) suffices for such approximate memorization. Finally, \cite{vardi2021optimal} showed that, by fixing a width  $12$, networks can memorize any $N$ points with $M$ one-dimensional classes using $O(N^{1/2}+log(M))$ layers (or neurons). See \cite{vardi2021optimal} for the exact expression of the required depth. We highlight that none of the previously mentioned works provides an explicit estimate of the parameter norms ($\ell^2$ or $\ell^\infty$); they focus mainly on the number of parameters.

From the control theory perspective, in \cite{dom_zuauza_neuralode, ALVAREZLOPEZ2024106640}, simultaneous controllability for ResNets and neural ordinary differential equations is proven (see also \cite{agrachev2022control}). This implies the memorization property. The novelty in \cite{dom_zuauza_neuralode} lies in the genuinely constructive approach to building parameters. An extension of \cite{dom_zuauza_neuralode} can be found in \cite{cheng2023}, where it is shown that for neural networks with sufficiently large depth but fixed width, interpolation can be guaranteed through the use of non-linear activation functions. Finally, we mention that constructive methods to prove controllability have been widely used in classical control theory; see \cite{MR4546177, MR4522384, MR4520413, MR1421408, MR4560736} and the references therein for a detailed discussion.
\smallbreak

\noindent{\bf Universal Approximation Theorem:} Classical results in this field primarily focus on shallow neural networks \cite{256500,Cybenko1989,Hornik1989MultilayerFN,LESHNO1993861,pinkus1999approximation}, with Cybenko's celebrated work \cite{Cybenko1989} as a notable example, who proved that a single hidden layer neural network could approximate any continuous function within a compact set of $\mathbb{R}^n$ using a sigmoidal activation function.  However, such density results trace back to 1932, with Wiener's Tauberian theorem, which provides necessary and sufficient conditions under which any function in $L^{1}(\mathbb{R})$ or $L^{2}(\mathbb{R})$ can be approximated by linear combinations of translations of a given $L^{1}(\mathbb{R})$ profile, the gaussian, for instance.

On the other hand, recent years have demonstrated that deep networks typically offer better approximation capabilities compared to shallow networks. In this context,  \cite{telgarsky2016benefits} shows that if a ReLU deep neural network is capable of approximating a function with a given error $\varepsilon$ using $L$ layers and relatively narrow width, then a shallower network with a fixed depth of $O(L^{1/3})$ layers would require a width that increases exponentially with $L$ to achieve the same approximation error $\varepsilon$. This finding highlights one of the principal advantages of deeper architectures in neural networks.

With regard to universal approximation in $L^p$ spaces, \cite{article2} demonstrates that a deep neural network with the ReLU activation function and width  $d+4$ can approximate any function in $L^1(\mathbb{R}^d; \mathbb{R})$. They allocate $d$ neurons for transferring input information to subsequent layers, two neurons to carry the information of the approximation made by the previous layers, and two neurons for approximation on each layer. The same article also proves that if the width of a deep neural network is less than $d$, it is impossible to approximate $L^1(\mathbb{R}^d; \mathbb{R})$ or $L^1(\Omega; \mathbb{R})$ for a compact $\Omega$. In \cite{kidger2020universal}, for $p \in [1,\infty)$, a compact set $\Omega \subset \mathbb{R}^d$, and $m \geq 1$, it is established that it is possible to approximate the spaces $L^p(\mathbb{R}^d; \mathbb{R}^m)$ and $L^p(\Omega; \mathbb{R}^m)$ with a ReLU network of width $d+m+1$. Their main argument for approximating $L^p(\mathbb{R}^d; \mathbb{R}^m)$ involves using a neural network to approximate cutoff functions. For $L^p(\Omega; \mathbb{R}^m)$ functions, they prove universal approximation for $C(\Omega; \mathbb{R}^m)$ functions using $m+d+1$ neurons, concluding by density. More precise estimates of the minimal width in $L^p(\mathbb{R}^d; \mathbb{R}^m)$ and $L^p(\Omega; \mathbb{R}^m)$ are presented in \cite{park2020minimum}, which determine the minimal widths to be $\max\{d+1, m\}$ and $\max\{d+2, m+1\}$, respectively. The proof of this theorem utilizes a coding scheme, consisting of encoding (projecting) $x \in \Omega$ into a codeword (scalar values) containing information about $x$, and then a decoder transforming each codeword into a target function $f(x)$. This scheme is applied to approximate continuous functions and completed with density arguments in $L^p$. In the particular case of the \textit{Leaky-ReLU} activation function, a variant of ReLU,  in \cite{cai2022achieve} it was proven that for a compact $\Omega \subset \mathbb{R}^d$, the minimum width required to approximate functions in $L^p(\Omega; \mathbb{R}^m)$ is $\max\{d, m, 2\}$. Recently, \cite{kim2024minimum} has shown that the minimum width of a neural network with a ReLU activation function necessary to approximate $L^p(\Omega; \mathbb{R}^m)$ is $\max\{d, m, 2\}$,  $\Omega$ being a compact set. The proof of this result, based on \cite{park2020minimum}, employs the coding scheme to approximate continuous functions in a compact $\Omega$, concluding the result for $L^p(\Omega; \mathbb{R}^m)$ functions by density.

Universal approximation theorems for the space of continuous functions in the case of arbitrary depth are discussed in \cite{hanin2019universal,hanin2017approximating,
kidger2020universal,park2020minimum,9827563,li2024minimum}. In particular, the minimal width for approximating $L^p$ functions using the recurrent neural networks (RNN) architecture has been studied in \cite{MR4583282}. We refer to \cite{alberti2023sumformer} for  universal approximation theorems in $L^p$, using transformers. For an extended introduction to the universal approximation theorem in more general spaces, see the survey article \cite{devore2021neural}.

\subsection{Our contribution.}

In Theorem~\ref{multiclass_theorem}, we present a constructive proof of the fact that the neural network defined by \eqref{discrete_dynamics} satisfies the finite sample memorization property with width $2$ and depth $L = 2N + 4M - 1$, which implies a total number of neurons of order $O(N)$. To the best of our knowledge, this is the first purely constructive and geometrically interpretable proof of memorization for narrow MLPs. Our approach allows not only for an explicit step-by-step construction of the network parameters but also provides explicit upper bounds on their norms. This has important implications for regularized training:  minimizing a standard empirical loss with $\ell^2$-regularization leads to parameters whose norms are uniformly bounded by the ones of our constructed parameters. In the limit of vanishing regularization, this guarantees exact interpolation while maintaining control over the norm of the trained parameters.

Regarding universal approximation, for MLP of fixed width $d+1$, we establish universal approximation results in $L^p(\Omega;\mathbb{R}+)$ and $L^p(\Omega;\mathbb{R})$, for any $p \in [1,\infty)$ and bounded domain $\Omega \subset \mathbb{R}^d$. Again, our primary contribution lies in our purely constructive proof based on a detailed geometric and recursive argument. In contrast to other constructive results such as \cite{park2020minimum, devore2021neural, kim2024minimum}, our approach emphasizes geometric intuition and visualization at each layer, following ideas similar to \cite{dom_zuauza_neuralode}. While the width requirement in \cite{kim2024minimum} is optimal (namely, width equals $d$), our construction requires width $d+1$ but yields a fully interpretable and elementary strategy to understand how approximation is achieved. Moreover, all parameters involved are given explicitly, which allows us to estimate the depth required to approximate functions with a prescribed accuracy.

\subsection{Outline} The rest of the paper is organized as follows:
In Section \ref{section_preliminars}, we conduct a geometric analysis of the discrete system \eqref{discrete_dynamics}, introducing fundamental tools essential to our proofs. In Section \ref{summary_proof}, we offer an informal demonstration on constructing parameters to guarantee the finite sample memorization property, illustrated with a specific example. Section \ref{formal_proof} contains the formal proof of Theorem \ref{multiclass_theorem}, followed by the proof of the universal approximation theorem (Theorem \ref{UAT_LP}). Section \ref{sec:proo_training} provides the proof of the theorems related to the regularized training and vanishing regularization. Finally, in Section \ref{further_comentaries}, we discuss extensions and open problems.

 \subsection{Notation}
 Throughout this article, we will use the following notation:
 \begin{itemize}

\item We denote by $\llbracket1,L\rrbracket$ the set of numbers $\{1,\dots,L\}$.
\smallbreak

\item The symbol $\cdot$ denotes the Euclidean scalar product between two vectors.   
\smallbreak

\item Given a set $Q$, its cardinal is denoted by $|Q|$.
\smallbreak

\item $\mathbb{S}^{d}$ denotes the  unit $d-$sphere in $\R^{d+1}$.
\smallbreak

\item $\mathcal{W}^L$ and $\mathcal{B}^L$  denote the families of parameters $\{W_j\}_{j=1}^L$ and  $\{b_j\}_{j=1}^L$, respectively.
\smallbreak

\item $w_{max}$ denotes the width of the neural network defined as $\max_{j\in \llbracket1,L\rrbracket}\{d_j\}$.
\smallbreak

\item $m_d(\Omega)$ denotes the Lebesgue measure of $\Omega$ in $\R^d$.
\smallbreak

\item $\|w\|$ stands for the Euclidean norm of the vector $w$ in $\R^d$.
 \smallbreak

\item $\sigma$ denotes the ReLU function and $\vecsigma$ denotes its vector-valued version. 
 \end{itemize}
 
\section{Preliminaries}\label{section_preliminars}
\subsection{Geometrical
interpretation}\label{geometrical_notions}
This section illustrates the dynamics of the system \eqref{discrete_dynamics} from a geometric perspective. 
In what follows, we will refer to the property of finite sample memorization as simply data classification. To simplify the notation, we also avoid the dependence of $\vecsigma_{j}=:\vecsigma$ with respect to the dimension $d_j$ for every $j\in \llbracket1,L\rrbracket$.

\subsubsection{A single hyperplane: }\label{single_hyper}
Let us begin by analyzing the simple case $(N,L,d,d_1)=(1,1,2,1)$. Consider $x^0\in \R^2$, $w\in \R^{1\times 2}$ and $b\in \R$. Under these conditions, the system \eqref{discrete_dynamics} corresponds to 
\begin{align*}
    x^1=\sigma(w\cdot x^0+b)\in \R.
\end{align*}
Let the hyperplane
\begin{align}\label{hiperplane_geometric_interpretation}
    H:=\{x\in \R^2\,:\,w\cdot x+b=0 \},
\end{align}
which divides the space into two half-spaces determined by $w\cdot x+b>0$ and $w\cdot x+b\leq0$ respectively. Thus, the value of $\sigma(w\cdot x^0+b)$ is either zero or equals to $w\cdot x^0+b=\|w\|d(x^0,H)$, depending on the sign of $w\cdot x^0+b$. Here $d(x^0,H)$ denotes the distance between $x^0$ and the hyperplane $H$. This is illustrated in Figure \ref{figure_single_hyp}.
\begin{figure}[H]
\centering
\includegraphics[width=0.6\textwidth]{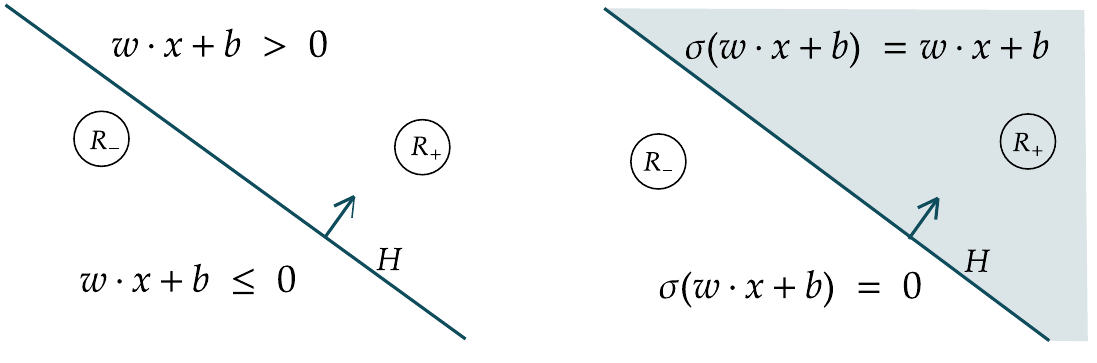}
\caption{Left: $H$ divides the space into two half-spaces $R_{+}$ and $R_{-}$. Right: $R_{+}$ represents the half-space where $\sigma$ is active, while $R_{-}$ represents its null half-space.}\label{figure_single_hyp}
\end{figure}

For future reference, we will say that \emph{$x^0$ is in the activation sector (or region) of $H$} if $w\cdot x^0+b>0$. In Figure \ref{figure_single_hyp}, the sector where the hyperplane $H$ is activated is denoted by $R_{+}$.
\smallbreak
Note that by appropriately choosing the norm of $w$, the distance of the points $x$ within the activation sector can be scaled, either by moving the points closer to or further away from the hyperplane $H$.
 
\subsubsection{Two and more hyperplanes}\label{two_and_more_hyper}
Consider two vectors $w^1,\,w^2\in\R^{1\times 2}$, and scalars $b^1,\,b^2\in \R$. Let us define the matrix $W=(w^1,w^2)^{\top
}$ and the vector $b=(b^1,b^2)^{\top}$. Then, for $x^0\in\R^2$, we have 
\begin{align}\label{computation_x1}
    x^1=\vecsigma(Wx^0+b)=\begin{pmatrix}\sigma(w^1\cdot x^0+b^1)\\ \sigma(w^2\cdot x^0+b^2)\end{pmatrix}.
\end{align}
Denote by $H_1$ and $H_2$ the two hyperplanes defined by $ w^1\cdot x+b^1=0$ and $w^2\cdot x+b^2=0$, respectively. Let $r_1=\|w^1\|d(x^0,H_1)$ and $r_2=\|w^2\|d(x^0,H_2)$, then we have that
\begin{align*}
    \sigma(w^1\cdot x^0+b^1)=\begin{cases} r_1, &\text{if $x^0$ is in the activation sector of $H_1$},\\ 0, &\text{otherwise},\end{cases}
\end{align*}
 while for the second coordinate,
\begin{align*}
    \sigma(w^2\cdot x^0+b^2)=\begin{cases} r_2, &\text{if $x^0$ is in the activation sector of $H_2$},\\ 0, &\text{otherwise}.\end{cases}
\end{align*}
The hyperplanes $H_1$ and $H_2$ partition the plane into four disjoint regions. Depending on the region where a given point $x\in\R^2$ lies, the function $\vecsigma(Wx+b)$ takes a particular value, as depicted in Figure \ref{fig:two_hiperplanes},  mapping $x^0$ into a new point $x^1$ with coordinates $(r_1,r_2)$, which depend, in particular, on the sector where $x^0$ lies.
\begin{figure}
    \centering
    \includegraphics[width=0.6\textwidth]{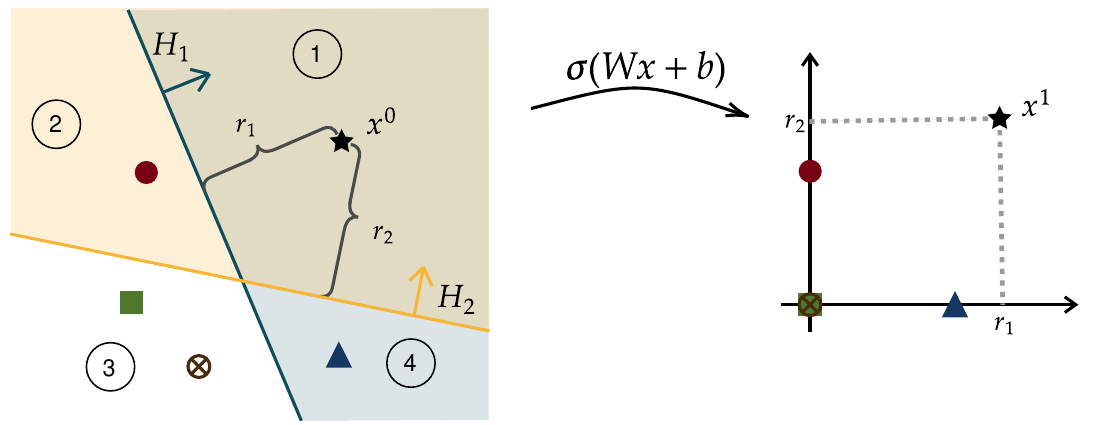}
    \caption{Left: Two hyperplanes split the space into four regions. Different points are chosen in each region. Right: Output of the nonlinear map $\vecsigma(W x+b)$. The green square and brown cross in region 3 are both mapped to the same point, $(0,0)$. The black star is mapped to a new one in the first quadrant, according to its distances to the two hyperplanes. The other two points are mapped to the coordinate axes according to the distance to the hyperplane of the active component.}\label{fig:two_hiperplanes}
\end{figure}
All points in region 1 are mapped to the first quadrant of the plane. Points in regions 2 and 4 are mapped to the coordinate axes. Meanwhile, all points in region 3,  the kernel of the map $\vecsigma(W x+b)$, are mapped to the origin.

\begin{remark}\label{observacion_colapso} Some comments are in order.

\noindent{$\bullet$} As discussed,
$
     K:=\{x\in\R^2\,:\, w^1\cdot x+b^1\leq 0, \text{ and } w^2\cdot x+b^2\leq 0 \}
$
 is the kernel of $\vecsigma(Wx+b)$. It is determined by the parameters $W$ and $b$, that are to be designed to map points to the null point via $\vecsigma$. This allows clustering data. However, it is crucial to choose parameters $W$, and $b$ carefully to ensure that the kernel $K$ does not contain data related to different labels; otherwise, the map $\vecsigma(Wx+b)$ would collapse different labeled points into the same one. If this were to happen, this would render the data classification task impossible. 
\smallbreak

  \noindent {$\bullet$} The same construction can be extended to any number of hyperplanes by considering $W\in\R^{r\times d}$ and $b\in \R^r$. In this case, the function $\vecsigma(W x+b)$ defines a partition of $\R^d$ determined by the family of hyperplanes $H_j=\{x\in\R^d\,:\,w_j\cdot x+b_j=0\}$ for $j\in \llbracket1,r\rrbracket$. These hyperplanes determine the convex kernel $K$, which is not necessarily unbounded since $d+1$ hyperplanes (or more) in a $d-$dimensional space, can determine a bounded kernel, which is then a convex polyhedron.
\smallbreak

  \noindent {$\bullet$} The hyperplanes and the norm of $W$
can be determined first by geometric considerations, and then their parameters are extracted a posteriori to set the neural network's weight and bias.

\end{remark}

\subsection{Projection Lemma}

We present a technical result that will be systematically applied in our proof. This lemma ensures that, given a finite number of points in a $d$-dimensional space, we can always find a direction determining an injective one-dimensional projection of the data.

\begin{lemma}\label{proyeccion_de_datos}
Let us consider a finite set of distinct data $\lbrace x_i\rbrace_{i=1}^N\subset\R^{d}$ such that $x_j\neq x_i$ if $i\neq j$. Then there exists a vector $v\in\mathbb{S}^{d-1}$ such that 
\begin{align}\label{lem1:dif_equation}
    v\cdot x_i\neq v\cdot x_j, \qquad \text{for every }\, i\neq j\in\llbracket1,N\rrbracket.
\end{align}
\end{lemma}
\begin{proof}
For each pair $1\le i<j\le N$, define the set of directions
\begin{align*}
H_{ij} =\big\{v\in\mathbb S^{d-1}\,:\,v\cdot(x_i - x_j)=0\big\}.
\end{align*}
Since $x_i\neq x_j$, the vector $q_{ij}:=x_i - x_j\in\R^d\setminus\{0\}$, so $H_{ij}$ is the intersection of the unit sphere with the hyperplane orthogonal to $q_{ij}$. Hence, $H_{ij}$ is a closed subsphere of dimension $d-2$ and thus a proper, nowhere‐dense subset of $\mathbb S^{d-1}$.
Set
\begin{align*}
G =\bigcup_{1\le i<j\le N}H_{ij}.
\end{align*}
Then $G$ is a finite union of closed, nowhere‐dense subsets of the compact manifold $\mathbb S^{d-1}$. By the Baire Category Theorem (or by observing that each $H_{ij}$ has surface‐measure zero and hence their finite union cannot exhaust the full‐measure sphere), we conclude $\mathbb S^{d-1}\setminus G\neq\emptyset$. Any $v\in\mathbb S^{d-1}\setminus G$ satisfies
\begin{align*}
v\cdot(x_i - x_j)\neq0,\quad\forall i\neq j,
\end{align*}
which is equivalent to $v\cdot x_i\neq v\cdot x_j$ whenever $i\neq j$, concluding the proof.
\end{proof}
 
\begin{remark}\label{observation22}
A few remarks are necessary.

\noindent{$\bullet$} The above argument extends without modification to any countable collection of pairwise distinct points ${x_i}_{i\in\mathbb{N}}\subset\mathbb{R}^d$. Indeed, $G$ is a countable union of closed subspheres of codimension 1, so by the Baire Category Theorem $\mathbb{S}^{d-1}\setminus G\neq\varnothing$. However, for an uncountable (e.g., \ continuous) data set, one may have $G=\mathbb{S}^{d-1}$, and thus, no separating direction exists.

\noindent{$\bullet$} If the data set is finite, each $H_{ij}$ is a closed subsphere of dimension $d-2$, and their finite union $G$ is a closed set with empty interior. Hence $G^c=\mathbb{S}^{d-1}\setminus G$ is open and dense. In particular, for any $g\in G$ and any $\varepsilon>0$, there exists $g_\varepsilon$ in an $\varepsilon-$neighborhood of $g$ such that $g_\varepsilon\cdot(x_i - x_j)\neq0$ for all $i\neq j$, showing that the projection property is stable under small perturbations.

\end{remark}

 \section{Sketch of the proof of Theorem \ref{multiclass_theorem}: An example }\label{summary_proof}
 In this section, we illustrate the proofs of Theorem \ref{multiclass_theorem} through a specific example. The formal proof and the estimation of the depth $L$ are provided in the subsequent section. 
\smallbreak
Let us consider the dataset $\{x_i,y_i\}_{i=1}^8\subset \R^3\times\llbracket0,3\rrbracket$.  For ease of visualization, we assume their labels correspond to four shapes of different colors: red circle, blue triangle, green square, and brown cross. We aim to drive blue triangles to $0$, red circles to $1$, brown crosses to $2$, and green squares to $3$. In the following, we will refer to a \emph{class of elements} as a set that contains points associated with the same label. Therefore, in our example, we have four such classes of elements.
\smallbreak
In the ensuing discussion, we illustrate the main steps needed to complete the classification process.
\smallbreak

\noindent \textbf{1) Preconditioning of the data: }We choose the parameters $w_1$ and $b_1$ to project the $3-$dimensional data injectively into a one-dimensional space, ensuring that all the projected data remain distinct. Note that the neural network's width in this step is one since we use one hyperplane. 
\begin{figure}[H]
    \centering
    \includegraphics[width=0.7\textwidth]{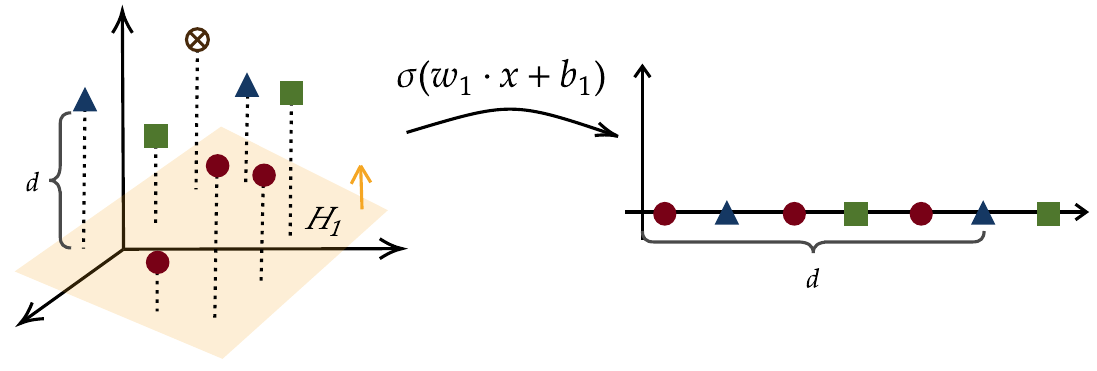}
    \caption{
     Projection of the data in a one-dimensional space using the hyperplane $H_1$.
       } \label{example1}
\end{figure}

\smallbreak

\noindent \textbf{2) Compression process: } Inspired by Remark \ref{observacion_colapso}, in this step, we aim to collapse each class of elements into a single point. For this purpose, it is enough to show how to collapse a single class while keeping the other three classes well separated throughout the process, and then proceed inductively. We illustrate this step by compressing the red circles, indicating which hyperplanes are needed to carry out this process.
\smallbreak

 \noindent   {$\bullet$ Step 2.1:}
    Starting from the output of the previous step, we define two hyperplanes, $H_2^1$ and $H_2^2$. These are chosen such that the red circle at the left is mapped to the $y-$axis, the blue triangle in between is driven to $(0,0)$, and the remaining data points are mapped to the $x-$axis. 
\begin{figure}[H]
    \centering
    \includegraphics[width=0.7\textwidth]{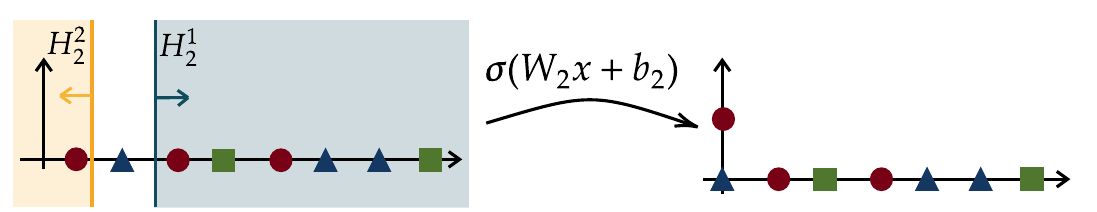}
\end{figure}
\smallbreak

\noindent    {$\bullet$ Step 2.2:} We now consider two diagonal hyperplanes $H_3^1$ and $H_3^2$ enclosing the red circles between them. Their activation sector is chosen to make the two red circles collapse to the origin (see Remark \ref{observacion_colapso}).
    \begin{figure}[H]
    \centering
    \includegraphics[width=0.7\textwidth]{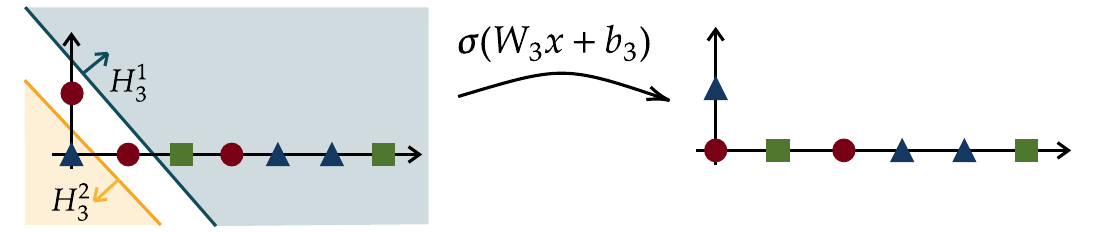}
\end{figure}
\smallbreak

\noindent    {$\bullet$ Step 2.3:} With the same goal as in Step 2.1, we consider two hyperplanes $H_4^1$ and $H_4^2$ that place one red circle on the $y-$axis and another red circle on the $x-$axis.
    \begin{figure}[H]
    \centering
    \includegraphics[width=0.7\textwidth]{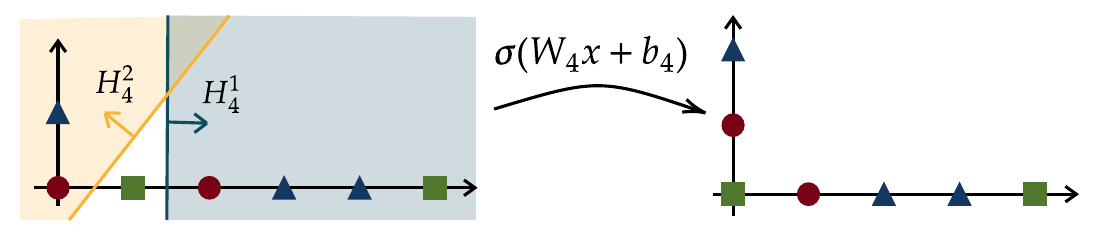}
\end{figure}
The slope of $H_4^2$ plays an important role since, if the hyperplane was vertical, the blue triangle and red circle to its left (and all the points that lie on $y-$axis) would be at the same distance from the hyperplane $H_4^2$ and thus would be driven to the same value. 
\medbreak

We can iteratively apply Steps 2.2 and 2.3, as many times as necessary, to collapse all the red circles. 

By applying this process to each class, we can define the input-output map $\phi^{L_2}:\R\to\R^2$ described by Figure \ref{intro1}.
\begin{figure}[H]
    \centering
    \includegraphics[width=0.7\textwidth]{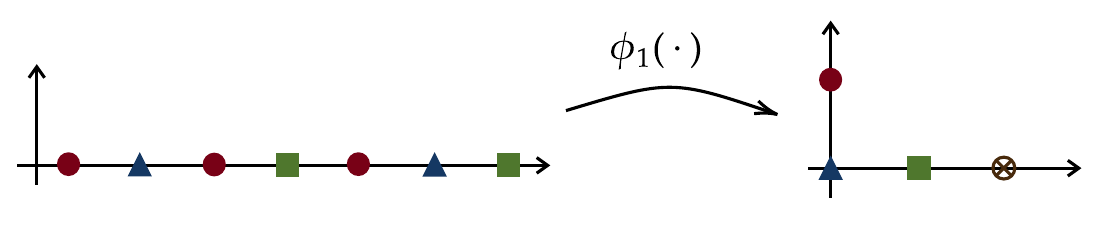}
    \caption{
    Compression of each class into a single point using the mapping $\phi_1$.} \label{intro1}
\end{figure}

\smallbreak

\noindent \textbf{3) Data sorting: }After completing the previous step, all points within each class have been driven into the same position. Hence, the points of each group become indistinguishable and inseparable, allowing us to treat them as a single reference point. Let us denote the reference point associated with the label $i$ as $z_i$.

Note, however, that the outputs of the last step do not provide any specific ordering of these reference points. In this third step, our aim is to show how to reorder these reference points along the real line according to their labels.
\smallbreak

We will outline the first steps and illustrate how to carry out the inductive process.
\smallbreak

\noindent    {$\bullet$ Step 3.1:} We begin by projecting the two-dimensional data into the real line using any hyperplane, ensuring that all projected data remain distinct.
    \begin{figure}[H]
    \centering
    \includegraphics[width=0.7\textwidth]{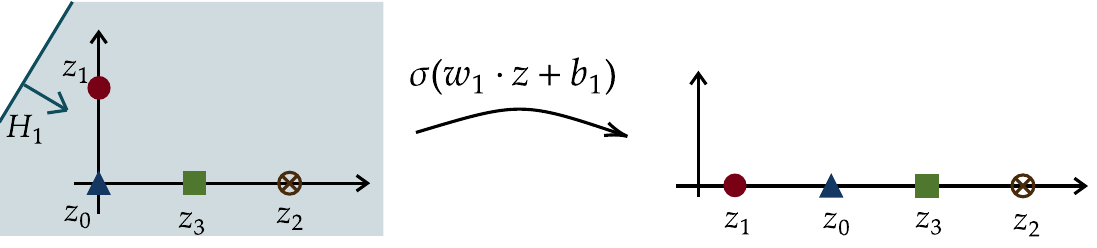}
\end{figure}
\smallbreak

\noindent    {$\bullet$ Step 3.2:}     
    We consider two vertical hyperplanes to drive only the first point $z_0$ to $(0,0)$.
    \begin{figure}[H]
    \centering
    \includegraphics[width=0.7\textwidth]{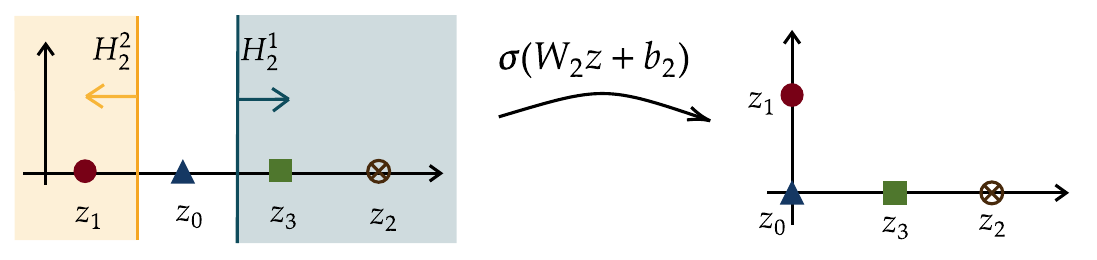}
\end{figure}
\smallbreak

\noindent    {$\bullet$ Step 3.3:} We consider a hyperplane such that the activated semi-space contains all the points, and the closest point to it is the one in $(0,0)$. This allows us to sort the first point. 
    \begin{figure}[H]
    \centering
    \includegraphics[width=0.7\textwidth]{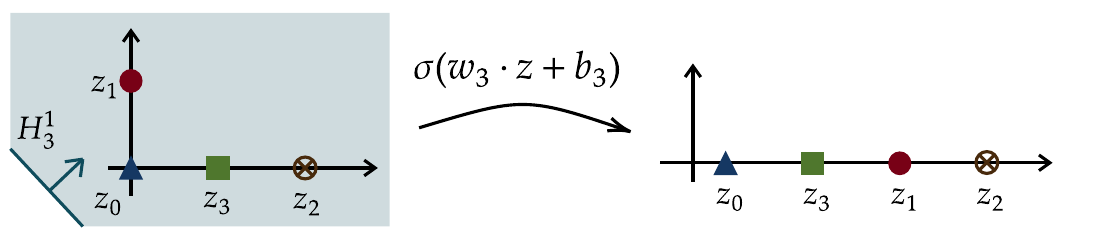}
\end{figure}
\smallbreak

\noindent    {$\bullet$ Step 3.4:} Again, we consider vertical hyperplanes $H_4^1$ and $H_4^2$, this time to drive only $z_1$ to $(0,0)$. 
    \begin{figure}[H]
    \centering
    \includegraphics[width=0.7\textwidth]{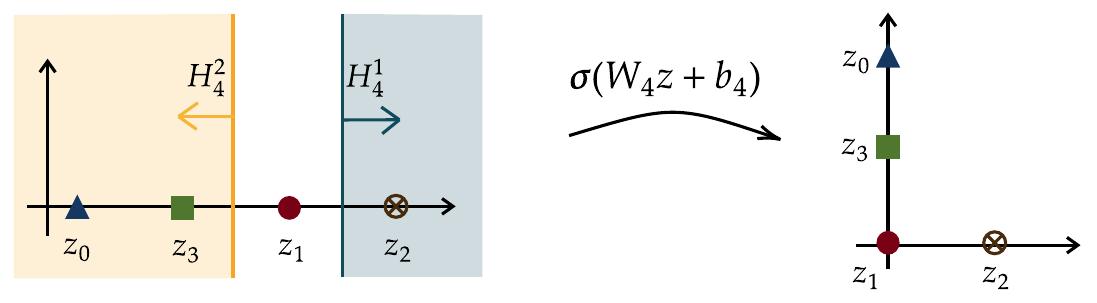}
\end{figure}
\smallbreak

\noindent    {$\bullet$ Step 3.5:} We consider a hyperplane $H_5^1$ such that the nearest point is $z_0$ and the farthest point is $z_1=(0,0)$. This allows us to ensure that $z_1$ will be the farthest point from zero.
    \begin{figure}[H]
    \centering
    \includegraphics[width=0.7\textwidth]{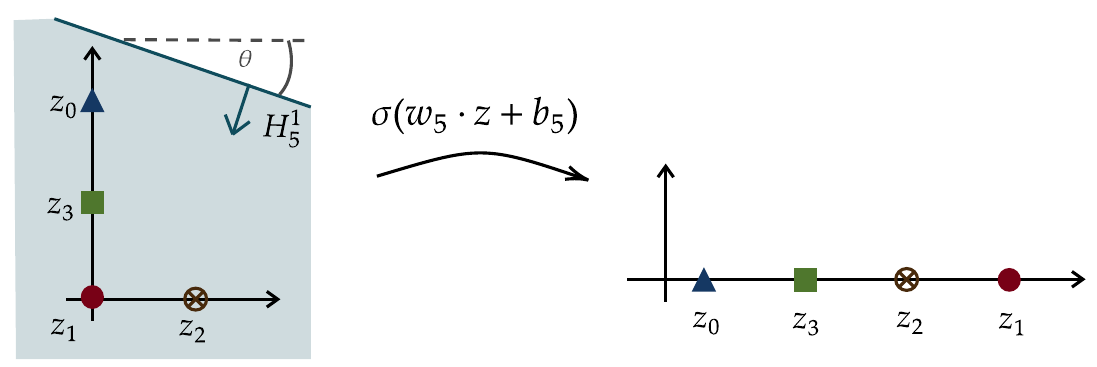}
\end{figure}
\smallbreak

\noindent    {$\bullet$ Step 3.6:} We consider vertical hyperplanes to drive only $z_2$ to $(0,0)$.
    \begin{figure}[H]
    \centering
    \includegraphics[width=0.7\textwidth]{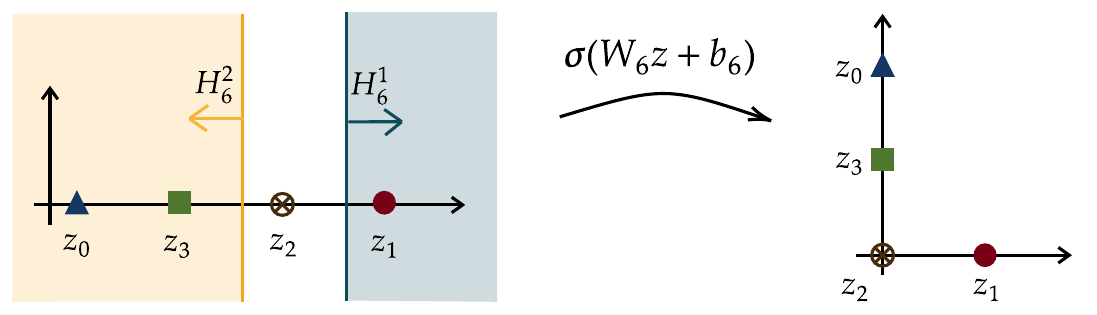}
\end{figure}
\smallbreak

 \noindent   {$\bullet$ Step 3.7}:  We consider a hyperplane such that the closest point is $z_0$, then $z_1$, and the farthest point is $z_2=(0,0)$.
        \begin{figure}[H]
    \centering
    \includegraphics[width=0.7\textwidth]{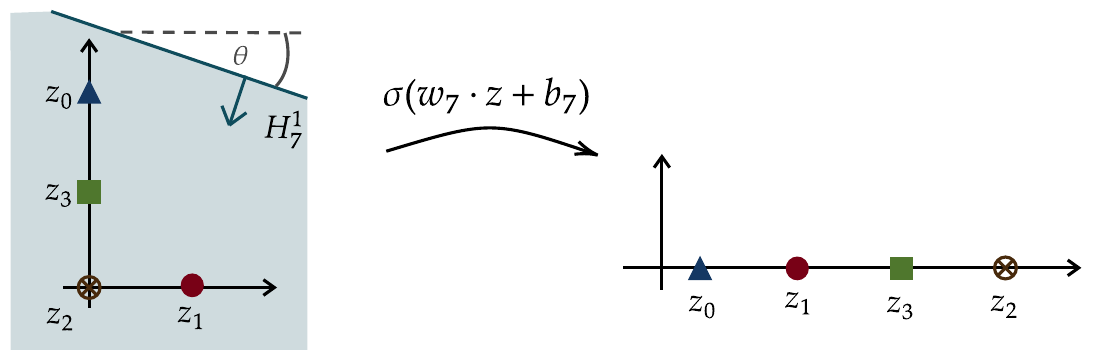}
\end{figure}
\medbreak

The first two points are well collocated, and the data we want to sort next, $z_2$,  is at the end.  Applying steps 3.6 and 3.7 iteratively, we can sort $z_2$ and all the remaining positions, always taking a suitable slope $\theta$ for the hyperplane in Step 3.5. 
\medbreak

 \noindent\textbf{4) Mapping to the respective labels:} We will show how to drive each point to its corresponding label. This is done by applying projections and choosing the weights properly to have a correct distance scaling.
\smallbreak

\noindent{$\bullet$ Step 4.1:} We begin by considering a hyperplane containing $z_0$, so that it can contract or dilate the position of $z_1$, to drive this point to $1$, while sending $z_0$ to $(0,0)$. This maps $z_0$ to $0$ and $z_1$ to $1$. 
    \begin{figure}[H]
    \centering
    \includegraphics[width=0.7\textwidth]{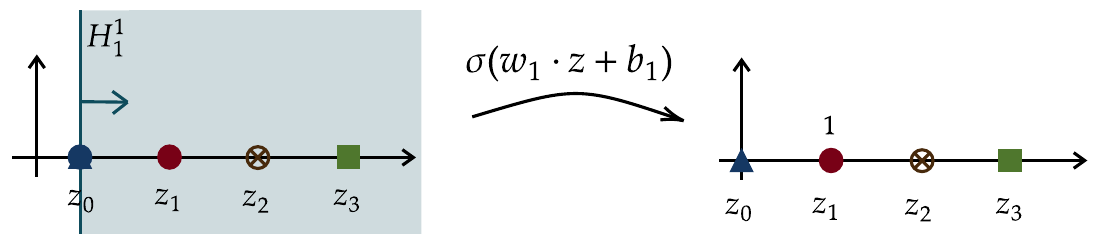}
    \end{figure}
\smallbreak

\noindent{$\bullet$ Step 4.2:} We now consider two hyperplanes. The first hyperplane passes through $z_0$ and ensures that the existing order of the data along the $x$-axis is preserved (after applying $\sigma$, this order is maintained along the $y$-axis). The second hyperplane is placed between $z_1$ and $z_2$ and is used to push or pull $z_2$ along the $x$-axis, moving it closer to the value $1$ on the $x$-axis after applying $\sigma$.
        
    \begin{figure}[H]
    \centering
    \includegraphics[width=0.7\textwidth]{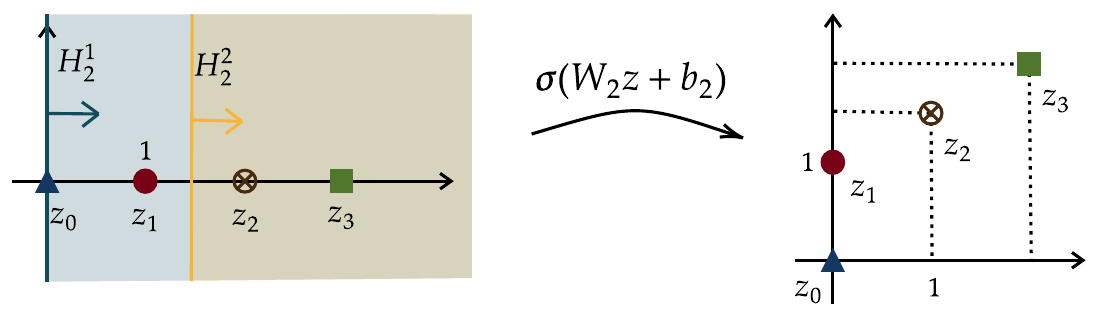}
    \end{figure}
\smallbreak

\noindent{$\bullet$ Step 4.3:} We define a hyperplane of equation $w_3\cdot x=0$, so it contains $z_0$. The parameter $w_3$ is such that $\sigma(w_3\cdot z_1)=1$ and  $\sigma(w_3\cdot z_2)=2$.
    \begin{figure}[H]
    \centering
    \includegraphics[width=0.7\textwidth]{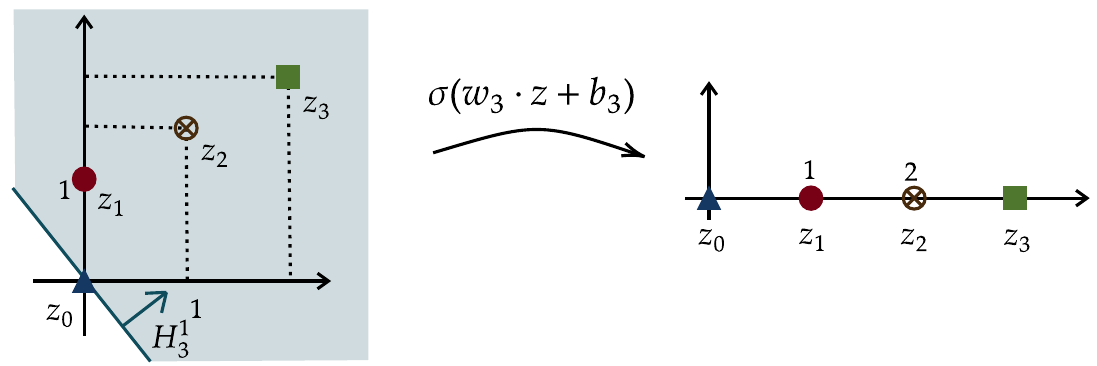}
    \end{figure}
\medbreak

We note that we can again go back to step 4.2, considering the same hyperplane containing $z_0$ and the second one located between $z_2$ and $z_3$. Then, when reproducing step 4.3, we would choose $w_3$ such that $w_3\cdot z_1=1$,  $w_3\cdot z_2=2$, and $w_3\cdot z_3=3$. This is feasible because the first two conditions coincide (given that $z_2=2z_1$)  (see Step 4 in Section \ref{formal_proof} for the explicit construction of $w_3$).

By iterating this procedure and combining steps 4.2 and 4.3, we can bring all data to their respective labels.

\begin{remark}[Minimal width deep neural network]\label{remark_minimal_width}
   At this point, it becomes evident that at least two hyperplanes are necessary to develop an algorithm for classifying data. Indeed, if we were restricted to using just a single hyperplane, we would be unable to develop a compression process for every data set. Therefore, it is not feasible to classify any $d-$dimensional dataset using a $1$-wide neural network. Thus, $2$ is the minimal width for a neural network (as defined by \eqref{discrete_dynamics}) capable of addressing any classification problem.
\end{remark}

The computation of the depth of the process described above is detailed in the following section. It shows that for the first stage, $1$ layer is necessary;  $2N$ layers for the second stage; $2M+1$ for the third one; and, finally, $2M-3$ layers for the fourth stage. In total, the depth of the neural network is $2N+4M-1$ layers. Moreover, it is possible to see that in the first stage (since we project from $\R^d$ to $\R$), the number of neurons is $d$. In the second stage, the number of neurons is $2(2N)$. In the third stage, we alternate between one and two dimensions, and the number of neurons becomes $(M+1)+2M$. For the fourth stage, the number of neurons also alternates and is equal to $(M-3)+2 M$. The total number of neurons is then $4N+6M+d-2$.

\begin{remark}
The strategy described above has been implemented in Python for a binary classification; the code is available in \href{https://github.com/Martinshs/multiclass_UAT/tree/main}{github.com/Martinshs}. This implementation demonstrates how the explicit parameters provided in the proof of Theorem~\ref{multiclass_theorem} can be used to drive each input to its target level set without any training process.
\end{remark}

\section{Proof of Theorem \ref{multiclass_theorem} }\label{formal_proof}

This section is devoted to presenting the details of the proof of Theorem \ref{multiclass_theorem}. Each stage described in the previous section is developed in a separate subsection. The proofs of the steps $2$, $3$, and $4$ will be done by induction.

\subsection{{Preconditioning of the data:} } 

To prove the first step, let us consider the dataset $\lbrace x_i,y_i\rbrace_{i=1}^N\subset\R^{d}\times \llbracket0,M-1\rrbracket$.  Lemma \ref{proyeccion_de_datos} assures the existence of a vector $w_1\in\R^{d}$ satisfying
    \begin{align*}
        w_1\cdot x_i\neq w_1\cdot x_j,
    \end{align*}
    for every $i\neq j\in \llbracket1,N\rrbracket$. Now, let $b_1\in\R$ be large enough such that 
\begin{align*}
   w_1\cdot x_i+b_1>0, \quad \forall i\in \llbracket1,N\rrbracket,
\end{align*}
which implies
\begin{align*}
     \sigma(w_1\cdot x_i+b_1)\neq \sigma(w_1\cdot x_j+b_1),  \quad \forall i\neq j\in \llbracket1,N\rrbracket,
\end{align*}
and also $\sigma(w_1\cdot x_i+b_1)>0$ for all $i\in \llbracket1,N\rrbracket$.  We denote by $\lbrace x_i^1\rbrace_{i=1}^{N} \subset \R$ the projected one-dimensional new data given by
\begin{align}\label{definition_first_itera}
    x^1_i=\sigma(w_1\cdot x_i+b_1), \quad i\in \llbracket1,N\rrbracket.
\end{align}
Note that the points $\{x_i\}_{i=1}^{N}$ are collocated according to their distance to the hyperplane  $w_1\cdot x+b_1=0$. In other words, for $L_0=1$ and with the choice of the parameters $\mathcal{W}=\{w_1\}$, and $\mathcal{B}=\{b_1\}$ we have that
\begin{align}\label{phi_L0}
    \phi^{L_0}(x_i)=\phi(\mathcal{W}^1,\mathcal{B}^1,x_i)=x^1_i.
\end{align}
This is illustrated in Figure \ref{example1}. 

\subsection{{Compression process.}} 
We divide this section into two parts. In the first part, we show that an induction procedure suffices. The second part focuses on showing that a single class can be compressed.

\subsubsection{The induction strategy}

Consider the sets corresponding to the different classes of points:
\begin{align}\label{Conjuntos_xx}
    \mathcal{C}_k=\{x_i \text{ with } i\in \llbracket1,N\rrbracket \,:\, y_i=k \},\quad\text{and}\quad \mathcal{C}=\bigcup_{k=0}^{M-1} \mathcal{C}_k.
\end{align}
 In the sequel we write $\phi(\mathcal{W}^L,\mathcal{B}^L,\mathcal{C}_k)=z_k$ when 
$
    \phi(\mathcal{W}^L,\mathcal{B}^L,x)=z$  for every $x\in \mathcal{C}_k
$,
Therefore, compressing all the classes is equivalent to proving the existence of parameters $\mathcal{W}^{\tilde{L}}$ and $\mathcal{B}^{\tilde{L}}$, $\tilde{L}>0$, and a sequence of different vectors $\{z_k\}_{k=0}^{M-1}\subset\R^2$ such that 
\begin{align}\label{compresion_to_differents_points}
     \phi(\mathcal{W}^{\hat{L}},\mathcal{B}^{\hat{L}},\mathcal{C}_k)=z_k,\qquad\text{for every } k\in\llbracket0,M-1\rrbracket.
\end{align}
The following proposition guarantees that this problem can be handled in an inductive manner. 

\begin{proposition}\label{key_proposition} 
 For any $k\in\llbracket0,M-1\rrbracket$ fixed but arbitrary, we assume that there exist $\tilde z_0 \in \R^2$, $\tilde{L}\geq 1$, $\mathcal{W}^{\tilde{L}}$ and $\mathcal{B}^{\tilde{L}}$ such that 
\begin{align}\label{condition_1_propo1}
     \phi(\mathcal{W}^{\tilde{L}},\mathcal{B}^{\tilde{L}},\mathcal{C}_{k})=\tilde z_0,\qquad  \phi(\mathcal{W}^{\tilde L},\mathcal{B}^{\tilde L},\mathcal{C}\setminus\mathcal{C}_k)\neq \tilde z_0,
\end{align}
and 
\begin{align}\label{condition_2_propo1}
    \phi(\mathcal{W}^{\tilde{L}},\mathcal{B}^{\tilde{L}},\tilde z_1)\neq  \phi(\mathcal{W}^{\tilde{L}},\mathcal{B}^{\tilde{L}},\tilde z_2),\quad\text{for all   } \tilde z_1,\,\tilde z_2\in\mathcal{C}\setminus \mathcal{C}_{k}, \, \tilde z_1\neq \tilde z_2.
\end{align}
Then, there exist $L_1\geq 1$, $\mathcal{W}^{L_1}$, $\mathcal{B}^{L_1}$ and different vectors $\{z_k\}_{k=0}^{M-1}\subset\R^2$ such that 
\begin{align*}
     \phi(\mathcal{W}^{L_1},\mathcal{B}^{L_1},\mathcal{C}_k)=z_k,
\end{align*}
for $k\in\llbracket0,M-1\rrbracket$.
\end{proposition}

In other words, Proposition \ref{key_proposition} shows that to compress all the classes of points, it is sufficient to compress a single (but arbitrary) class of points without collapsing the points not belonging to that class.

The proof of Proposition \ref{key_proposition} can be found in Appendix \ref{appendix_1}.

\subsubsection{Compression of a single class.}\label{subsec:compression}

    Let us take $k\in\llbracket0,M-1\rrbracket$ arbitrarily but fixed. Our goal is to drive the class $\mathcal{C}_k$ to some vector $z_k\in \R^2$ in $\tilde{L}\geq 1$ steps. We will do it by induction. 
    
    We focus on the worst-case scenario in which points in the class $\mathcal{C}_k$ are isolated, not having neighboring points of the same class, which could be treated simultaneously as a single point,  reducing the number of layers needed.
\medbreak
  \noindent In this procedure, we will combine  two operations, in an alternating manner:
 \smallbreak

  \noindent {$\bullet$} Data structuring: Construction of hyperplanes driving the data set to some particular structure.
 \smallbreak

  \noindent {$\bullet$} Compression process: Using the structure established in the prior step,  introduce hyperplanes to collapse points belonging to the same class.
\medbreak

\noindent \textbf{(1) Initial Step:}
We show that the two first points of the class $\mathcal{C}_k$ closer to zero can be compressed.
\smallbreak
\noindent \emph{Data structuring: } Given that data have been projected into the one-dimensional real line, without loss of generality, we can assume that the new data $\{x_i^1\}_{i=1}^N$, defined in \eqref{definition_first_itera}, are indexed according to their order, i.e., $x_i^1\leq x_j^1$ for every $i\leq j$. Let $\mathcal{C}_k^1$ be given by 
    \begin{align*}
        \mathcal{C}_k^1=\left\{     \sigma(w_1\cdot x+b_1)\in\R\,:\, x\in \mathcal{C}_k \right\},
    \end{align*}
where $w_1$ and $b_1$ are the parameters defined in the \emph{Preconditioning of the data} step.
    Let us denote by $x_{r_1}^1$  the smallest element of the class $\mathcal{C}_k^1$. Then, we introduce the parameters $W_2=(w_2^1,w_2^2)^\top
$ and $b_2=(b_2^1,b_2^2)^\top
$ with
\begin{align*}
w_2^1 = 1,\quad w_2^2 =-1,\quad b_2^1=-\left(\frac{x_{r_1+2}^1+x_{r_1+1}^1}{2}\right),\quad \text{and}\quad b_2^2=\frac{x_{r_1+1}^1+x_{r_1}^1}{2}.
\end{align*}
Data are then mapped into the 2-dimensional vectors (see Figure \ref{figure_paso1})
\begin{align*}
    x^2_i=\vecsigma(W_2 x^1_i+b_2)=\begin{pmatrix}\sigma(w_2^1 x^1_i+b_2^1)\\\sigma(w_2^2 x^1_i+b_2^2)\end{pmatrix},
\end{align*}
such that
\begin{align}\label{estructura_inicial_datos}
    \begin{cases}
x_i^2=(0,a_i^2)\quad \text{for all }i\in\llbracket1,r_1\rrbracket, \\
x_{r_1+1}^2=(0,0),\\
x_i^2=(a_i^2,0) \quad \text{for all }i\in\llbracket r_1+2,N\rrbracket , 
\end{cases}
\end{align}
for some $\{a_i^2\}_{i=1}^N\subset\R_+$.
Denote by $H_2^1$ and $H_2^2$ the vertical hyperplanes defined by the parameters $(w_2^1,b_2^1)$ and $(w_2^2,b_2^2)$, respectively. 
\begin{figure}[H]
    \centering
    \includegraphics[width=0.7\textwidth]{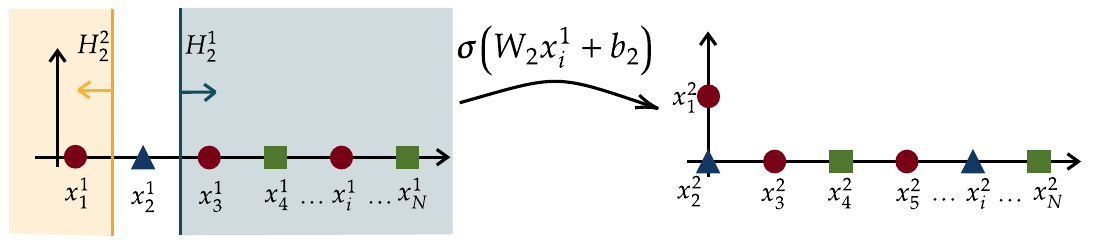}
    \caption{
       In these figures, $\mathcal{C}_k$ corresponds to the class of red circles. The vertical hyperplane $H_2^2$ separates $x_1^1$ and $x_2^1$, and the hyperplane $H_2^1$ is placed between $x_2^1$ and $x_3^1$. The ReLU vector-valued function maps the points to the left of $H_2^2$ to the $y-$axis and those to the right of $H_2^1$  to the $x-$axis. The point between the two planes is mapped to the origin. } \label{figure_paso1}
\end{figure}

\noindent\emph{Compression: }
Let $x_{r_1}^2=\vecsigma(W_2 x^1_{r_1}+b_2)$ with $x^1_{r_1}$ defined in the previous step, and
\begin{align*}
        \mathcal{C}_k^2=\left\{   \vecsigma(W_2 x+b_2)\in\R\,:\, x\in \mathcal{C}_k^1 \right\}.
    \end{align*}
Denote by $x_{r_2}^2\in \mathcal{C}_k^2$ the closest element to the null vector on the $x-$axis in the class $\mathcal{C}_k^2$. Then, define  $W_3=(w_3^1,w_3^2)^\top
$ and $b_3=(b_3^1,b_3^2)^\top
$ with \begin{align}\label{parameter_1}
    w_3^1=&\left(\frac{a^2_{r_1-1}+a^2_{r_1}}{a^2_{r_2+1}+a^2_{r_2}},1\right),\quad w_3^2=\left(-\frac{a^2_{r_1}+a^2_{r_1+1}}{a^2_{r_2}+a^2_{r_2+1}},-1\right),\\
    \quad &b_3^1=-\left(\frac{a^2_{r_1-1}+a^2_{r_1}}{2}\right),\quad b_3^2=\frac{a^2_{r_1}+a^2_{r_1+1}}{2}
\end{align}
and set
$
    x^3_i=\vecsigma(W_3 x^2_i+b_3),
$
which are of the form
\begin{align}\label{estructura2_datos}
    \begin{cases}
x_i^3=(0,a^3_i)\quad \text{for all }i\in\llbracket r_1+1,r_2-1\rrbracket, \\
x_{r_{1,2}}^3:=x_{r_1}^3=x_{r_2}^3=(0,0),\\
x_i^3=(a_i^3,0) \quad \text{for all }i\in\llbracket 1,r_1-1\rrbracket\cup\llbracket r_2+1,N\rrbracket  
\end{cases}
\end{align}
for some $\{a_i^3\}_{i=1}^N\subset\R_+$.
The hyperplanes $H_3^1$ and $H_3^2$ are defined by the parameters $(w_3^1,b_3^1)$ and $(w_3^2,b_3^2)$, respectively. The argument above is illustrated in Figure \ref{figure8}.
\begin{figure}[H]
    \centering
    \includegraphics[width=0.7\textwidth]{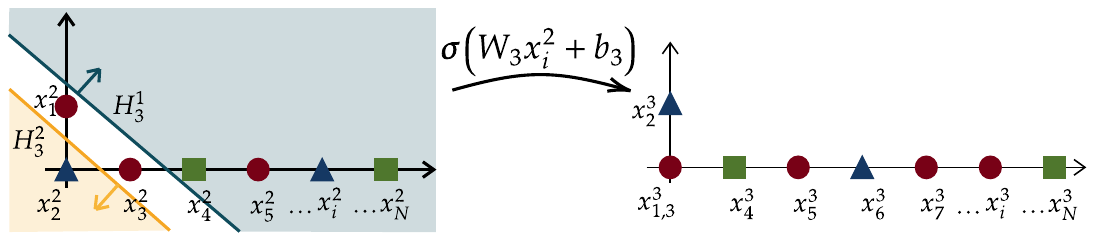}
    \caption{
    The points above the hyperplane $H_3^1$ are mapped to the $x-$axis, and the points below the hyperplane $H_3^2$ are mapped to the $y-$axis, while the points between the two hyperplanes are compressed to the null vector. We denote by $x_{1,3}^3=x_1^3=x_3^3$ the null vector to which the two red circles collapse.
    }\label{figure8}
\end{figure}

\noindent {\bf (2) Inductive Step: } 
The initial step has been achieved. We now aim to show that induction can also be applied successfully. In fact, to do that, it suffices to apply the arguments of the initial step again and again. Note that the model under consideration collapses point to the exact same location, and once this happens, they will never split again in the forthcoming iterations. In this way, if $\alpha_k$ denotes the number of elements in $\mathcal{C}_k$ i.e. $|\mathcal{C}_k|=\alpha_k$, all points in $\mathcal{C}_k$ will collapse applying $2\alpha_k$ times the nonlinear mapping  $\sigma$, so that $L_k=2\alpha
_k$.
\smallbreak

Let us assume that we have compressed the first $j\in\llbracket1,\alpha_k\rrbracket$ elements of the class $\mathcal{C}_k$. Note that to compress $j$ elements, it is necessary to apply two steps per element (\emph{data structuring} and \emph{compression} steps). Therefore, it is necessary to apply $2j$ steps. Denote by $\mathcal{C}_k^{2j+1}$ the class  $\mathcal{C}_k^1$ after having applied  $2j$ steps to it. Let us show that we can compress the  $j+1$-th element of $\mathcal{C}_k^{2j+1}$.
\smallbreak

 \noindent    \emph{Data structuring: } Denote by $x_{r_t}^{2j+1}$ the elements of $\mathcal{C}_k^{2j+1}$ for $t\in \llbracket1,\alpha_k\rrbracket$. Observe that, after compressing the first $j$ elements of $\mathcal{C}_k^{2j+1}$, we will always have that $x_{r_{t}}^{2j+1}=(0,0)$ for $t\in\llbracket1,j\rrbracket$ and $x_{r_{j+1}}^{2j+1}=(a_{j+1}^{2j+1},0)$ for some $a_{j+1}^{2j+1}\in\R$. Since we assumed from the beginning that they are not neighboring points of the same class, let $x_{s}^{2j+1}=(a^{2j+1}_{s},0)\notin \mathcal{C}_k^{2j+1}$ be the nonzero closest point to $x_{r_1}^{2j}=(0,0)$ in the $x-$axis. We also consider $x_{s+1}^{2j+1} = (a_{s+1}^{2j+1}, 0)$ with $a_{s+1}^{2j+1} \neq 0$, the point to the right of $x_{s}^{2j+1}$ on the $x$-axis (if it does not exist, we take $x_{s+1}^{2j+1} := x_{s}^{2j+1} + (0.5, 0)$).
    
Let us consider the parameters
\begin{align}\label{parameter_2}
    w_{2j+2}^1=(0,1),\quad w_{2j+2}^2=\left(\frac{1}{2},-\frac{1}{2}\right),\quad b_{2j+2}^1=-\left(\frac{a^{2j+1}_{s+1}+a_{s}^{2j+1}}{2}\right),\quad b_{2j+2}^2=\frac{a_{s}^{2j+1}}{4}.
\end{align}
Define $W_{2j+2}=(w_{2j+2}^1,w_{2j+2}^2)^\top
$ and $b_{2j+2}=(b_{2j+2}^1,b_{2j+2}^2)^\top
$ and set
\begin{align*}
    x^{2j+2}_i=\vecsigma(W_{2j+2} x^{2j+1}_i+b_{2j+2}),
\end{align*}
so that
\begin{align*}
    \begin{cases}
x_{r_t}^{2j+2}=(0,a_j^{2j+2})\quad \text{for all }t\in\llbracket 1,j \rrbracket,\\
x_{s}^{2j+2}=(0,0),\\
x_{r_{j+1}}^{2j+2}=(a_{j+1}^{2j+2},0)   .
\end{cases}
\end{align*}

 \noindent \emph{Compression: } We are going to compress $x_{r_j}^{2j+2}$ with $x_{r_{j+1}}^{2j+2}$. Let $x_u^{2j+2} = (0,a_u^{2j+2})$ and  $x_d^{2j+2}=(0,a_d^{2j+2})$ be the points lying above and below $x_{r_j}^{2j+2}$ on the $y-$axis, respectively. Also, denote by $x_{rr}^{2j+2}=(a_{rr}^{2j+2},0)$ and  $x_{ll}^{2j+2}=(a_{ll}^{2j+2},0)$ the points that are to the right and left of $x_{r_{j+1}}^{2j+2}$ in the $x-$axis. Thus, we consider the parameters
\begin{align}\label{parameter_iterative_compression}
     \nonumber w_{2j+3}^1=&\left(\frac{a^{2j+2}_{u}+a^{2j+2}_{j}}{a^{2j+2}_{rr}+a^{2j+2}_{j+1}},1\right),\quad w_{2j+3}^2=\left(-\frac{a^{2j+2}_{d}+a^{2j+2}_{j}}{a^{2j+2}_{ll}+a^{2j+2}_{j+1}},-1\right),\\
    \quad &b_{2j+3}^1=-\left(\frac{a^{2j+2}_{u}+a^{2j+2}_{j}}{2}\right),\quad b_{2j+3}^2=\frac{a^{2j+2}_{d}+a^{2j+2}_{j}}{2}.
\end{align}
Define $W_{2j+3}=(w_{2j+3}^1,w_{2j+3}^2)^\top
$ and $b_{2j+2}=(b_{2j+3}^1,b_{2j+3}^2)^\top
$. Therefore, we have
\begin{align*}
    x^{2j+3}_i=\vecsigma(W_{2j+3} x^{2j+2}_i+b_{2j+3}).
\end{align*}
This selection of parameters allows us to ensure that $x_{r_t}^{2j+3}=(0,0)$ for all $t\in\llbracket 1,j+1\rrbracket$. 
This concludes the induction argument.
\medbreak

As we already observed, we need $\hat L_k=2\alpha_k$ layers to compress all the elements of $\mathcal{C}_k$. Consequently, we have shown that for any arbitrary $k\in\llbracket0,M-1\rrbracket$ there exist $z_k\in \R^2$, a depth $\hat L_k\geq 1$, and parameters $\mathcal{W}^{\hat L_k}$ and $\mathcal{B}^{\hat L_k}$ such that \eqref{condition_1_propo1} and \eqref{condition_2_propo1} hold.  Furthermore, we can explicitly construct the parameters by following \eqref{parameter_2} and \eqref{parameter_iterative_compression}.

As a consequence of the Proposition \ref{key_proposition}, there exist $L_1\geq 1$, parameters $\mathcal{W}^{L_1}$, and $\mathcal{B}^{L_1}$, and a sequence of different points $\{z_k\}_{k=0}^{M-1}\subset\R^2$ such that 
\begin{align}\label{phi_l1}
     \phi(\mathcal{W}^{L_1},\mathcal{B}^{L_1},\mathcal{C}_k)=z_k,\qquad \text{for all }k\in\llbracket0,M-1\rrbracket.
\end{align}
Therefore, to compress all classes, the vector-valued $\vecsigma$ function must be applied
\begin{align*}
    \sum_{k=0}^{M-1} 2\alpha_k=2\sum_{k=0}^{M-1} |\mathcal{C}_k|=2N,
\end{align*}
times. In other words, the depth of the neural network has to be $L_1=2N$ to compress all classes.

\subsection{{Data sorting}}
In the previous step, we have shown that we can reduce our dataset $\{x_i\}_{i=1}^{N}{\subset}\R^d$ to a set $\{z_k\}_{k=0}^{M-1}{\subset}\R^2$, in which each element represents a class or label. Without loss of generality, we assume that each $z_k$ is associated with a label $k$.
\smallbreak

In this section, we aim to find $L_2>0$ and parameters $\mathcal{W}^{L_2},\,\mathcal{B}^{L_2}$ such that for a strictly increasing sequence $\{\xi_k\}_{k=0}^{M-1}\subset\R$, we have
\begin{align}\label{goal_step_3}
\phi(\mathcal{W}^{L_2},\mathcal{B}^{L_2},z_k)=\xi_k,\quad \text{for all } k\in \llbracket0,M-1\rrbracket.
\end{align}
Let $\{\beta_\eta\}_{\eta=0}^{M-1}$ be a sequence of positive numbers and $\{\hat\xi_k\}_{k=0}^{M-1}\subset\R$ a strictly increasing sequence. Note that to prove \eqref{goal_step_3}, it is sufficient to show that
\begin{align}\label{goal_induction_step3}
\begin{split}
\phi(\mathcal{W}^{\beta_\eta}&,\mathcal{B}^{\beta_\eta},z_k)=\hat\xi_k, \qquad \text{for all } k\in\llbracket0,\eta\rrbracket,\\
&\phi(\mathcal{W}^{\beta_\eta},\mathcal{B}^{\beta_\eta},z_{\eta+1})= \hat\xi_{M-1},
 \end{split}
\end{align}
for every $\eta\in\llbracket0,M-2\rrbracket$. This is equivalent to asserting that we can order the first $\eta$ points, and place  $z_{\eta+1}$ as the farthest point from zero. Clearly, when $\eta=M-2$ we recover \eqref{goal_step_3}.
\smallbreak

We will prove \eqref{goal_induction_step3} by induction on $\eta$, applying a data structuring process, similar to the \textit{``compression of a single class"} step, and a projection process in which Lemma \ref{proyeccion_de_datos} will be consistently utilized.
\smallbreak

 \noindent {\bf (1) Initial Step: } Our goal is to prove that \eqref{goal_induction_step3} is fulfilled for $\eta=0$. We proceed in several steps.
 \smallbreak
  \noindent  \emph{Projection:} We start by projecting the data to the one-dimensional line. By Lemma \ref{proyeccion_de_datos}, there exist $w_1$ and $b_1$ such that 
\begin{align*}
    z_k^1=\sigma(w_1\cdot z_k+b_1)\in \R,
\end{align*}
satisfies $z_i^1\neq  z_j^1$ for all $i\neq j\in \llbracket0,M-1 \rrbracket$. In the following, when $z_k^l\in \R$ for some $l\geq 1$, we add an extra sub-index $j_k$ in  $z_{k,j_k}^l$ to denote the actual position with respect to the other elements counting from left to right (see Figure \ref{illustration_1}). Clearly, depending on which point $k$ we consider, its position ($j_k$) can be different. However, we will only make the dependence of $j_k$ on $k$ explicit when necessary.
\begin{figure}[H]
    \centering
    \includegraphics[width=0.7\textwidth]{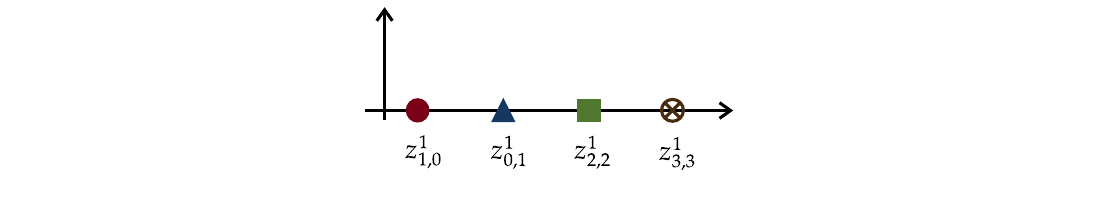}
    \caption{
    Sequence $\{z_{k,j}^1\}_{k=0}^{M-1}\subset\R$. The index $k$ indicates the class to which $z_{k,j}^1$ belongs, while the index $j$ indicates its position from left to right.} \label{illustration_1}
\end{figure}

\noindent\emph{Data structuring:} Let us define
\begin{align*}
w_2^1 = 1,\quad w_2^2 =-1,\quad b_2^1=-\left(\frac{z_{0,j+1}^1+z^1_{0,j}}{2}\right),\quad \text{and}\quad b_2^2=\frac{z_{0,j}^1+z_{0,j-1}^1}{2},
\end{align*}
and denote $W_2=(w_2^1,w_2^2)^\top
$ and $b_2=(b_2^1,b_2^2)^\top
$. Then, we define
$
    z^2_{k}=\vecsigma(W_2 z^1_{k}+b_2).
$
With the above parameters, we ensure that $z_{0}^1=(0,0)$. See Figure \ref{figure_steps1}.
\begin{figure}[H]
    \centering
    \includegraphics[width=0.6\textwidth]{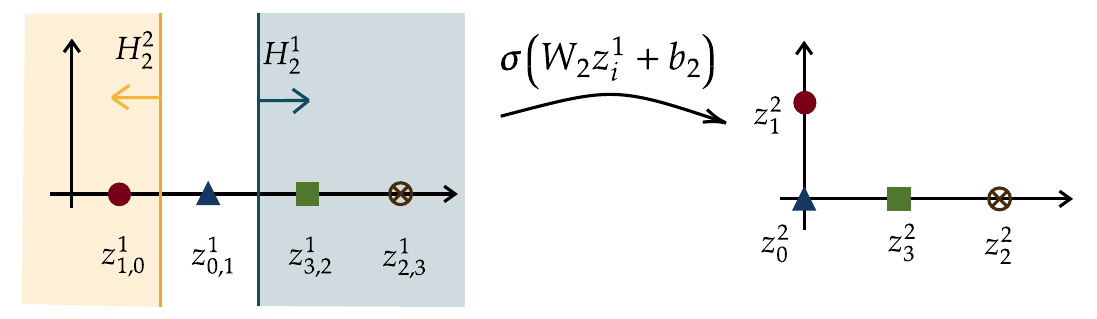}
    \caption{ The hyperplanes $H_2^1$ and $H_2^2$ are defined by the equations $w_2^1\cdot x+b_2^1=0$ and $w_2^2\cdot x+b_2^2=0$, respectively.
   This step is similar to the first one in the compression process (see Figure \ref{figure_paso1})} \label{figure_steps1}
\end{figure}

\noindent\emph{Projection: }Due Lemma \ref{proyeccion_de_datos} there exist $w_3\in \R^2$ and $b_3\in\R$ such that
\begin{align*}
 0\leq w_3\cdot z_0^2+b_3 < w_3\cdot z_{k}^2+b_3,\quad \text{for all } k\in \llbracket1,M-1\rrbracket.
\end{align*}
Define
$
    z_{k,j}^3=\sigma(w_3\cdot z_k^2+b_3).
$
By construction, $z_0$ is the closest point to the hyperplane $ w_3\cdot z+b_3=0$, so it will be the closest point to zero after the projection step. Consequently, $j_0=0$, and the first point has been sorted. It remains to prove that the second point can be moved to the last position.
\begin{figure}[H]
    \centering
    \includegraphics[width=0.6\textwidth]{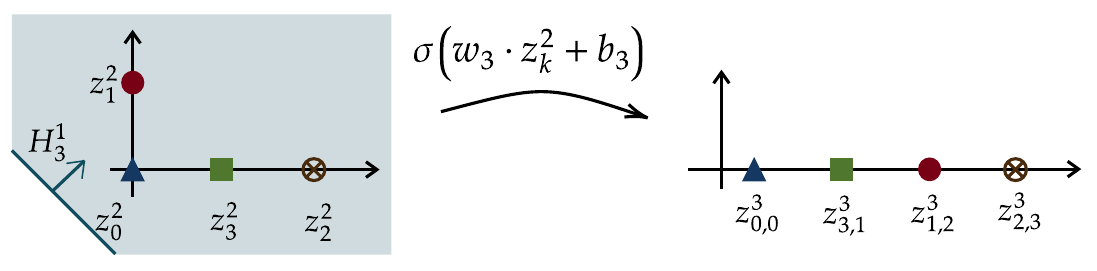}
    \caption{
    The hyperplane $H_3^1$ is defined by the equation $w_3^1\cdot x+b_3=0$ so that $z_0^3$ is the first point from left to right.}\label{illustration_3}
\end{figure}

\noindent\emph{Data structuring:} Let us define
\begin{align*}
w_4^1 = 1,\quad w_4^2 =-1,\quad b_4^1=-\left(\frac{z_{1,j+1}^3+z^3_{1,j}}{2}\right),\quad \text{and}\quad b_4^2=\frac{z_{1,j}^3+z_{1,j-1}^3}{2},
\end{align*}
and denote $W_4=(w_4^1,w_4^2)^\top
$ and $b_4=(b_4^1,b_4^2)^\top
$. Then, define
$
    z^4_{k}=\vecsigma(W_4 z^3_{k}+b_4).
$
With the above parameters, $z_{1}^4=(0,0)$, while $z_0^4$ is the farthest point from the origin on the $y-$axis (see Figure \ref{figure_steps4}).
\begin{figure}[H]
    \centering
    \includegraphics[width=0.6\textwidth]{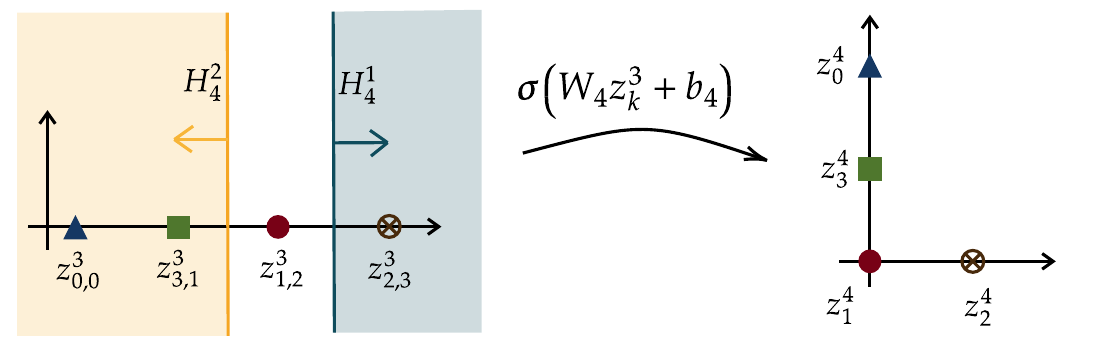}
    \caption{ The hyperplanes $H_4^1$ and $H_4^2$ are defined by the equations $w_4^1\cdot x+b_4^1=0$ and $w_4^2\cdot x+b_4^2=0$, respectively.} \label{figure_steps4}
\end{figure}

\noindent\emph{Projection:} Let us consider a vector $w_5\in\R^2$ and $b_5\in \R$ such that 
\begin{align}\label{conidition_1}
    0\leq w_5\cdot z_0^4+b_5< w_5\cdot z_k^4+b_5,\quad \text{ for all }  k\in \llbracket1,M-1\rrbracket,
\end{align}
and also 
\begin{align}\label{conidition_2}
    w_5\cdot z_k^4+b_5<w_5\cdot z_1^4+b_5,\quad\text{ for all } k\in \llbracket2,M-1\rrbracket.
\end{align}
With these parameter values: 
$
    z_{k,j}^5=\sigma(w_5\cdot z_k^4+b_5).
$
 By construction, we have that \eqref{goal_induction_step3} is satisfied when $\eta=0$, concluding the initial step.  
\begin{remark}\label{remark_iteration_sorting} The following aspects should be highlighted.
\smallbreak

  \noindent {$\bullet$} As observed in Figure \ref{figure_steps5}  the hyperplane $H_5^1$ considered satisfies the conditions \eqref{conidition_1} and \eqref{conidition_2}. They are satisfied whenever $\theta\in(0,\theta^*)$, where $\theta$ is the angle between the $x-$axis and the hyperplane $H_5^1$, and $\theta^*$ is the angle between the $x-$axis and hyperplane containing $z_{0}^4$ and the farthest point of $\{z_k^4\}_{k=1}^{M-1}$ on the $x-$axis.
\smallbreak

  \noindent {$\bullet$} Note that, in the initial step of the proof, we iterate five times to obtain \eqref{goal_induction_step3} with \( \eta = 0 \). But given the configuration  \(\{ z_k^5 \}_{k} \), only two extra steps (specifically, the last data structuring and projection steps) are necessary to satisfy \eqref{goal_induction_step3} with \( \eta = 1 \). Therefore, to sort the first \( \eta\in \llbracket1,M-3\rrbracket \) classes, one requires $ 5 + 2\eta$ iterations.
\smallbreak

\noindent {$\bullet$} By construction, the point to be sorted is always the one that is placed the furthest from the origin. Therefore, it is enough to sort the first $M-1$ points, and the point $M$ will automatically be sorted.
\end{remark}
\begin{figure}[H]
    \centering
    \includegraphics[width=0.6\textwidth]{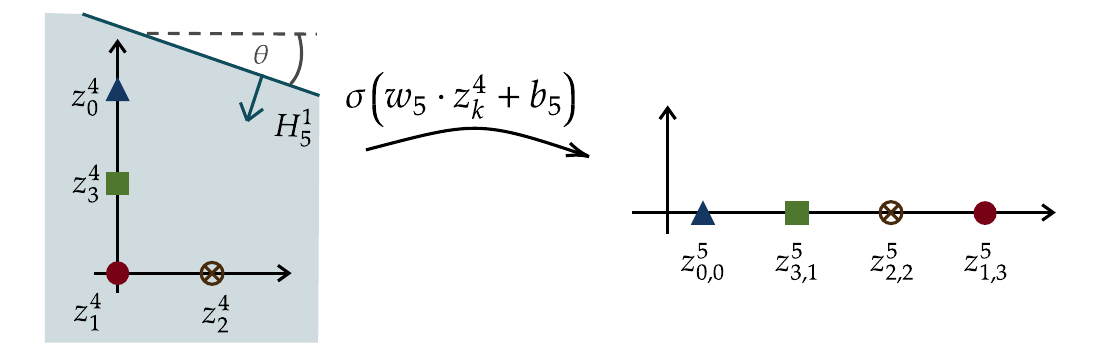}
    \caption{ The hyperplane $H_5^1$ is defined by  $w_5^1\cdot x+b_5^1=0$, so that $j_0=0$ and $j_1=3$.} \label{figure_steps5}
\end{figure}

\noindent {\bf (2) Inductive Step: } Let $\eta\in \llbracket1,M-2\rrbracket$ and $\ell:=5+2(\eta+1)$. Assume that we have sorted the first $\eta$ points of $\{z_{k}^{\ell}\}_{k=0}^{M-1}$, that is $j_k=k$ for $k\in\llbracket0,\eta\rrbracket$. We will show that we can sort one extra element. 
\smallbreak
By construction, we can assume that the element that we have to sort is the farthest one in the $x-$axis, i.e., $j_{\eta+1}=M-1$. 

\noindent\emph {Data structuring:} Let $k_1,\,k_2\in \llbracket\eta,M-1\rrbracket$ be such that $z_{k_1,j_{\eta+2}-1}^{\ell}$ and $z_{k_2,j_{\eta+2}+1}^{\ell}$ are the left and right neighborhood points of $z_{\eta+2,j_{\eta+2}}^{\ell}$, respectively (if $z_{k_2,j_{\eta+2}+1}^{\ell}$ does not exist, it is enough to take $z_{k_2,j_{\eta+2}+1}^{\ell}:=z_{k_1,j_{\eta+2}}^{\ell}+1$). Thus, consider the parameters
\begin{align*}
w_{\ell+1}^1& = 1,\quad w_{\ell+1}^2 =-1,\quad b_{\ell+1}^1=-\left(\frac{z_{k_2,j_{\eta+2}+1}^{\ell}+z_{\eta+2,j_{\eta+2}}^{\ell}}{2}\right),\\
&\qquad \text{and}\quad b_{\ell+1}^2=\frac{z_{\eta+2,j_{\eta+2}}^{\ell}+z_{k_1,j_{\eta+2}-1}^{\ell}}{2}.
\end{align*}
Define $W_{\ell+1}=(w_{\ell+1}^1,w_{\ell+1}^2)^\top
$ and $b_{\ell+1}=(b_{\ell+1}^1,b_{\ell+1}^2)^\top
$ and set
\begin{align*}
    z^{\ell+1}_{k}=\vecsigma(W_{\ell+1} z^{\ell}_{k}+b_{\ell+1}).
\end{align*}
We obtain that $z^{\ell+1}_{\eta+2}=(0,0)$ and $z^{\ell+1}_{\eta+1}=(a,0)$, for some $a>0$,  is the farthest point in the $x-$axis. Moreover, for a decreasing sequence of positive numbers $\{a_{k}\}_{k=0}^{M-1}$, we deduce
\begin{align}\label{caracterizacion_sucesion_induccion}
    z^{\ell+1}_k=(0,a_k),\quad\text{for all }k\in \llbracket0, k_1 \rrbracket. 
\end{align}

\noindent\emph{Projection: } Let $w_{\ell+2}\in \R^2$ and $b_{\ell+2}\in \R$ be such that
\begin{align}\label{r1_induction} 
0\leq w_{\ell+2}\cdot z_k^{\ell+1}+b_{\ell+2} < w_{\ell+2}\cdot z_{k+1}^{\ell+1}+b_{\ell+2}\quad\text{for all }k\in \llbracket0,\eta\rrbracket,\\
w_{\ell+2}\cdot z_k^{\ell+1}+b_{\ell+2} < w_{\ell+2}\cdot z_{\eta+2}^{\ell+1}+b_{\ell+2}\quad\text{for all }k\in \llbracket0,M-1\rrbracket.\label{eq:farthes_inductive}
\end{align}

The assumptions about $W_{\ell+2}$ mentioned above are not restrictive. Condition \eqref{r1_induction} is feasible, as shown by \eqref{caracterizacion_sucesion_induccion}, where the sequence $z_{k}^{\ell+1}$ is arranged on the $y-$axis for $k\in \llbracket0, M-2\rrbracket$ as a decreasing sequence. Furthermore, \eqref{eq:farthes_inductive} is achievable by selecting the appropriate slope of the hyperplane defined by the parameters (see Remark \ref{remark_iteration_sorting}). We then set
$
   z_{k,j}^{\ell+2}=\sigma(w_{\ell+2}\cdot z_k^{\ell+1}+b_{\ell+2}),
$
and, by construction, we have that $j_k=k$ for $k\in \llbracket1,\eta+1\rrbracket$ and $j_{\eta+2}=M-1$, concluding the induction.
\medbreak
Therefore, taking $\eta=M-2$ in   \eqref{goal_induction_step3}, we can sort all the data. Thus for $L_2=5+2(M-2)=1+2M$, there exist $\mathcal{W}^{L_2}$, $\mathcal{B}^{L_2}$ and a strictly increasing sequence $\{\xi_k\}_{k=0}^{M-1}\subset\R$ such that \eqref{goal_step_3} holds.

\subsection{Mapping to the respective labels}
We start from the output of the previous step, where we have shown that for $L_2=1+2M$ there exist parameters $\mathcal{W}^{L_2}$and $\mathcal{B}^{L_2}$ such that for a strictly increasing sequence $\{\xi_k\}_{k=0}^{M-1}\subset\R$, we have
\begin{align}\label{goal_step_3_again}
\phi^{L_2}(x_i)=\phi(\mathcal{W}^{L_2},\mathcal{B}^{L_2},z_k)=\xi_k,\quad \text{for all } k\in \llbracket0,M-1\rrbracket.
\end{align}
Our goal in this step is to prove, again by induction, that there exist $L_3>0$, and parameters $\mathcal{W}^{L_3}$ and $\mathcal{B}^{L_3}$ such that
\begin{align}\label{phi_l3}
\phi^{L_3}(x_i)=\phi(\mathcal{W}^{L_3},\mathcal{B}^{L_3},\xi_k)=k,\quad \text{for all } k \in \llbracket0,M-1\rrbracket.
\end{align} 
\smallbreak

\noindent {\bf (1) Initial Step: } We begin by sorting the first three elements. 
\smallbreak

\noindent \emph{Data projection.} Consider the parameters
\begin{align*}
    w_1=\frac{1}{\xi_1-\xi_0}, \text{ and } b_1= \frac{-\xi_0}{\xi_1-\xi_0},
\end{align*}
and $
    \xi_k^1=\sigma(w_1\cdot\xi_k+b_1).
$
We have  $\xi_0^1=0$ and $\xi_1^1=1$.
\begin{figure}[H]
    \centering
    \includegraphics[width=0.6\textwidth]{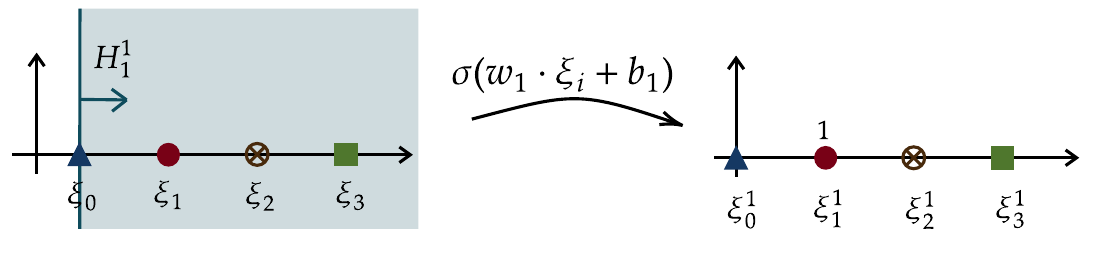}
    \caption{ The hyperplane $H_1^1$ is defined by the equation $w_1\cdot x+b_1=0$. } \label{figure_steps1_label}
\end{figure}

\smallbreak
\noindent \emph{Data structuring. } With the parameters
\begin{align*}
w_2^1 =\frac{2}{\xi_2^1-\xi_1^1} ,\quad w_2^2 =1,\quad b_2^1=-\left(\frac{\xi_2^1+\xi_1^1}{\xi_2^1-\xi_1^1}\right),\quad \text{and}\quad b_2^2=0
\end{align*}
define $w_2=(w_2^1,w_2^2)^\top
$ and $b_2=(b_2^1,b_2^2)^\top
$, and set
$
    \xi_k^2=\vecsigma(W_2 \xi_k^1+b_2).
$
By construction, we deduce 
\begin{align*}
    \xi_0^2=(0,0),\,
    \xi_1^2=(0,1),\,
    \xi_2^2=(1,\xi_2^1),\,
    \xi_k^2=(a_k,c_k),\quad \text{for all }k\in \llbracket3,M\rrbracket,
\end{align*}
where $\{a_k\}_{k=3}^M$, and $\{c_k\}_{k=3}^M$ are two increasing sequences satisfying that $a_3>1$ and $c_3>\xi_2^1$.
\begin{figure}[H]
    \centering
    \includegraphics[width=0.6\textwidth]{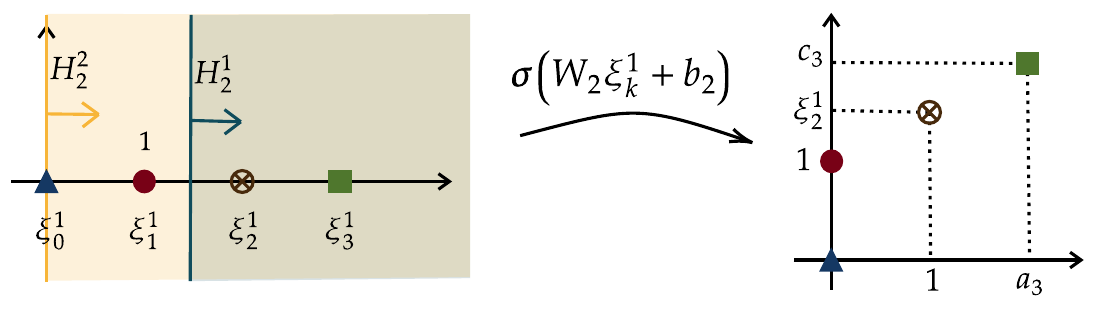}
    \caption{ The hyperplanes $H_2^1$ and $H_2^2$ are respectively defined by the equations $w_2^1 x+b_2^1=0$ and $w_2^2 x+b_2^2=0$.} \label{figure_steps2_label}
\end{figure}

\noindent \emph{Data projection. } Let us consider the parameters 
$
    w_3=(2-\xi_2^1,1)$ and $ b_3=0,
$
and define 
$
    \xi_k^3=\sigma(w_3\cdot\xi_k^2+b_3).
$
Clearly, we have that $\xi_k^3=k$ for all $k\in\llbracket0,2\rrbracket$. 
\begin{figure}[H]
    \centering
    \includegraphics[width=0.6\textwidth]{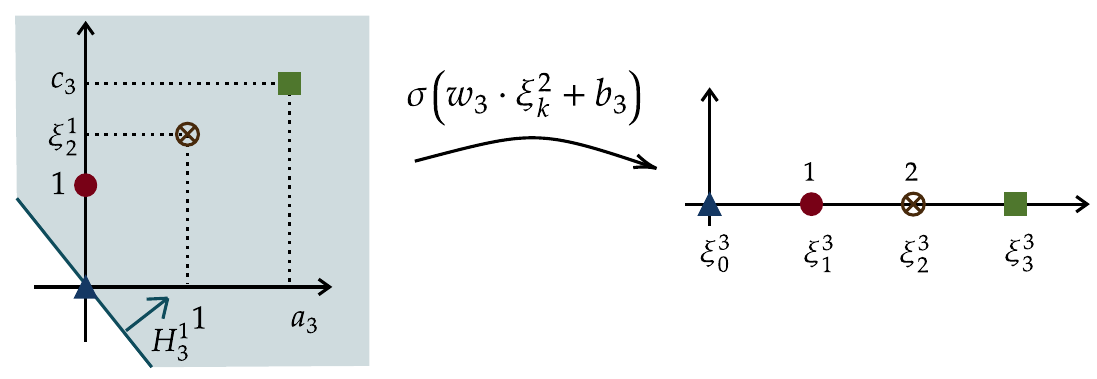}
    \caption{ The hyperplane $H_3^1$ is defined by the equation $w_3\cdot x+b_3=0$. } \label{figure_steps3_label}
\end{figure}

\begin{remark}\label{remark_42}
    We have applied $\vecsigma$ (or $\sigma$) three times to sort the first three points. However, to sort one more point, only two further steps are needed: data structuring and projection. Therefore, to sort $n\geq 2$ points (from the configuration of the previous step), $1+2(n-2)=2n-3$ applications of  $\vecsigma$ (or $\sigma$) are needed.
\end{remark}

\noindent {\bf (2) Inductive Step: }Consider $\eta\in \llbracket0,M-1\rrbracket$ and define $l:=1+2(\eta-2)$. Let us assume that $\xi_{k}^{l}=k$ for all $k\in \llbracket0,\eta\rrbracket$. We will show that there exist parameters such that $\xi_{k}^{l+2}=k$ for all $k\in \llbracket0,\eta+1\rrbracket$. We will proceed in two steps: data structuring and data projection.
\smallbreak
    \emph{Data structuring. }Let us consider the parameters
\begin{align*}
w_{l+1}^1 = \frac{2}{\xi_{\eta+1}^{l}-\xi_{\eta}^{l}},\quad w_{l+1}^2 =1,\quad b_{l+1}^1=-\left(\frac{\xi_{\eta+1}^{l}+\xi_{\eta}^{l}}{\xi_{\eta+1}^{l}-\xi_{\eta}^{l}}\right),\quad \text{and}\quad b_{l+1}^2=0.
\end{align*}
Define $W_{l+1}=(w_{l+1}^1,w_{l+1}^2)^\top
$ and $b_{l+1}=(b_{l+1}^1,b_{l+1}^2)^\top
$, and consider 
\begin{align*}
    \xi_k^{l+1}=\vecsigma(W_{l+1}\xi_k^{l}+b_{l+1})
\end{align*}
so that $\xi_{\eta+1}^{l+1}=(1,\xi_{\eta+1}^{l})$ and
\begin{align*}
     &\xi_k^{l+1}=(0,k),\quad \text{for all } k\in\llbracket0,\eta\rrbracket,\\
     &\xi_k^{l+1}=(a_k,c_k),\quad \text{for all } k\in\llbracket\eta+2,M-1\rrbracket,
\end{align*}
where $\{a_k\}_{k=\eta+2}^{M-1}$ and $\{c_k\}_{k=\eta+2}^{M-1}$ are two sequences of strictly increasing numbers such that $a_k>1$ and $c_k>\xi_{\eta+1}^{l}$ for all $k\geq \eta+2$.

\noindent \emph{Data projection. } Finally, define 
$
\xi_k^{l+2}=\sigma(w_{l+2}\cdot \xi_k^{l+1}+b_{l+2})
$
with
\begin{align*}
    w_{l+2}=((\eta+1)-\xi_{\eta+1}^{l},1),\quad\text{and }\quad b_{l+2}=0
\end{align*}
so that $\xi_k^{l+2}=k$ for all $k\in\llbracket0,\eta+1\rrbracket$. 
\smallbreak
Observe that, in order to drive the $M$ points to their respective labels, we need to apply $L_3=2M-3$ steps.
\medbreak

Summarizing, the input-output map $\phi^L$ of Theorem \ref{multiclass_theorem}, for $N$ points with $M$ classes, is given by the composition of the mappings $\phi_i^{L_i}$ given by \eqref{phi_L0}, \eqref{phi_l1}, \eqref{goal_step_3} and \eqref{phi_l3} respectively, i.e.,
\begin{align*}
    \phi^L=(\phi^{L_3}\circ\phi^{L_2}\circ\phi^{L_1}\circ\phi^{L_0}),
\end{align*}
 with $L=L_0+L_1+L_2+L_3= 1 + 2N + (2M+1) +(2M-3) = 2N+4M-1$. 
\bigbreak

\noindent
\textbf{Proof of Corollary \ref{coro:estimation_norms}:} Now, we continue with the proof of Corollary \ref{coro:estimation_norms}. To this end, let us recall that we have assumed 
$\{x_i,y_i\}\subset B_{R_x}^d(0)\times B_{R_y}^1(0)$, with $R_x,R_y>0$. We proceed step by step to estimate the norms.

\smallbreak
\noindent\textbf{Preconditioning of the data:} Denote by $(\mathcal W_1,\mathcal B_1)=(W_1,b_1)$ the parameters in this step.  Recall that $\|w_1\|_2=1$ and we take $b_1=2R_x$, so that 
\begin{align*}
    \|W_1\|_F=\|w_1\|_2=1,
    \quad
    \|b_1\|_2=2R_x,
    \quad
    \|W_1\|_\infty\le1,
    \quad
    \|b_1\|_\infty=2R_x.
\end{align*}
Hence, we have that
\begin{align}\label{eq:estimation_1_l2}
\tnorm{(\mathcal{W}_1,\mathcal{B}_1)}_2^2
    &=\sum_{j=1}^1\bigl(\|W_j\|_F^2+\|b_j\|_2^2\bigr)
    =1^2+(2R_x)^2
    =1+4R_x^2,
\end{align}
and for the $l^\infty-$ norm we get 
\begin{align}\label{eq:estimation_1_linf}
\tnorm{(\mathcal{W}_1,\mathcal{B}_1)}_\infty
    &=\max_{j\in\{1\}}\{\|W_j\|_\infty,\|b_j\|_\infty\}
    =\max\{1,2R_x\}
    =2R_x.
\end{align}

\smallbreak
\noindent
\textbf{Compression step:} The total depth of the network is $L_2 = 2N$. For each $j \in \{0, \dots, N - 1\}$, we first analyze the  data structuring layer ($W_{2j}$, $b_{2j}$):
\begin{align*}
W_{2j} = \begin{pmatrix} 1 \\ -1 \end{pmatrix}, \quad \text{so that} \quad \|W_{2j}\|_F^2 = 2, \quad \|b_{2j}\|_2^2 \leq 2M_1^2.
\end{align*}
Now, for the Compression layer ($W_{2j+1}$, $b_{2j+1}$):
\begin{align*}
W_{2j+1} = \begin{pmatrix} A_j & 1 \\ -B_j & -1 \end{pmatrix}, \quad \text{with} \quad A_j, B_j \leq 1, \quad \|W_{2j+1}\|_F^2 \leq 4, \quad \|b_{2j+1}\|_2^2 \leq 2M_1^2.
\end{align*}
Here $M_1=\max_i\{\|W_1x_1+b_1\|_{\infty}\}$ and therefore $M_1 \leq 3R_x$. To compute the full norm in this step, we sum over $j$ from $0$ to $N-1$, obtaining
\begin{align}\label{eq:estimation_2_l2}
\nonumber\tnorm{(\mathcal{W}_2, \mathcal{B}_2)}_2^2 &= \sum_{j=0}^{N-1} \left( \|W_{2j}\|_F^2 + \|b_{2j}\|_2^2 + \|W_{2j+1}\|_F^2 + \|b_{2j+1}\|_2^2 \right) \\
&= N(2 + 2M_1^2) + N(4 + 2M_1^2) = N(6 + 4M_1^2).
\end{align}
For the $\ell^\infty$-norm we have that $\|W_j\|_\infty \leq 1,$ and $\|b_j\|_\infty \leq M_1,$ thus
\begin{align}\label{eq:estimation_2_linf}
{\tnorm{(\mathcal{W}_2, \mathcal{B}_2)}_\infty \leq3R_x}.
\end{align}
\medbreak

\noindent
\textbf{Data sorting:} Let $R_z := \max_k \|{z_k}\|_\infty$ being $\{z_k\}$ the output from the compression step. The network depth is $L_3 = 2M + 1$. We consider projection layers indexed by $j \in\{ 0, 2, \dots, 2M\}$ and data structuring layers indexed by $j \in\{1, 3, \dots, 2M - 1\}$. For each projection layer $W_{2j}$ and $b_{2j}$:
\begin{align*}
\|W_{2j}\|_F^2 = 1, \quad \|b_{2j}\|_2^2 \leq R_z^2.
\end{align*}
There are $M + 1$ such layers. For each structuring layer $W_{2j+1}$ and $b_{2j+1}$:
\begin{align*}
\|W_{2j+1}\|_F^2 = 2, \quad \|b_{2j+1}\|_2^2 \leq 2R_z^2.
\end{align*}
There are $M$ such layers. The total norm is
\begin{align*}
\tnorm{(\mathcal{W}_3, \mathcal{B}_3)}_2^2 &= \sum_{j=0}^{M} \left( \|W_{2j}\|_F^2 + \|b_{2j}\|_2^2 \right) + \sum_{j=0}^{M-1} \left( \|W_{2j+1}\|_F^2 + \|b_{2j+1}\|_2^2 \right) \\
&= (M+1)(1 + R_z^2) + M(2 + 2R_z^2) = (3M + 1)(1 + R_z^2).
\end{align*}
To estimate $R_z$, let $z_i^{(0)} = x_i^1$ be the output of Step 1, and assume $\|z_i^{(0)}\|_\infty \leq M_1$ for all $i$. Then each iteration in Step 2 satisfies:
\begin{align*}
\|z_i^{(\ell)}\|_\infty \leq\|W_\ell\|_\infty \cdot \|z_i^{(\ell-1)}\|_\infty + \|b_\ell\|_\infty \leq \|z_i^{(\ell-1)}\|_\infty + M_1.
\end{align*}
By induction over $2N$ layers, we obtain:
\begin{align*}
R_z := \max_i \|z_i^{(2N)}\|_\infty \leq (2N + 1) M_1 \leq 2R_x(2N + 1).
\end{align*}
Using the bound for $R_z$ derived above,
\begin{align}\label{eq:estimation_3_l2}
{\tnorm{(\mathcal{W}_3, \mathcal{B}_3)}_2^2 \leq {(3M + 1)(1 + 4R_x^2(2N + 1)^2)}}.
\end{align}
For the $\ell^\infty$-norm, since all parameters are bounded by $R_z$, we deduce
\begin{align}\label{eq:estimation_3_linf}
{\tnorm{(\mathcal{W}_3, \mathcal{B}_3)}_\infty \leq 2R_x(2N + 1)}.
\end{align}

\textbf{Mapping to the labels:} We denote by $(\mathcal{W}_4, \mathcal{B}_4)$ the collection of parameters used in the final step of the construction. Recall that the total number of layers in this step is $L_4 = 2M - 3$. The first three layers implement the initial projection and positioning of the first three values. All parameters involved in these layers are explicitly defined and uniformly bounded in terms of the geometry of the values $\{\xi_k\}_{k=0}^{M-1}$. Therefore, there exists a constant $C_\xi > 0$, depending only on the relative distance of the data $\{x_i\}_{i=1}^N$, such that
\begin{align*}
    \|W_j\|_F^2 + \|b_j\|_2^2 \leq C_\xi(1+R_y^2),\quad \|W_j\|_\infty, \|b_j\|_\infty \leq \max\{1,C_\xi\},\quad \text{for } j=1,2,3.
\end{align*}
In the inductive step, we have that for each $\eta\in\llbracket2,M-2\rrbracket$, two layers are required to map $\xi_{\eta+1}\mapsto \eta+1$. This is done with the  Data structuring layer, for which we have
Parameters satisfy again
\begin{align*}
    \|W_j\|_F^2 + \|b_j\|_2^2 \leq C_\xi\,(1 + R_y^2),
    \quad
    \|W_j\|_\infty, \|b_j\|_\infty \leq \max\{1,\,C_\xi\}.
\end{align*}
Then we apply the Projection layer with weight $w=(\eta+1-\xi_{\eta+1},1)$, and bias $0$, and hence
\begin{align*}
    \|W_j\|_F^2 = (\eta+1 - \xi_{\eta+1})^2 + 1 \leq (M + R_y)^2, \quad\|b_j\|_2^2 = 0,\quad
    \|W_j\|_\infty &\leq M + R_y,  \quad
    \|b_j\|_\infty = 0.
\end{align*}
Then, putting all together, we have three initial layers plus $(M-3)$ inductive steps. Therefore, we deduce
\begin{align}\label{eq:estimation_4_l2}
    \nonumber\tnorm{(\mathcal{W}_4,\mathcal{B}_4)}_2^2
    &= 3\,C_\xi\,(1 + R_y^2) + (M-3)\bigl[C_\xi\,(1 + R_y^2) + (M + R_y)^2\bigr]\\
    &= M\,C_\xi\,(1 + R_y^2) + (M-3)\,(M + R_y)^2.
\end{align}
Analogously for the $\ell^\infty$ norm, taking the maximum over all layers yields
\begin{align}\label{eq:estimation_4_linf}
    {
    \tnorm{(\mathcal{W}_4,\mathcal{B}_4)}_\infty
    \le \max\bigl\{\max\{1,C_\xi\},\,M + R_y\bigr\}.
    }
\end{align}
Finally, combining estimation \eqref{eq:estimation_1_l2}-\eqref{eq:estimation_2_l2}-\eqref{eq:estimation_3_l2}-\eqref{eq:estimation_4_l2}, we deduce that 
\begin{align*}
\tnorm{(\mathcal W,\mathcal B)}_2^2
&\le
\bigl(1 + 4R_x^2\bigr)
+ N\bigl(6 + 36R_x^2\bigr)
+ (3M + 1)\bigl[\,1 + 9R_x^2(2N+1)^2\,\bigr]\\
&\quad + \bigl[M\,C_\xi(1+R_y^2) + (M-3)(M+R_y)^2\bigr],
\end{align*}
and for the $l^{\infty}-$norm we combine estimations \eqref{eq:estimation_1_linf}-\eqref{eq:estimation_2_linf}-\eqref{eq:estimation_3_linf}-\eqref{eq:estimation_4_linf} to obtain
\begin{align*}
    \tnorm{(\mathcal W,\mathcal B)}_\infty
&=
\max\bigl\{\,2R_x,\;3R_x,\;3R_x(2N+1),\;\max\{C_\xi,\,M+R_y\}\bigr\}.
\end{align*}
In particular, we can guarantee the existence of a constant $C>0$ independent of $N,\,M,\,R_x,\,R_y$ such that
\begin{align*}
\tnorm{(\mathcal W,\mathcal B)}_2
&\le
C(1 + R_x\sqrt{N} + R_x N \sqrt{M} + R_y M),
\end{align*}
and 
\begin{align*}
    \tnorm{(\mathcal W,\mathcal B)}_\infty
&\le
C\bigl(R_x N + M + R_y\bigr).
\end{align*}

\section{Universal approximation theorem}
In this section, we prove the Universal Approximation Theorem in $L^p(\Omega;\R_{+})$ (\Cref{UAT_LP}).

\begin{proof} We proceed according to the \textit{Strategy of the proof} after \Cref{UAT_LP}, in \Cref{sec:main_results}.

\medbreak

\noindent {\bf Step 1 (Hyperrectangles construction):} Let us consider $\mathcal{C}$, the smallest hyperrectangle containing $\Omega$, oriented according to the axes of the canonical basis of $\R^d$. Consider $0<h<1$ and $0<\delta\ll h$. Define an equispaced grid $G_\delta^h\subset \R^d$, of thickness $\delta$, oriented according to the axes of the canonical basis of $\R^d$. Define the family of hyperrectangles $\mathcal{H}=\{\mathcal{H}_i\}_{i=1}^{N_h}$ as $\mathcal{H}=\mathcal{C}\setminus G_\delta^h$. We also consider $\mathcal{H}^{G}:=\{\mathcal{H}^G_i\}_{i=1}^{N^G_h}$ a family of hyperrectangles such that  
\begin{align*}
G_\delta^h=\bigcup_{i=1}^{N_G}\mathcal{H}_i^G,
\end{align*}
see Figure \ref{fig:compres_1}. Note that the family $\mathcal{H}$ depends on $h$, and the family $\mathcal{H}^G$ depends on $h$ and $\delta$; however, this dependency will be omitted to simplify the notation.

The number of hyperrectangles $N_h$ on $\mathcal{H}$ satisfies    
\begin{align}\label{estimation_NG}
    N_{h}\leq h^{-d}C_{\Omega},
\end{align}
with $C_{\Omega}$ a constant depending on $m_d(\Omega)$, where $m_d(\cdot)$ denotes the Lebesgue measure in $\R^d$. Taking into account that the number of edges of a $d$-dimensional hypercube is $2d(d-1)$, the Lebesgue measure $m_d(G^h_\delta)$ of the grid $G_\delta^h$ intersecting $\mathcal{C}$ is bounded by
\begin{align}\label{eq:estimation_measure_G_h_d}
    m_d(G^h_\delta)\leq C_{\Omega, d} \delta (h+\delta)^{d-1}h^{-d}.
\end{align}
Thus, for any $\gamma>0$, taking $\delta=h^{1+\gamma}$, the volume of $G_\delta^h$ tends to zero as $h\to 0$. In the following, we will take $\gamma=p$. 
\begin{figure}
\centering
\subfloat[]{
\includegraphics[width=0.4\textwidth]{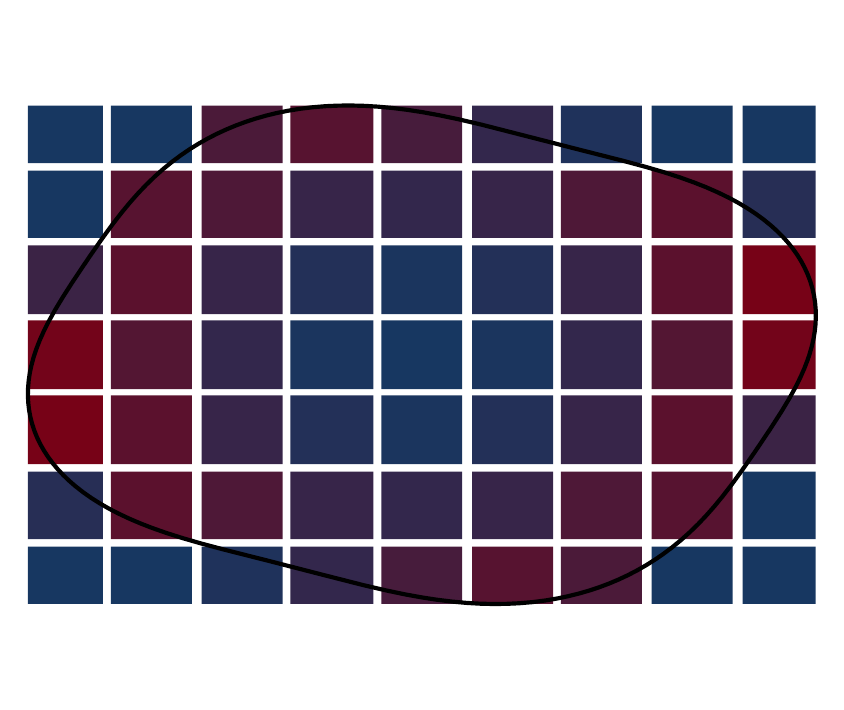}}
\subfloat[]{
\includegraphics[width=0.4\textwidth]{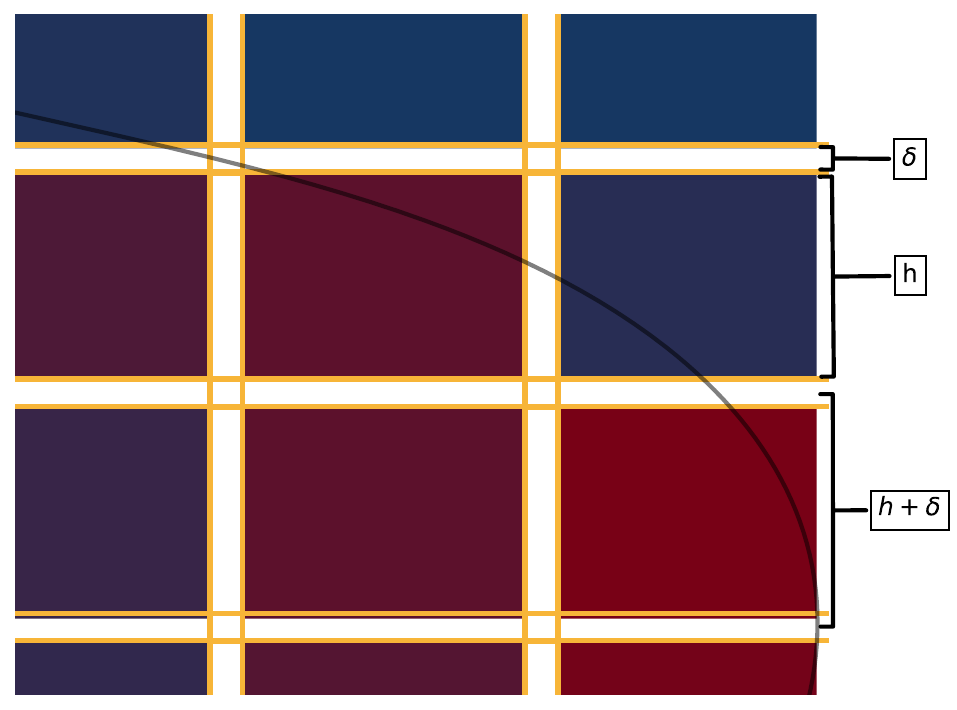}}
\caption{(A) Illustration of the simple function $f_h$ on $\mathcal{H}$, defined in \eqref{eq:simple_function_def}. (B) The main features of $G^h_\delta$ are represented: The mesh thickness $\delta$, the size $h>0$, and the hyperrectangles $\mathcal{H}^G_i$, illustrated with orange boundary.}\label{fig:compres_2}
\end{figure}
\smallbreak

\medbreak

Let us fix a function $f\in L^p(\Omega;\R_{+})$. Extending it by zero, we assume that $f\in L^p(\mathcal{C};\R_{+})$. By the density of simple functions, we know that $f$ can be approximated by a sequence of simple functions. In particular, we can construct a simple function supported on hyperrectangles as follows: Let us consider the constants
\begin{align*}
    f_{1,i}^h:=\frac{1}{m_d(\mathcal{H}_{i})}\int_{\mathcal{H}_i}f(x)\,dx, \quad \text{for }i\in\llbracket1,N_h\rrbracket,
\end{align*}
and 
\begin{align*}
    f_{2,i}^h:=\frac{1}{m_d(\mathcal{H}_{i}^G)}\int_{\mathcal{H}_i^G}f(x)\,dx, \quad \text{for }i\in\llbracket1,N_G\rrbracket.
\end{align*}
That is, $f_{1,i}^h$ and $f_{2,i}^h$ are the average value of the function $f$ in the hyperrectangle $\mathcal{H}_i$ and $\mathcal{H}_i^G$, respectively. Then, we introduce the simple function
\begin{align}\label{eq:simple_function_def}
    f_h(x)=\sum_{i=1}^{N_h} f_{1,i}^h \chi_{\mathcal{H}_i}(x) +\sum_{i=1}^{N_G} f_{2,i}^h \chi_{\mathcal{H}_i^G}(x),
\end{align}
where $\chi_{\mathcal{H}_i}$ denotes the characteristic function on the set $\mathcal{H}_i$, similar for $\mathcal{H}_i^G$. In the following, we denote by $M_{h}>0$ the number of values that the function $f_h$ takes on the family of hyperrectangles $\mathcal{H}$. Note that $M_h\leq N_h$. 

Let us observe that by the Lebesgue differentiation theorem, the sequence $f_h$ approximates $f$ a.e. on $\mathcal{C}$ as $h\to 0$, and therefore, due to the dominated convergence theorem, we have that $\|f-f_h\|_{L^p(\mathcal{C};\R_+)}\to0$ as $h\to 0$. In particular, for all $\varepsilon>0$ there exists $h_1>0$ small enough such that for every $0<h<h_1$ we have 
\begin{align}\label{eq:estima_simple_function}
    \|f-f_{h}\|_{L^p(\mathcal{C};\R_+)}<\varepsilon/2.
\end{align}
Moreover, as shown in \cite[Section 6.2]{MR1689432}, if $f \in W^{1,p}(\Omega; \mathbb{R}_+)$, there exists a constant $C > 0$ independent of $h > 0$ such that
\begin{align}\label{eq:estimation_f_h}
    \|f - f_h\|_{L^p(\Omega; \mathbb{R}_+)} \leq C \max\{\text{diam}(\mathcal{H}), \text{diam}(\mathcal{H^G})\} \|f\|_{W^{1,p}(\Omega; \mathbb{R}_+)}
\end{align}
where $\text{diam}(\mathcal{H}) = \max_{\mathcal{H}_i \in \mathcal{H}} \{\text{diam}(\mathcal{H}_i)\}$. 
Since $\text{diam}(\mathcal{H^G})< \text{diam}(\mathcal{H})$, and $\text{diam}(\mathcal{H})$ is at most $\sqrt{d}h > 0$, \eqref{eq:estimation_f_h} reduces to
\begin{align}\label{eq:estimation_f_h2}
    \|f - f_h\|_{L^p(\Omega; \mathbb{R}_+)} \leq C h \|f\|_{W^{1,p}(\Omega; \mathbb{R}_+)}
\end{align}
Thus, estimate \eqref{eq:estima_simple_function} is ensured by taking 
\begin{align}\label{values_h1}
 h_1 = \frac{\varepsilon}{2C \|f\|_{W^{1,p}(\Omega; \mathbb{R}_+)}}.
\end{align}

        \medbreak
       \noindent  { \bf Step 2 (Approximation of $f_h$ using a neural network):}
         In this step, we will construct a neural network approximating the simple function $f_h$. This is done by mapping the hyperrectangles of $\mathcal{H}$ into the $M_h$ values of $f_{h}$ via a neural network.  
        
        {\bf Step 2.1 (Compresion of one hyperrectanle)}: In the same spirit as the compression process in Section \ref{subsec:compression}, we first show that a single $\mathcal{H}_i$ can be compressed without mixing the other hyperrectangles. This allows compressing the whole family $\{\mathcal{H}_i\}_{i=1}^{N_h}$. 
        
        This is done in two stages. First, we apply a compression process driving the $d-$dimensional hyperrectangle into a $(d+1)-$dimensional Euclidean space, allowing us to drive a hyperrectangle to a point. In the second stage, we project the data to the $d-$dimensional space, keeping the structure of the hyperrectangle.
        
       \medbreak
        
   \noindent  \emph{\bf Step 2.1.1 (First layer):} Let us suppose we want to compress a fixed hyperrectangle $\mathcal{H}_{*}$ that is located on one edge of the hyperrectangle $\mathcal{C}$, and oriented in a canonical direction $e_{\hat{\eta}}$ for some  $\hat{\eta}\in \llbracket1,d\rrbracket$. We consider the following family of hyperplanes      \begin{align}\label{hiper_approximation}  
   \nonumber &H_{\eta}:=\{x\in\R^d\,:\,e_{\eta} x+b_\eta=0\},\quad \text{for } \eta\in\llbracket1,d\rrbracket,\\
    &H_{d+1}=\{x\in\R^d\,:\,-e_{\hat{\eta}} x+b_{\hat{\eta}}^*=0\},
        \end{align}
        and the respective activation regions (see Section \ref{single_hyper})
        \begin{align}\label{eq:regions_uat}
            \nonumber&R_\eta:=\{x\in\R^d\,:\,e_{\eta}x+b_\eta\geq 0 \},\quad \text{for } \eta\in\llbracket1,d\rrbracket,\\           
            &R_{d+1}:=\{x\in\R^d\,:\,-e_{\hat{\eta}}x+b_{\hat{\eta}}^*\geq 0 \}.
        \end{align}
       In particular, observe that \begin{align}\label{eq:empty_regions}
            R_{\hat{\eta}}\cap R_{d+1}  =\varnothing.
        \end{align}
        Here, we have one hyperplane for each canonical direction and one extra hyperplane in the direction $e_{\hat{\eta}}$, so there are in total $d+1$ hyperplanes. The constants $b_k$ are chosen in such a way that the hyperplanes are placed around the hyperrectangle $\mathcal{H}_*$, that is, $b_k$'s are taken such that 
        \begin{equation*}
    \sigma(e_{\eta}x+b_\eta)=0,\quad
\sigma(-e_{\hat{\eta}}x+b_{\hat{\eta}}^*)=0,
\quad \text{for all } x\in \mathcal{H}_{*},\, \eta\in\llbracket1,d\rrbracket.
        \end{equation*}

 Let us apply the map defined by the parameters from the hyperplanes \eqref{hiper_approximation}. This defines the new family of hyperrectangles $\{\mathcal{H}_i^1\}_{i=1}^N$. Namely, denote by $W^1\in\R^{d+1\times d}$ and $b^1\in\R^{d+1}$ the matrices given by
     \begin{align*}
       W^1=\begin{pmatrix}
           e_1| e_2| \dots| e_{\hat{\eta}}|(-1)e_{\hat{\eta}}| \dots| e_d \end{pmatrix}^{\top}, \quad \quad b^1=\begin{pmatrix}
           b_1| b_2| \dots| b_{\hat{\eta}}| b_{\hat{\eta}}^{*}| \dots | b_{d} \end{pmatrix}^{\top}.
     \end{align*}
     Then, the new family of hyperrectangles is given as
     \begin{align}\label{relation_x1_x}
         x_1=\vecsigma(W^1x+b^1)\in \mathcal{H}_i^1,\quad \text{for all }x\in \mathcal{H}_i,\, i\in \llbracket1,N\rrbracket.
     \end{align}     
   An illustration of the family $\{\mathcal{H}_i^1\}_{i=1}^N$ for a $2-$dimensional example is given  in Figure \ref{fig:complete_compress} (Appendix \ref{appendix_2}).

\medbreak

     \noindent \emph{\bf Step 2.1.2 (Second layer):} The previous selection of parameters can drive all points in $\mathcal{H}_{*}$ into $\mathcal{H}^1_{*} =\{{\bf 0}_{d+1}\}$, where ${\bf 0}_{d+1}$ denotes the null vector in $\mathbb{R}^{d+1}$, and since $d+1$ hyperplanes have been employed, our hyperrectangles are carried into a $d+1$-dimensional space. To recursively apply this process, we need to project the hyperrectangles into a $d$-dimensional space, ensuring that the previous steps define an injective mapping for the hyperrectangles, i.e., distinct hyperrectangles are carried to distinct locations without mixing them.
        
      For this purpose, let $I_n \subset \llbracket1,d+1\rrbracket$ be a set of indices with $n\geq1$ elements. We introduce the subregions $\mathcal{S}_{I_n}$ defined as
\begin{align}\label{definition_subregions}
\mathcal{S}_{I_n}=\left\{x\in \mathbb{R}^d: x\in\bigcap_{\eta\in I_n} R_{\eta} \text{ and } x\notin \bigcup_{\eta\in I_n^c} R_{\eta}\right\},
\end{align}
where $I_n^c$ is the complement of $I_n$ with respect to $\llbracket1,d+1\rrbracket$. Observe that if $x\in \mathcal{S}_{I_n}$, then $x$ belongs to $n$ regions. Likewise, if $\mathcal{H}_i\subset \mathcal{S}_{I_n}$, $\mathcal{H}_i$, it is in $n$ regions. Due to the construction of the hyperplanes, a hyperrectangle can belong to at most $d$ regions. See \Cref{fig:compres_3} (A). 
\begin{figure}
    \centering
\includegraphics[scale=0.22]{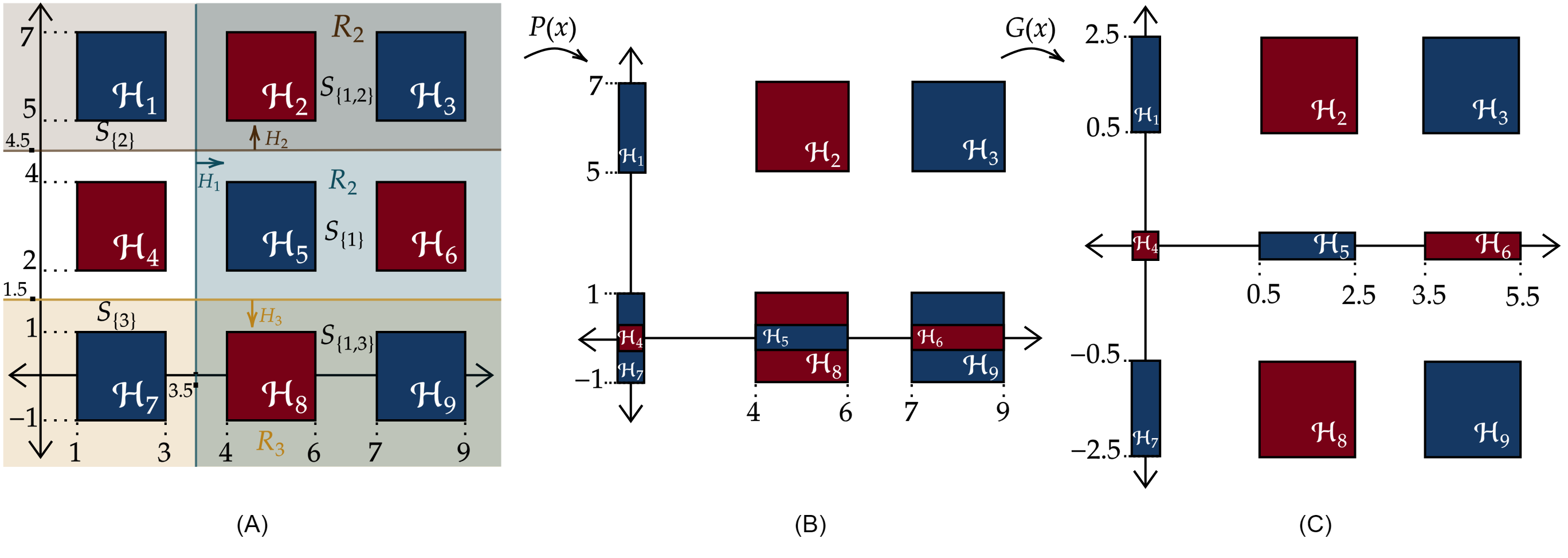}
    \caption{(A) We illustrate how the hyperplanes $H_\eta$ (three of them in this $2-$dimensional setting) enclose the hyperrectangle $\mathcal{H}_{4}$. Additionally, this defines the regions $R_\eta$ and the subregions $\mathcal{S}_{I_n}$ introduced in \eqref{eq:regions_uat} and  \eqref{definition_subregions}, respectively. Here, we can observe that the parameters of $H_1$ are $w_1=(1,0)$ and $b_1=-3.5$, for $H_2$ are $w_2=(0,1)$ and $b_2=-4.5$, and for $H_3$, $w_3=-(0,1)$ and $b_3=1.5$ . (B) We show how the function $P(x)$, defined in \eqref{eq:definition_F}, maps the hyperrectangles. In (B), we keep the same labels for each hyperrectangle as in (A), but it needs to be understood as $P(\mathcal{H}_i)$ for every $i$. The rectangles here, as $\mathcal{H}_1$, illustrate $1-$dimensional hyperrectangles (lines). For $\mathcal{H}_1$, we have that $p_1(x)=p_3(x)=0$ and $p_2(x)=1$ for all $x\in \mathcal{H}_1$, thus $P(x)=(0,x^{(2)})$. The small square illustrates a $0-$dimensional hyperrectangle, i.e., a point. The chosen parameters ensure that $P(\mathcal{H}_4)=\{{\bf0}_d\}$. (C) Here, we illustrate the image of $G$, taking as an input $P(\mathcal{H})$. Observe that $G$ translates the hyperrectangles such that there is no overlapping between them, recovering the same structure of (A), but with some hyperrectangles of smaller dimensions. In particular, those that belong to one subregion in (A) in (C) become lines. The function $F$ introduced in $\eqref{eq:definition_F}$ maps the hyperrectangles in (A) to the hyperrectangles in (C).
     }
    \label{fig:compres_3}
\end{figure}

Additionally, if $x$ belongs to $n$ regions, $\vecsigma(W^1x+b^1)$ is a vector with $n$ non zero coordinates. This motivates the introduction of the pattern activation function $p_\eta:\R^d\to\R$ defined as 
        \begin{align*}
            p_{\eta}(x)=\begin{cases}
                1\quad \text{if }x\in R_\eta,\\
                0\quad \text{if }x\notin R_\eta,
            \end{cases}\quad\text{for every }\eta\in\llbracket1,d+1\rrbracket.
        \end{align*}
        Since \eqref{eq:empty_regions}, we have that either \begin{align}\label{property_p_chiquita}
            p_{\hat{\eta}}(x)=0\quad \text{or}\quad p_{d+1}(x)=0,\quad \text{for every }x\in\R^d.
        \end{align}   
Furthermore, $p_{\eta}(x)=1$ if $x\in\mathcal{S}_{I_n}$ with $\eta\in I_n$, and $p_{\eta}(x)=0$ otherwise.
Let us introduce $F:\R^d\to\R^d$ defined as
\begin{align}\label{eq:definition_F}
F(x)=P(x)+G(x)=\left(\sum_{\eta=1}^{d+1} p_{\eta}(x)x^{(\eta)}e_{\eta}\right) +\left(\sum_{\eta=1}^{d+1}p_\eta(x)b_\eta e_\eta\right).
\end{align}
Figure \ref{fig:compres_3} illustrates how the map $F(x)$ acts on the hyperrectangles $\mathcal{H}$ in a two-dimensional example. The following lemma, proved in \Cref{appendix_2}, gives us important properties of $F$.

\begin{lemma}\label{lemma:all_subregions}
We have that $F(\mathcal{H}_*)= {\bf 0}_{d}$, where ${\bf 0}_{d}$ denotes the null vector in $\mathbb{R}^{d}$. Moreover, for any hyperrectangles $\mathcal{H}_i,\mathcal{H}_j\subset \mathcal{H}$ with $i\neq j$, we have that $F(\mathcal{H}_i)\cap F(\mathcal{H}_j)=\emptyset$. 
\end{lemma}

For each $\mathcal{H}_i$, there are two possibilities for $F(\mathcal{H}_i)$: it either retains the shape of a hyperrectangle, in which case $P(\mathcal{H}_i)$ is equal to $\mathcal{H}_i$ and $G(\mathcal{H}_i)$ acts as a translation, or it collapses into a lower-dimensional hyperrectangle, where $P(\mathcal{H}_i)$ projects the hyperrectangle and $G(\mathcal{H}_i)$ translates it. In fact, the dimension of the resulting hyperrectangle depends on the number of regions the original hyperrectangle $\mathcal{H}_i$ is associated with. See Figure \ref{fig:compres_3} for an illustration.

 \smallbreak
     Consider  $W^2:=(W^1)^\top \in\R^{d\times d+1}$ as the transpose of $W^1$, and a vector $b^2\in\R^d$ such that 
     \begin{align*}
         e_k(W^2x_1+b^2)>0, \quad \text{for all } x_1\in \mathcal{H}_i^1,\, i\in \llbracket1,N_h\rrbracket,
     \end{align*}
     and for all $k\in \llbracket1,d\rrbracket$. The family of hyperrectangles $\{\mathcal{H}_i^2\}_{i=1}^{N}$ is defined by
    \begin{align}\label{relation_x2_x1}
         x_2=\vecsigma(W^2x_1+b^2)\in \mathcal{H}_i^2,\quad \text{for all }x_1\in \mathcal{H}_i^1,\, i\in \llbracket1,N_h\rrbracket.
     \end{align} 
    Let us note that the hyperrectangles $\{\mathcal{H}_i^2\}_{i=1}^{N}$ are not mixed, that is, for $\mathcal{H}_i,\mathcal{H}_j\in \mathcal{H}$ with $i\neq j$ we have that $\mathcal{H}_i^2\cap \mathcal{H}_j^2=\emptyset$. 
    Furthermore, $\mathcal{H}_{*}$ is mapped to a single point.
    To verify the above, let us observe that  for $x\in\mathcal{H}$, we have 
    \begin{align}\label{H_2}
        x_2=\vecsigma(W^2\vecsigma(W^1x+b^1)+b^2) = W^2\vecsigma(W^1x+b^1)+b^2.
    \end{align}
   Let $x^{(i)}$ denote the $i-$th coordinate of $x$. Using the fact that 
$
        \vecsigma(W^1x^{(1)}+b^1)=p_1(x)(W^1x^{(1)}+b^1),
$
    for every $x\in\mathcal{H}_i$ and given \eqref{property_p_chiquita}, we obtain 
    \begin{align*}
        x_2&=W^2\begin{pmatrix}
            p_1(x)(e_1x+b_1)\\ p_2(x)(e_2x+b_2)\\ \vdots\\ p_{\hat{\eta}}(x)(e_{\hat{\eta}}x+b_{\hat{\eta}})\\ p_{d+1}(x)(-e_{\hat{\eta}}x+b_{\hat{\eta}}^*) \\ \vdots \\p_{d}(x)(e_dx+b_{d})
        \end{pmatrix} +b^2= \begin{pmatrix}
            p_1(x)(e_1x+b_1)\\ p_2(x)(e_2x+b_2)\\ \vdots \\ p_{\hat{\eta}}(x)(e_{\hat{\eta}}x+b_{\hat{\eta}})+p_{d+1}(x)(e_{\hat{\eta}}x-b_{\hat{\eta}}^*)\\ \vdots\\p_{d}(x)(e_dx+b_{d})
        \end{pmatrix} +b^2\\
        &= \left(\sum_{\eta=1}^d p_{\eta}(x)e_\eta x+p_{d+1}(x) e_{\hat{\eta}}x \right) +\left( \sum_{\eta=1}^d p_{\eta}(x)b_{\eta}e_\eta-p_{d+1}(x)b_{\hat{\eta}}^*e_{\hat{\eta}}\right)+ b^2=F(x) +b^2.
    \end{align*}
    The last constant, $b^2$, simply translates all hyperrectangles by the same magnitude. Therefore, due to Lemma \ref{lemma:all_subregions}, we conclude that the transformation $F(\mathcal{H}) + b^2$ does not mix the hyperrectangles, preserves their structure, and maps the hyperrectangle $\mathcal{H}_*$ to $\{\mathbf{0}_d\}$.
    
    For the example shown in \Cref{fig:compres_3}, the hyperrectangles $\{\mathcal{H}_i^2\}_{i=1}^{N_h}$ correspond to those in part (C), but translated to the positive quadrant $\R_+^2$. This translation is carried out by $b^2$, which, in this example, can be chosen as $b^2=(0.1, 2.6)$.

\medbreak
\noindent{\bf Step 2.2 (Compression of all hyperrectangles):} In the previous step, we successfully compressed $\mathcal{H}_*$ into a single point without mixing the other hyperrectangles. However, this process also slightly perturbed the remaining hyperrectangles, transforming them into hyperrectangles of lower dimensions. Specifically, hyperrectangles that were originally $n$-dimensional are now $(n-1)$-dimensional.

Let $E$ represent the set of edges of $\mathcal{C}$. We define $\{E_{i}\}_{i=1}^{d} \subset E$ as a set of orthogonal edges. The family of hyperrectangles $\mathcal{H}^E$ is then given by
$
    \mathcal{H}^E=\{ \mathcal{H}_i\in \mathcal{H}\,:\,\mathcal{H}_i \cap \{E_{i}\}_{i=1}^{d}\neq \emptyset\},
$
that is, the set of hyperrectangles in $\mathcal{H}$ intersecting at least one edge in $\{E_{i}\}_{i=1}^{d}$. We will prove that compressing the hyperrectangles in $\mathcal{H}^E$ into points also compresses all the hyperrectangles in $\mathcal{H}$ into points.
\smallbreak
    \begin{figure}
        \centering
\includegraphics[width=0.92\textwidth]{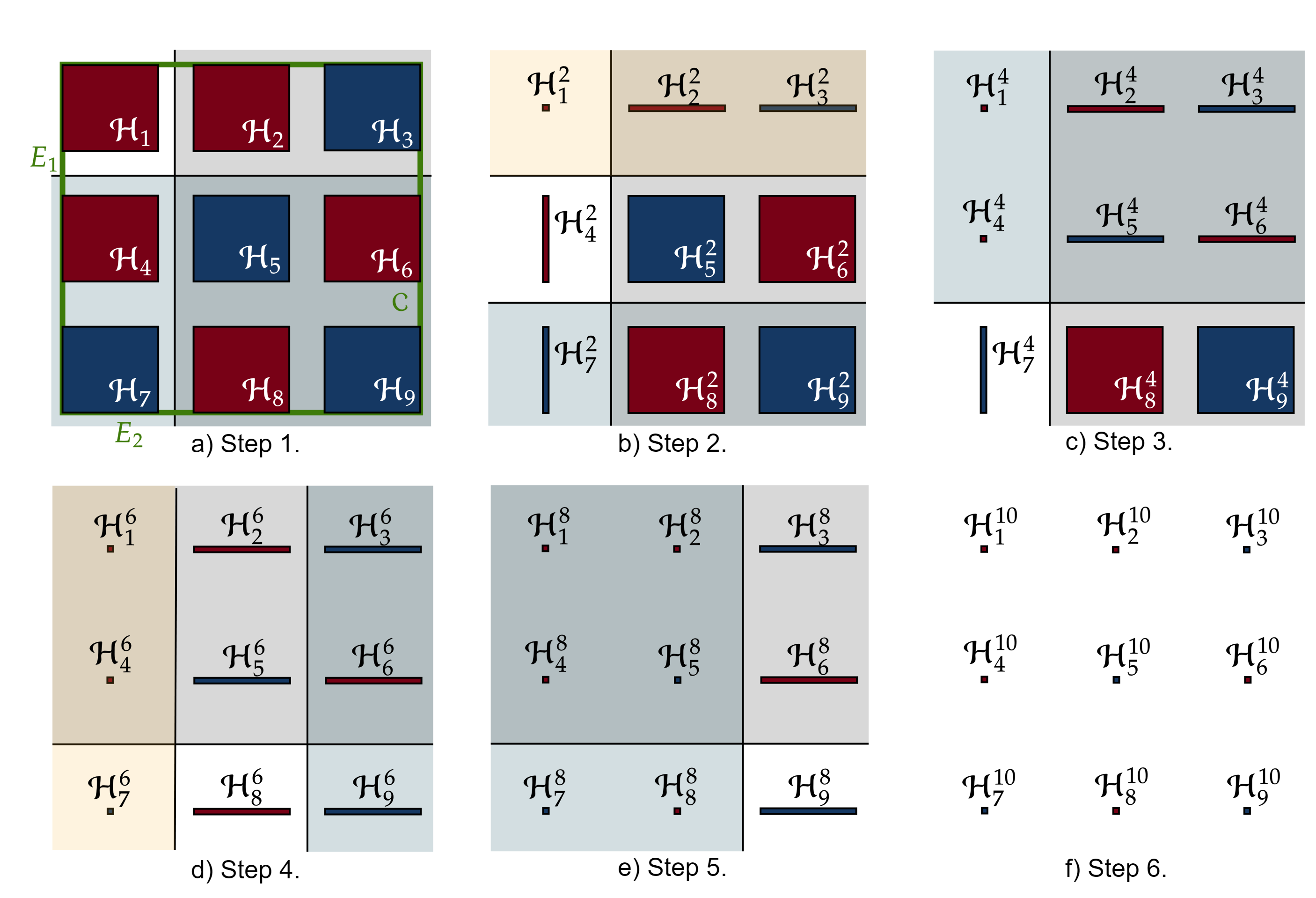}
        \caption{In the figure, we show an example of the compression process following the guidelines from Steps 2.1-2.2. First, in Step 1, we select two edges, $E_1$ and $E_2$ of $\mathcal{C}$, oriented according to different canonical vectors. Thus, $\mathcal{H}^E = \{\mathcal{H}_1, \mathcal{H}_4, \mathcal{H}_7, \mathcal{H}_8, \mathcal{H}_9\}$. Then, we compress the first hyperrectangle, $\mathcal{H}_1$, using two hyperplanes (a two-layer neural network with two neurons in the first layer and two neurons in the second layer). Here, $\mathcal{H}_2$ and $\mathcal{H}_3$ belong to $S_{\{1\}}$, and intersects $\mathcal{H}_4$ and $\mathcal{H}_7$  $S_{\{2\}}$. Next, in Step 2, we compress the second hyperrectangle on $E_1$, corresponding to $\mathcal{H}_4^2$. We use three hyperplanes to compress it (a two-layer neural network with three neurons in the first layer and two neurons in the second layer). Here, $\mathcal{H}_5^2$ and $\mathcal{H}_6^2$ belongs to $S_{\{1\}}$. In Step 3, we continue with $\mathcal{H}_7^4$. After compressing all the hyperrectangles along $E_1$, we have reduced every hyperrectangle to  $(d-1)$-dimensional objects ($d-1 = 1$ in this example). By applying the same process to the hyperrectangles along $E_2$ and continuing with Steps 4 and 5, we complete the compression process.}\label{fig:example_compresion}
    \end{figure}
    
We refer to Figure \ref{fig:example_compresion} for a graphical description of what follows.

Let $E_1$ be one of the selected orthogonal edges of $\mathcal{C}$. By Step 2.1, we can compress any hyperrectangle $\mathcal{H}_1 \in \mathcal{H}^E$ intersecting $E_1$ using a two-layer neural network. Moreover, observe that if a $d$-dimensional hyperrectangle $\mathcal{H}_i$ belongs to a subregion $S{I_n}$ with $n \leq d$, then it collapses into a $(d-1)$-dimensional hyperrectangle. Indeed, the condition $\mathcal{H}_i \subset S_{I_n}$ with $n \leq d$ implies that there exists an index $\bar{\eta} \in \llbracket 1, d+1 \rrbracket$ such that $p_{\bar{\eta}}(x) = 0$ for all $x \in \mathcal{H}_i$. Consequently, the $\bar{\eta}$-th coordinate of every point in $\mathcal{H}_i$ vanishes, collapsing $\mathcal{H}_i$ in to a $(d-1)-$dimensional hyperrectangle.

By applying the previous argument to all hyperrectangles in $\mathcal{H}^E$ intersecting $E_1$, we deduce that each $\mathcal{H}_i \in \mathcal{H}$ eventually enters a region $S_{I_n}$ with $n \leq d - 1$. As a result, there exists a coordinate $\bar{\eta} \in \llbracket 1, d+1 \rrbracket$ such that $p_{\bar{\eta}}(x) = 0$ for all $x \in \mathcal{H}_i$, and thus $\mathcal{H}_i$ collapses into a $(d-1)$-dimensional hyperrectangle. Then, we continue with the hyperrectangles that intersect $E_2$. Since $E_2$ is orthogonal to $E_1$, the vanishing coordinate is now $\tilde{\eta} \neq \bar{\eta}$, and each $(d-1)$-dimensional hyperrectangle is further collapsed into a $(d-2)$-dimensional hyperrectangle. Repeating this process for each edge $E_k$ with $k \in\{ 1, \dots, d\}$, and noting that each step corresponds to the vanishing of a new and different coordinate, we conclude that every $\mathcal{H}_i \in \mathcal{H}$ is eventually mapped to a $0$-dimensional object (a point). This concludes the compression process.

\medbreak

\begin{remark}
If we had not followed the edge-based strategy, we would have needed to compress hyperrectangles located inside $\mathcal{C}$. In the worst case, enclosing a $d$-dimensional hyperrectangle requires $2d$ hyperplanes, implying that the neural network would need to have width $2d$. By restricting the compression to hyperrectangles on the edges of $\mathcal{C}$, we reduce this requirement to $d+1$ hyperplanes per step. Therefore, this strategy allows us to reduce the width of the neural network from $2d$ to $d+1$ neurons per layer.
\end{remark}

\noindent {\bf Step 2.3 (Neuronal network construction):} Let $N_E^j$ be the  number of hyperrectangles on the edge $E_j$ for every $j\in \{1,\dots, d\}$. After applying Step 2, we have applied a two-layer neural network $N_E:=\sum_{j=1}^dN_E^j$ times, and we have constructed a family of parameters $\mathcal{W}^{2N_E}$ and $\mathcal{B}^{2N_E}$ such that
    \begin{align}\label{compression_of_H}
      \phi^{2N_E}(\mathcal{H})  :=\phi(\mathcal{W}^{2N_E},\mathcal{B}^{2N_E},\mathcal{H})=\{x_i\}_{i=1}^{N_h},
    \end{align}
    where $\{x_i\}_{i=1}^{N_h}\subset\R^d$ is a sequence of points.  Then we can apply Theorem \ref{multiclass_theorem} to find the parameters $\mathcal{W}^{L}$ and $\mathcal{B}^L$ with $L=2N_h+4M_h-1$, and an input-output map $\phi^{L}$ of \eqref{discrete_dynamics} such that
    \begin{align}\label{map_to_the_labels}
      \phi^{L}(x_i):=  \phi(\mathcal{W}^{L},\mathcal{B}^L,x_i)=f^h_i,\quad\text{for all }i\in \llbracket1,N_h\rrbracket.
    \end{align}
     Finally, composing the maps given by \eqref{compression_of_H} and \eqref{map_to_the_labels}, i.e., $\phi^{\mathcal{L}}= \phi^{L}\circ \phi^{2N_E}$ with $\mathcal{L}=L+2N_E$, we have that
    \begin{align}\label{final_map_hipercube}
        \phi^{\mathcal{L}}(x):=\phi(\mathcal{W}^\mathcal{L},\mathcal{B}^{\mathcal{L}},x)=f_h(x),\quad\text{for every } x\in \mathcal{H},
    \end{align}
where $\mathcal{W}^\mathcal{L}=\mathcal{W}^L\cup \mathcal{W}^{2N_h}$ and $\mathcal{B}^\mathcal{L}=\mathcal{B}^L\cup \mathcal{B}^{2N_h}$.

    \medbreak
\noindent \textbf{Step 3 (Error estimation):} Let us recall that in {\bf Step 1}, for a given $f\in L^p(\Omega;\R_{+})$ and $\varepsilon>0$, we have chosen $h>0$ small enough such that \eqref{eq:estima_simple_function} holds.

To estimate the error between $f_h$ and $\phi^{\mathcal{L}}$, since $\Omega \subset \mathcal{C}=\mathcal{H}\cup G_\delta^h$, we can write
\begin{align*}
    \|\phi^{\mathcal{L}}(x)-f_h(x)\|^p_{L^p(\Omega;\R_{+})} \leq \int_{\mathcal{H}} |\phi^{\mathcal{L}}(x)-f_h(x)|^p \, dx + \int_{G_{\delta}^h} |\phi^{\mathcal{L}}(x)-f_h(x)|^p \, dx.
\end{align*}
Owing to \eqref{final_map_hipercube}, the first term on the right-hand side vanishes. Thus,
\begin{align}\label{eq:prior_estimation}
    \|\phi^{\mathcal{L}}(x)-f_h(x)\|^p_{L^p(\Omega;\R_{+})} \leq \int_{G_{\delta}^h} |\phi^{\mathcal{L}}(x)-f_h(x)|^p \, dx \leq \int_{G_{\delta}^h} |f_h(x)|^p \, dx + \|\phi^{\mathcal{L}}\|_{L^\infty(G_{\delta}^h)}^p m_d(G^h_\delta).
\end{align}
Let us estimate each term on the right-hand side of \eqref{eq:prior_estimation}. Since $f_h$ converge to $f$ in $L^p(\Omega;\R_{+})$, in particular we have that 
\begin{align*}
    \|f_h-f\|_{L^p(\mathcal{H}_i^G;\R_{+})}= \|f^h_i-f\|_{L^p(\mathcal{H}_i^G;\R_{+})}\to0,\quad\text{ as }h\to0. 
\end{align*}
Due to the triangular inequality, we deduce that 
\begin{align}\label{eq:estimation_f_ih1}
|f_{i}^h|\leq C_1(1+\|f\|_{L^p(\Omega;\R_{+})}),    
\end{align}
 with $C_1>0$ a constant independent of $h$. Then, it follows that
\begin{align*}
    \int_{G_{\delta}^h} |f_h(x)|^p \, dx & = \int_{G_{\delta}^h} \left| \sum_{i=1}^{N_h^G} f_{i}^h \chi_{\mathcal{H}_i^G}(x) \right|^p \, dx =  \sum_{i=1}^{N_h^G} |f_{i}^h|^p m_d(\mathcal{H}_i^G) \\
    & \leq  C_1^p (1+\|f\|_{L^p(\Omega;\R_+)}^p) \sum_{i=1}^{N_h^G} m_d(\mathcal{H}_i^G)  =  C_1^p (1+\|f\|_{L^p(\Omega;\R_+)}^p) m_d(G^h_\delta).
 \end{align*}
Thus, using \eqref{eq:estimation_measure_G_h_d} we have 
\begin{align*}
\int_{G_{\delta}^h} |f_h(x)|^p \, dx \leq C_2\delta (h+\delta)^{d-1} h^{-d},
\end{align*}
with $C_2=  C_1^p C_{\Omega,d}(1+\|f\|_{L^p(\Omega;\R_+)}^p)$. Then, since we are taking 
 $\delta=h^{1+p}$, we obtain 
\begin{align}\label{eq:estimation_uat_1}
\int_{G_{\delta}^h} |f_h(x)|^p \, dx \leq C_2 (1+\|f\|_{L^p(\Omega;\R_+)}^p) h^{p}.
\end{align}
Before continuing, let us consider the following lemma, whose proof can be found in Appendix \ref{appendix_2}.
\begin{lemma}\label{estimation_norm_psi}
    Let  $\phi^{\mathcal{L}}$ be the map defined by \eqref{final_map_hipercube}, and denote by $l_{\mathcal{C}}$ the longest edge of $\mathcal{C}$. Then, for $h<l_{\mathcal{C}}\log(2)/(d+1)$ we have 
    \begin{align*}
            \|\phi^{\mathcal{L}}\|_{L^\infty(\mathcal{C};\R_{+})}\leq C_3\left(1+\delta (h+\delta)+h\right),
    \end{align*}
   with $C_3>0$ a constant depending on $d$, $m_d(\mathcal{C})$, and $\|f\|_{L^p(\Omega;\R_{+})}$.
\end{lemma}
Let us assume that  $h<l_{\mathcal{C}}\log(2)/(d+1)$. Then, using Lemma \ref{estimation_norm_psi}, \eqref{eq:estimation_measure_G_h_d}, and since $\delta<h^{\gamma+1}$ we deduce
\begin{align}\label{eq:estimation_uat_3}
\|\phi^{\mathcal{L}}\|_{L^\infty(G_{\delta}^h;\R_{+})} m_d(G^h_\delta)&\leq  C_3\left(1+\delta (h+\delta)+h\right) C_{\Omega,d}  \delta (h+\delta)^{d-1} h^{-d}\leq4 C_3 C_{\Omega,d} h^{p}.
\end{align}
  Finally, putting together  \eqref{eq:estimation_uat_1} and \eqref{eq:estimation_uat_3}, we deduce that
\begin{align}\label{eq:estimation_nn_simplefunction}
    \|\phi^{\mathcal{L}} - f_h\|^p_{L^p(\Omega;\R_{+})} & \leq C_2 h^{p}+ (4 C_3 C_{\Omega,d}h^p)^p \leq \hat{C} h^p,
\end{align}
where $\hat{C} = C_2 + (4 C_3 C_{\Omega,d})^p$ is a positive constant depending on $d$, $p$, $m_d(\mathcal{C})$, and $\|f\|_{L^p(\Omega;\R_+)}$. Thus, taking $h < \varepsilon / (2\hat{C}^{1/p} )$, we deduce that 
\begin{align}\label{eq:estimation_network_epsi2}
    \|\phi^{\mathcal{L}} - f_h\|_{L^p(\Omega;\R_{+})} < \frac{\varepsilon}{2}.
\end{align}
Therefore, taking $h < \min\{\varepsilon / (2\hat{C}^{1/p} ), h_1\}$, with $h_1$ defined in \eqref{values_h1}, we have that \eqref{eq:estima_simple_function} and \eqref{eq:estimation_nn_simplefunction} hold, and consequently, we conclude that 
\begin{align*}
    \|f - \phi^{\mathcal{L}}\|_{L^p(\Omega;\R_{+})} < \varepsilon.
\end{align*}

\medbreak
\noindent {\bf Step 4 (Depth Estimation):} To estimate the depth of the neural network, we first estimate  $\min\{\varepsilon / (2\hat{C}^{1/p} ), h_1\}$. Assume that $f\in W^{1,p}(\Omega;\R_+)$. Then, let us observe that in an analogous way to \eqref{eq:estimation_f_h2}, we can deduce that 
\begin{align}\label{eq:estimation_f_h_3}
    |f_i^h|\leq \|f\|_{L^{p}(\Omega;\R_{+})} +h\|f\|_{W^{1,p}(\Omega;\R_{+})}.
\end{align}
Therefore, using \eqref{eq:estimation_f_h_3} instead \eqref{eq:estimation_f_ih1}, we have $C_2\leq  C_1^p C_{\Omega,d}\|f\|_{W^{1,p}(\Omega;\R_+)}^p$. Moreover, due to Lemma \ref{estimation_norm_psi}, we have that $C_3= C\|f\|_{W^{1,p}(\Omega;\R_+)}^p$. Consequently, we deduce that the constant $\hat{C}$ can be taken as
\begin{align*}
    \hat{C}\leq ( C_1^p C_{\Omega,d} +  4 C C_{\Omega,d})\|f\|_{W^{1,p}(\Omega;\R_+)}^p =: \tilde{C}^p\|f\|_{W^{1,p}(\Omega;\R_+)}^p
\end{align*}
Consequently, using \eqref{values_h1} we have
\begin{align*}
    \min\left\{\frac{\varepsilon}{2\hat{C}^{1/p} }, h_1\right\}& = \min\left\{\frac{\varepsilon}{2\hat{C}^{1/p} }, \frac{\varepsilon}{2C \|f\|_{W^{1,p}(\Omega; \mathbb{R}_+)}}\right\}\geq \frac{\varepsilon}{2\|f\|_{W^{1,p}(\Omega; \mathbb{R}_+)}}\min\left\{\frac{1}{\tilde{C} }, \frac{1}{ C }\right\}\\
    &=:\frac{\varepsilon C_5}{2\|f\|_{W^{1,p}(\Omega; \mathbb{R}_+)}}.
\end{align*}
Then, in particular, we take
\begin{align}\label{eq:h_for_the_convergence}
    h= \frac{\varepsilon C_5}{2\|f\|_{W^{1,p}(\Omega; \mathbb{R}_+)}}.
\end{align}
Now, due to Step 2.3, we know that $\mathcal{L}=2N_h+4M_h-1+2N_E$. Moreover, observe that $N_E$ and $N$ can be estimated by
 \begin{align}\label{eq:num_NE}
      N_E \leq d\left\lceil \frac{l_{\mathcal{C}}}{h+\delta} \right\rceil,\quad\text{and} \quad N_h\leq d\left\lceil \frac{l_{\mathcal{C}}}{h+\delta} \right\rceil^d.
 \end{align}
with $\delta = h^{1+p}$. Then, using $M_h\leq N_h$ and the estimations \eqref{eq:num_NE} and \eqref{eq:h_for_the_convergence}, we deduce the upper bound for the depth
\begin{align*}
    \mathcal{L} \leq C_6\left(\|f\|_{W^{1,p}(\Omega;\mathbb{R}_+)}^{d}\varepsilon^{-d} + \|f\|_{W^{1,p}(\Omega;\mathbb{R}_+)}\varepsilon^{-1}+1\right),
\end{align*}
where $C_6$ is a positive constant that depends on $m_d(\mathcal{C})$, $p$, and $d$. This concludes the proof of \eqref{upper_bound_L_introl}.
\end{proof}

\section{Proof of the training theorems}\label{sec:proo_training}
In this section, we provide the proof of Theorems \ref{th:estima_optimal _control_Jlambda} and \ref{th:estimation_functional_dif_activation} and Corollary \ref{coro:convergence_gelu_relu}.

\begin{proof}[Proof of Theorem \ref{th:estima_optimal _control_Jlambda}]
By the definition of minimizer, we have
\begin{align}\label{eq:estimationfunctionallambda}
\nonumber\mathcal{J}_{\lambda}(\mathcal{W}_\lambda^L,\mathcal{B}_\lambda^L)&\leq \mathcal{J}_{\lambda}(\mathcal{W}^L_*,\mathcal{B}^L_*)= \lambda\tnorm{(\mathcal{W}_*^L,\mathcal{B}_*^L)}^2_2+\frac{1}{N}\sum_{i=1}^N\loss\left(\phi(\mathcal{W}_*^L,\mathcal{B}^L_*,x_i), y_i\right)\\
&=\lambda\tnorm{(\mathcal{W}_*^L,\mathcal{B}_*^L)}^2_2+ \mathcal{J}_{0}(\mathcal{W}^L_*,\mathcal{B}^L_*)=\lambda\tnorm{(\mathcal{W}_*^L,\mathcal{B}_*^L)}^2_2,
\end{align}
where we have used \eqref{eq:functional_zero}. Now, by the definition of the functional $\mathcal{J}_\lambda$ and \eqref{eq:estimationfunctionallambda}, we obtain  
\begin{align}\label{eq:pre_lambda_to_zero}
    \frac{1}{N}\sum_{i=1}^N\loss\left(\phi(\mathcal{W}^L_\lambda,\mathcal{B}^L_\lambda,x_i), y_i\right)\leq  \mathcal{J}_{\lambda}(\mathcal{W}^L_\lambda,\mathcal{B}_\lambda^L)\leq \lambda\tnorm{(\mathcal{W}_*^L,\mathcal{B}_*^L)}^2_2.
\end{align}
Now, due to Corollary \ref{coro:estimation_norms}, we know that $\tnorm{(\mathcal{W}_*^L,\mathcal{B}_*^L)}_2$ can be uniformly bounded with respect to $\lambda$. Therefore, taking $\lambda \to 0$ in \eqref{eq:pre_lambda_to_zero}, we conclude \eqref{eq:limit_lambda_to_zero}. 

Next, by the definition of $\mathcal{J}_\lambda$, we obtain
\begin{align}\label{eq:estimation_sec_lambda_para}
\lambda\tnorm{(\mathcal{W}^L_\lambda,\mathcal{B}^L_\lambda)}_2\leq\mathcal{J}_\lambda(\mathcal{W}^L_\lambda,\mathcal{B}^L_\lambda)\leq \lambda\tnorm{(\mathcal{W}_*^L,\mathcal{B}_*^L)}^2_2.
\end{align}
Then, dividing \eqref{eq:estimation_sec_lambda_para} by $\lambda>0$, we deduce that the sequence $\{{(\mathcal{W}_\lambda^L,\mathcal{B}_\lambda^L)}\}_{\lambda>0}$ is bounded by $\tnorm{(\mathcal{W}_*^L,\mathcal{B}_*^L)}_2$. Moreover, the Bolzano–Weierstrass Theorem ensures that this sequence has a subsequence that converges to some $(\mathcal{W}_0^L,\mathcal{B}_0^L)$. Denoting such a subsequence again by $\{{(\mathcal{W}_\lambda^L,\mathcal{B}_\lambda^L)}\}_{\lambda>0}$, we observe that 
\begin{align}\label{eq:estimation_smart}
   \mathcal{J}_0(\mathcal{W}^L_\lambda,\mathcal{B}^L_\lambda) = \mathcal{J}_\lambda(\mathcal{W}^L_\lambda,\mathcal{B}^L_\lambda)-\lambda\tnorm{(\mathcal{W}^L_\lambda,\mathcal{B}^L_\lambda)}_2\leq\mathcal{J}_\lambda(\mathcal{W}^L_\lambda,\mathcal{B}^L_\lambda).
\end{align}
Then, combining \eqref{eq:estimationfunctionallambda} and \eqref{eq:estimation_smart}, and using the continuity of $\mathcal{J}_0$, we get 
\begin{align*}
\mathcal{J}_0(\mathcal{W}_0^L,\mathcal{B}_0^L) = \mathcal{J}_0\left(\lim_{\lambda\to0} (\mathcal{W}^L_\lambda,\mathcal{B}^L_\lambda)\right) =   \lim_{\lambda\to0} \mathcal{J}_0(\mathcal{W}^L_\lambda,\mathcal{B}^L_\lambda) = 0,
\end{align*}
concluding that $(\mathcal{W}_0^L,\mathcal{B}_0^L)$ is a minimizer of $\mathcal{J}_0$. Now, let us consider
\begin{align*}
(\tilde{\mathcal{W}}_0^L,\tilde{\mathcal{B}}_0^L) \in \argmin\left\{\tnorm{(\mathcal{W}^L,\mathcal{B}^L) }_2 \,:\,\mathcal{J}_0(\mathcal{W}^L,\mathcal{B}^L)=0 \right\},
\end{align*}
that is, $(\tilde{\mathcal{W}}_0^L,\tilde{\mathcal{B}}_0^L)$ is a parameter of minimal norm minimizing $\mathcal{J}_0$. In particular, since $(\tilde{\mathcal{W}}_0^L,\tilde{\mathcal{B}}_0^L)$ is a minimizer of $\mathcal{J}_0$, \eqref{eq:uniform_estimatin_control_lambda} implies that 
\begin{align}\label{eq:minimal_norm_lambda}
\tnorm{(\mathcal{W}_\lambda^L,\mathcal{B}_\lambda^L)}_2\leq \tnorm{(\tilde{\mathcal{W}}_0^L,\tilde{\mathcal{B}}_0^L)}_2,\quad \text{for all }\lambda>0.
\end{align}
Taking $\lambda \to 0$ in \eqref{eq:minimal_norm_lambda}, we obtain  
\begin{align}\label{eq:minimal_norm}
\tnorm{(\mathcal{W}_0^L,\mathcal{B}_0^L)}_2\leq \tnorm{(\tilde{\mathcal{W}}_0^L,\tilde{\mathcal{B}}_0^L)}_2.
\end{align}
If \eqref{eq:minimal_norm} holds with strict inequality, then we contradict the fact that $(\tilde{\mathcal{W}}_0^L,\tilde{\mathcal{B}}_0^L)$ is a parameter of minimal norm. Therefore, we must have $$\tnorm{(\mathcal{W}_0^L,\mathcal{B}_0^L)}_2= \tnorm{(\tilde{\mathcal{W}}_0^L,\tilde{\mathcal{B}}_0^L)}_2,$$
that is, $(\mathcal{W}_0^L,\mathcal{B}_0^L)$ is also a parameter of minimal norm that minimize $\mathcal{J}_0$.
\end{proof}

\begin{proof}[Proof of Theorem \ref{th:estimation_functional_dif_activation}]
We begin by analyzing the deviation between $x_i^j$ and $\hat x^{j}_i$, the solutions of \eqref{discrete_dynamics} and \eqref{eq:discrete_dynamics_different_activation} at layer $j$. We have
\begin{align}\label{eq:recursive_inequality}
\nonumber \|x^j_i-\hat x_i^j\|& \leq \|\hat\vecsigma_j((\hat{W}_j\hat x_i^{j-1}+\hat{b}_j) - \vecsigma_j(W_j\hat x_i^{j-1}+b_j)\| + \|\vecsigma_j(W_j \hat x_i^{j-1}+b_j) - \vecsigma_j(W_jx_i^{j-1}+b_j)\|\\
&\leq \nu_j + \|W_j\|\|x^{j-1}_i - \hat x_i^{j-1}\| \leq \nu_j + \tnorm{(\mathcal{W}_*^L,\mathcal{B}_*^L)}_2 \|x^{j-1}_i - \hat x_i^{j-1}\|,
\end{align}
where we have used the fact that the radii $R_j$ defined in \eqref{eq:difinition_radios_error} are large enough to control the deviation between the activations at each layer. We also used the Lipschitz continuity of the activation functions on compact sets.

Since $\|x^0_i - \hat x_i^0\| = 0$, an iterative application of \eqref{eq:recursive_inequality} yields
\begin{align*}
\|\hat \phi(\hat{\mathcal W}_\lambda^L,\hat{\mathcal{B}}_\lambda^L, x_i) - \phi({\mathcal W}^L_*,{\mathcal{B}}^L_*, x_i)\|^2 &= \|x^{L}_i - \hat x_i^{L}\|^2 \leq \left( \sum_{j=1}^L \nu_j \tnorm{(\mathcal{W}_*^L,\mathcal{B}_*^L)}_2^{L-j} \right)^2\\
&\leq \|\nu\|_2^2 \sum_{j=1}^L \tnorm{(\mathcal{W}_*^L,\mathcal{B}_*^L)}_2^{2(L-j)} \\
&= \|\nu\|_2^2 \left( \frac{\tnorm{(\mathcal{W}_*^L,\mathcal{B}_*^L)}_2^{2L} - 1}{\tnorm{(\mathcal{W}_*^L,\mathcal{B}_*^L)}_2^2 - 1} \right) =: \mathscr{E}(\nu,\mathcal{W}_*^L,\mathcal{B}_*^L),
\end{align*}
where we used the Cauchy–Schwarz inequality and the formula for the sum of a geometric series.

Now, consider the following compact set
\begin{align*}
    \mathcal{K}_{\nu} := \left\{(z,y)\in \R^m \times \R^m \,:\, \|z\|\leq \mathscr{E}(\nu,\mathcal{W}_*^L,\mathcal{B}_*^L),\; y\in \{y_i\}_{i=1}^N\right\},
\end{align*}
and let us define
\begin{align}\label{eq:definitionA_error}
    \mathscr{A}_{\loss}(\nu,\mathcal{W}_*^L,\mathcal{B}_*^L) := \max_{(z,y)\in \mathcal{K}_{\nu}} \loss(z+y,y).
\end{align}
Due to the continuity and nonnegativity of $\loss$, it follows that $\mathscr{A}_{\loss} < \infty$ and $\mathscr{A}_{\loss} \geq 0$. Moreover, as $\|\nu\|_2 \to 0$, we have $\mathscr{E} \to 0$, and hence the compact set $\mathcal{K}_{\nu}$ converges to
\begin{align*}
    \mathcal{K}_0 = \left\{(z,y)\in \R^m \times \R^m \,:\, \|z\|\leq 0,\; y\in \{y_i\}_{i=1}^N\right\}.
\end{align*}
Consequently, we deduce that 
\begin{align*}
    \mathscr{A}_{\loss}(0,\mathcal{W}_*^L,\mathcal{B}_*^L) = \max_{(z,y)\in \mathcal{K}_0} \loss(y,y) = 0,
\end{align*}
since $\loss(y,y) = 0$. Thus, $\mathscr{A}_{\loss}(\nu,\mathcal{W}_*^L,\mathcal{B}_*^L) \to 0$ as $\|\nu\|_2 \to 0$.

Now, by the definition of minimizer and using \eqref{eq:definitionA_error}, we obtain
\begin{align*}
\hat{\mathcal{J}}_{\lambda}(\hat{\mathcal{W}}^L_\lambda,\hat{\mathcal{B}}^L_\lambda) &\leq \hat{\mathcal{J}}_{\lambda}({\mathcal{W}}^L_*,{\mathcal{B}}^L_*) = \lambda \tnorm{({\mathcal{W}}^L_*,{\mathcal{B}}^L_*)}^2_2 + \frac{1}{N} \sum_{i=1}^N \loss\left( \hat\phi({\mathcal{W}}^L_*,{\mathcal{B}}^L_*,x_i), y_i \right) \\
&= \lambda \tnorm{({\mathcal{W}}^L_*,{\mathcal{B}}^L_*)}^2_2 + \frac{1}{N} \sum_{i=1}^N \loss\left( \hat\phi({\mathcal{W}}^L_*,{\mathcal{B}}^L_*,x_i) - \phi({\mathcal{W}}^L_*,{\mathcal{B}}^L_*,x_i) + \phi({\mathcal{W}}^L_*,{\mathcal{B}}^L_*,x_i), y_i \right) \\
&= \lambda \tnorm{({\mathcal{W}}^L_*,{\mathcal{B}}^L_*)}^2_2 + \frac{1}{N} \sum_{i=1}^N \loss\left( \left[ \hat\phi({\mathcal{W}}^L_*,{\mathcal{B}}^L_*,x_i) - \phi({\mathcal{W}}^L_*,{\mathcal{B}}^L_*,x_i) \right]+y_i , y_i \right) \\
&\leq \lambda \tnorm{({\mathcal{W}}^L_*,{\mathcal{B}}^L_*)}^2_2 + \frac{1}{N} \sum_{i=1}^N \mathscr{A}_{\loss}(\nu,\mathcal{W}_*^L,\mathcal{B}_*^L) = \lambda \tnorm{({\mathcal{W}}^L_*,{\mathcal{B}}^L_*)}^2_2 + \mathscr{A}_{\loss}(\nu,\mathcal{W}_*^L,\mathcal{B}_*^L),
\end{align*}
which concludes the proof.
\end{proof}

\begin{proof}[Proof of Corollary \ref{coro:convergence_gelu_relu}]
Let us analyze the behavior of $\sigma_\varepsilon$ when $\varepsilon\to0$. We first observe that 
    \begin{align*}
        \lim_{u\to0^+ } \erf\left(\frac{1}{u}\right) = 1,\quad \text{and}\quad  \lim_{u\to0^- } \erf\left(\frac{1}{u}\right) = -1.
    \end{align*}
Now, taking $u=\varepsilon\sqrt{2}/x$, we observe that since $\varepsilon>0$, we have 
\begin{align*}
        \lim_{\varepsilon\to0^+ } \erf\left(\frac{x}{\varepsilon \sqrt{2}}\right) = \begin{cases}
            1&\text{if }x>0,\\
            -1& \text{if }x<0.
        \end{cases}
    \end{align*}
Consequently, we have
\begin{align}\label{eq:limit_gelu}
   \lim_{\varepsilon\to0^+ } \sigma_\varepsilon(x) =     \lim_{\varepsilon\to0^+ } \frac{x}{2}\left(1+\erf\left(\frac{x}{\varepsilon \sqrt{2}}\right) \right)= \begin{cases}
            x&\text{if }x>0,\\
            0& \text{if }x<0 ,
        \end{cases} = \ReLU(x) =\sigma(x).
    \end{align}
Therefore, since \eqref{eq:limit_gelu} holds for every $x$, we deduce that $\sigma_\varepsilon \to \sigma$ uniformly in $\R$. The rest of the proof follows as a direct application of Theorem \ref{th:estimation_functional_dif_activation}, by observing that when $\varepsilon\to 0$ then $\nu\to0$ as well. 
\end{proof}

\section{Further comments and open problems}\label{further_comentaries}
In this paper, we have demonstrated that a 2-wide deep neural network can address any classification problem with no more than \(O(N)\) layers. Additionally, we have established a universal approximation theorem for \(L^p(\Omega;\R_+)\) functions, requiring a neural network width \(d+1\). Notably, our proofs are fully constructive, explicitly detailing the parameters to be utilized, giving a fully geometric interpretation of the architecture employed, and providing a formal proof of each statement.

These explicit constructions yield bounds on the parameters and provide a priori estimates for minimizers of standard regularized training loss functionals in supervised learning. As the regularization parameter vanishes, the trained networks converge to exact classifiers.

In the following, we present some interesting open-related questions.
\smallbreak

 \noindent{\bf (1) Understanding $n$-wide deep neural networks.} The construction of the parameters in Theorem \ref{multiclass_theorem} provides a clear and geometric interpretation of why and how the neural network achieves memorization. In this context, it would also be interesting to describe and explain geometrically other results in the literature, such as the $3$-wide deep neural networks constructed in \cite{park2021provable}, or the two-layer network constructed in \cite{yun2019small} that achieves memorization with $O(N^{1/2})$ neurons.
\smallbreak

 \noindent{\bf (2) Topology of the Dataset.} The first step in the proof of Theorem \ref{multiclass_theorem} involves projecting the data into a one-dimensional space. This reduction simplifies the data structure, facilitating the development of our algorithm. However, this projection results in the possible loss of the original data distribution, since it could, for instance, disperse initially clustered points.
Therefore, to take advantage of the initial data distribution, we could project the data in a space of dimension greater than one, using more hyperplanes (neurons) at the first step (layer). One possibility is to find a low-dimensional space in which the points can be embedded, preserving distances. This is precisely what the Johnson-Lindenstrauss lemma states, ensuring the existence of a linear map that projects points in a lower dimensional space, preserving distances. This could reduce the number of hidden layers that Theorem \ref{multiclass_theorem} uses. Manipulating data in dimensions higher than $2$ can considerably reduce the number of layers needed,  \cite{park2021provable}.
\smallbreak

 \noindent{\bf (3) Width versus Depth. } As we have seen in the bibliographic discussion of Section \ref{sec:related_work}, networks with one hidden layer can memorize $N$ data points with $O(N)$ neurons, while adding an extra layer, making it possible using $O(\sqrt{N})$ neurons, \cite{yun2019small}. In the context of deep neural networks, width $12$ allows memorization with $O(N^{1/2}+\log(N))$ neurons,  \cite{park2021provable}, and width $3$ with $O(N^{2/3}\log(N))$ neurons. In this paper, we have shown that $O(N)$ neurons suffice for width $2$. These results exhibit a trade-off between the depth and width of the network. A systematic analysis of this compromise between depth and width would be desirable. 
\smallbreak

 \noindent{\bf (4) Extension of the Universal Approximation Theorem.} As observed in the proof of Theorem \ref{UAT_LP}, Theorem \ref{multiclass_theorem} was utilized to map the resulting points to their respective labels. However, a universal approximation theorem can also be concluded using other networks, not necessarily of width 2 \cite{park2021provable,vardi2021optimal,yun2019small,zhang2017understanding}. By maintaining a width of $d+1$, we can incorporate the neural network with width 3 introduced in \cite{park2021provable} to ensure universal approximation when $d \geq 2$. Combining this result with our strategy may lead to a neural network with reduced depth compared to the one given by Theorem \ref{UAT_LP}.

On the other hand, the first step of the proof of \Cref{UAT_LP} uses an approximation by simple functions of finite volume type on a regular set of hyperrectangles, ensuring an error of order $h$ \cite{davydov2010algorithms}. However, more sophisticated nonlinear approximation procedures, such as those based on \textit{dyadic partitions}, could achieve better convergence rates (see \cite{davydov2010algorithms}). Nevertheless, to approximate functions using the neural network, it would be necessary to develop an iterative algorithm operating over a non-uniform grid.
\smallbreak

\noindent{\bf (5) Minimal width universal approximation theorem for more general spaces.} Universal approximation theorems have been extended to Sobolev and Besov spaces, as discussed in \cite{devore2021neural, Siegel}. In \cite{Siegel}, it is shown that the class $W^{s,q}$, when compactly embedded in $L^p$, can be approximated by neural networks. For $p \leq q$, piecewise polynomial approximations on uniform grids are applicable, which neural networks can approximate. Additionally, the neural network in \cite{Siegel} uses a width  $25d + 31$. The uniformity of the grid allows us to extend our methodology to estimate the network depth while maintaining a small width and understanding the parameters involved in the approximation. However, for $p > q$, nonlinear or adaptive methods are required. Although \cite{Siegel} constructs a neural network to approximate nonlinear functions on non-uniform grids, the choice of parameters lacks clear geometric intuition, and no algorithmic procedure is provided.

\appendix
\section{}\label{appendix_1}
\begin{proof}(Proof of Proposition \ref{key_proposition})
Recall that the classes under consideration are defined as follows:
\begin{align}\label{Conjuntos_xx_ap}
    \mathcal{C}_k=\{x_i \text{ with } i\in \llbracket1,N\rrbracket \,:\, y_i=k \},\quad\text{and}\quad \mathcal{C}=\bigcup_{k=0}^{M-1} \mathcal{C}_k.
\end{align}
To prove that we can compress the $M$ classes, we proceed by induction. For the first class, i.e. $k=0$, the proof follows directly by noting that the assumptions of the statement ensure the existence of parameters $\mathcal{W}^{L_0},\mathcal{B}^{L_0}$ and $\tilde z_0\in\R^2$ such that $\phi(\mathcal{W}^{L_0},\mathcal{B}^{L_0},\mathcal{C}_{0})=\tilde z_0$.

Let us assume that the statement holds for some $0<k<M-1$. Thus, there exists a collection of points $\{z_j\}_{j=0}^k\subset \R^2$, $\tilde L_1>1$, and parameters $\mathcal{W}^{\tilde L_1},\mathcal{B}^{\tilde L_1}$ such that the input-output map satisfies
\begin{align*}
\phi_k(\mathcal{W}^{\tilde L_1},\mathcal{B}^{\tilde L_1},\mathcal{C}_{j})=z_j, \quad \text{ for every }j\in\llbracket0,k\rrbracket .
\end{align*}
In particular, we have that 
\begin{align*}
\phi_k(\mathcal{W}^{\tilde L_1},\mathcal{B}^{\tilde L_1},\mathcal{C}_{k+1})=\hat{\mathcal{C}}_{k+1},\quad \text{where } \hat{\mathcal{C}}_{k+1}=\{z_i \text{ with } i\in \llbracket1,N\rrbracket\,:\, y_i={k+1}\}.
\end{align*}
Let us prove the statement for $k+1$. Denote by $\hat{\mathcal{C}}=\phi^k(\mathcal{W}^{\tilde L_1},\mathcal{B}^{\tilde L_1},\mathcal{C})$. By hypothesis, there exist $\hat{z}_{k+1}\in\R^2$, $\tilde L_2\geq 1$, $\mathcal{W}^{\tilde L_2}$ and $\mathcal{B}^{\tilde L_2}$ such that 
\begin{align}\label{condition_1_propo1_ap}
\phi_{k+1}(\mathcal{W}^{\tilde L_2},\mathcal{B}^{\tilde L_2},\hat{\mathcal{C}}_{k+1})=\hat{z}_{k+1},\qquad  \phi_{k+1}(\mathcal{W}^{\tilde L_2},\mathcal{B}^{\tilde L_2},\hat{\mathcal{C}}\setminus\hat{\mathcal{C}}_{k+1})\neq \hat{z}_{k+1},
\end{align}
and 
\begin{align}\label{condition_2_propo1_ap}
    \phi_{k+1}(\mathcal{W}^{\tilde L_2},\mathcal{B}^{\tilde L_2},z^1)\neq  \phi_{k+1}(\mathcal{W}^{\tilde L_2},\mathcal{B}^{\tilde L_2},z^2),\quad\text{for all   } z^1,\,z^2\in\hat{\mathcal{C}}\setminus \hat{\mathcal{C}}_{k+1}, \, z^1\neq z^2.
\end{align}
Since the composition of input-output maps is again an input-output map of \eqref{discrete_dynamics}, we can compose $\phi_{k+1}$ and $\phi_k$. Then, by \eqref{condition_1_propo1_ap}-\eqref{condition_2_propo1_ap}, the map $\hat{\phi}:=\phi_{k+1}\circ\phi_k $ satisfies that 
\begin{align}\label{eq:last_relation_prop1_ap}
    \hat{\phi}(\mathcal{W}^{\tilde L_3},\mathcal{B}^{\tilde L_3},\mathcal{C}_j)=\hat{z}_j\quad \text{for every }j\in \llbracket 0,k+1\rrbracket,
\end{align}
where $\tilde L_3=\tilde L_2+\tilde L_1$, $\mathcal{W}^{\tilde L_3}=\mathcal{W}^{\tilde L_2}\cup\mathcal{W}^{\tilde L_1}$, and $\mathcal{B}^{\tilde L_3}=\mathcal{B}^{\tilde L_2}\cup\mathcal{B}^{\tilde L_1}$. This concludes the induction. The result follows by taking $k=M-2$ in \eqref{eq:last_relation_prop1_ap}.
\end{proof}

\begin{proof}(Proof of Corollary \ref{corollary_clasification_R})
Let us consider the dataset $\{(x_i, y_i)\}_{i=1}^N \subset \mathbb{R}^{d} \times \{\alpha_0, \dots, \alpha_{M-1}\}$, where $\{\alpha_k\}_{k=0}^{M-1} \subset \mathbb{R}$. Without loss of generality, assume that $y_i < y_j$ for every $i < j$, where $i, j \in \{0, \dots, M-1\}$. If $y_0 \geq 0$, we conclude by applying \Cref{multiclass_theorem}. If $y_0 < 0$, we consider a new set of labels given by $\hat{y}_i = y_i - y_0$ for every $i \in \{0, \dots, M-1\}$, noting that $\hat{y}_0 = 0$. Furthermore, $\{\hat{y}_i\}_{i=1}^N \subset \{\hat{\alpha}_k\}_{k=0}^{M-1} \subset \mathbb{R}_+$. Then, according to \Cref{multiclass_theorem}, there exist parameters $\mathcal{W}^L$ and $\mathcal{B}^L$ such that for $L = 2N + 4M - 1$ and $w_{\text{max}} = 2$, the input-output map of \eqref{discrete_dynamics_2} with $A_j^1 = \text{Id}_{d_j}$ (the identity matrix in $\mathbb{R}^{d_j \times d_j}$) for all $j \in \{1, \dots, L\}$, satisfies
\begin{align}\label{eq:input_output_map_1}
\phi^L(x_i) = \phi(\mathcal{A}^L, \mathcal{W}^L, \mathcal{B}^L, x_i) = \hat{y}_i, \quad \text{for all } i \in \{1, \dots, N\}.
\end{align}
It now suffices to construct a mapping that transforms each $\hat{y}_i$ into $y_i$ for every $i \in \{1, \dots, N\}$. Consider the parameters 
\begin{align*}
    w_1^{1} = -1, \quad w_2^{1} = 1, \quad b_1^{1} = -y_0, \quad b_2^{1} = y_0, \quad \text{and } A^{1} = (-1, 1),
\end{align*}
where $W^{1} = (w_1^{1}, w_2^{1})^\top$ and $b^{1} = (b_1^{1}, b_2^{1})^\top$. We denote by $\phi^{1}$ the input-output map of \eqref{discrete_dynamics_2} defined by $W^{1}, b^{1}$, and $A^{1}$. For every $\hat{y}_i$, we have that $-\hat{y}_i + y_0 \geq 0$ or $\hat{y}_i - y_0 \geq 0$. If $-\hat{y}_i + y_0 \geq 0$, then
\begin{align*}
    \phi^{1}(\hat{y}_i) &= A^{1} \cdot \vecsigma(W^{1} \hat{y}_i + b^{1}) = (-1, 1) \cdot \begin{pmatrix} \sigma(-\hat{y}_i + y_0) \\ \sigma(\hat{y}_i - y_0) \end{pmatrix} \\
    &= -\sigma(-\hat{y}_i + y_0) + \sigma(\hat{y}_i - y_0) = -(-\hat{y}_i + y_0) = y_i.
\end{align*}
Similarly, for $\hat{y}_i$ such that $\hat{y}_i - y_0 \geq 0$, we obtain $\phi^{1}(\hat{y}_i) = (\hat{y}_i - y_0) = y_i$. Therefore, the input-output map $\phi^{L+1} := \phi^{1} \circ \phi^{L}$ can memorize the dataset $\{(x_i, y_i)\}_{i=1}^N$. Moreover, since the width and depth of the neural network defined by $\phi^{1}$ are $2$ and $1$, respectively, the resulting neural network defined by $\phi^{L+1}$ has a width  $2$ and a depth $2N + 4M$.

\end{proof}

\section{}\label{appendix_2}

\begin{lemma}\label{lemma:single_subregion}
    For every $I_n$ with $n\geq 1$ and for all $\mathcal{H}_j,\mathcal{H}_i \in \mathcal{S}_{I_n}$ with  $i\neq j$, we have that $F(\mathcal{H}_i)\cap F(\mathcal{H}_j)=\emptyset.$
\end{lemma}
\begin{proof}(Proof of Lemma \ref{lemma:single_subregion})
    By contradiction, assume that there exist $\mathcal{H}_1$ and $\mathcal{H}_2$ in some $\mathcal{S}_{I_n}$ such that $F(\mathcal{H}_1)\cap F(\mathcal{H}_2)\neq\emptyset.$ Therefore, there exist $x_1\in \mathcal{H}_1$ and $x_2\in \mathcal{H}_2$ such that $F(x_1)=F(x_2)$. Since $\mathcal{H}_1$ and $\mathcal{H}_2$ are different hyperrectangles, there exists $k\in\llbracket1,d+1\rrbracket$ such that $x^{(k)}_1\neq x^{(k)}_2$. Due to the fact that $\mathcal{H}_1$ and $\mathcal{H}_2$ belong to the same subregion, we have that 
    \begin{align}\label{eq:relation_p_lemma1}
        p_\eta(x_1)=p_\eta(x_2), \quad\text{for every }\eta\in \llbracket1,d+1\rrbracket.
    \end{align}
    Since $F(x_1)=F(x_2)$ on each coordinate, we will have  
    \begin{align}\label{eq:Fx1=Fx2}
     p_\eta(x_1)x_1^{(\eta)}+ p_\eta(x_1)b_\eta =  p_\eta(x_2)x_2^{(\eta)}+ p_\eta(x_2)b_\eta, \quad\text{for every }\eta\in \llbracket1,d+1\rrbracket. 
    \end{align}
    Therefore, using \eqref{eq:relation_p_lemma1}, we conclude that $x_1^{(\eta)}=x_2^{(\eta)}$ for every $\eta\in \llbracket1,d+1\rrbracket$, which is a contradiction  since $x^{(k)}_1\neq x^{(k)}_2$.
\end{proof}

\begin{proof}(Proof of Lemma \ref{lemma:all_subregions})
From the definition of $p_\eta(x)$, we immediately have that $F(\mathcal{H}_*)= \mathbf{0}_{d}$. For the second part of the lemma, we proceed by contradiction. Let us assume that there exist $\mathcal{H}_1,\mathcal{H}_2$ such that $F(\mathcal{H}_1)\cap F(\mathcal{H}_2)\neq\emptyset$.
If $\mathcal{H}_1$ and $\mathcal{H}_2$ belong to the same subregion, we are done due to Lemma \ref{lemma:single_subregion}. 

Therefore, we can assume that $\mathcal{H}_1$ and $\mathcal{H}_2$ are in different subregions. This implies that a hyperplane separates them. Thus, there exists $k\in\llbracket1,d+1\rrbracket$ such that 
\begin{align}\label{eq:dichotomic_b}
    x_1^{(k)}<b_k<x_2^{(k)} \quad\text{or}\quad x_2^{(k)}<b_k<x_1^{(k)},
\end{align}
and that $p_k(x_1)\neq p_{k}(x_2)$ for every $x_1\in \mathcal{H}_1$ and $x_2\in \mathcal{H}_2$. Without loss of generality, we assume that 
\begin{align}\label{eq:dicotomic_p}
    p_k(x_1)=1\quad \text{and}\quad p_k(x_2)=0,
\end{align}
that is, $\mathcal{H}_1\subset R_k$, and due to the fact that the hyperrectangles do not intersect the hyperplanes, we have that 
\begin{align}\label{eq:stric_inequality_xk}
    e_k\cdot x_1+b_k>0.
\end{align}
We continue the proof by dividing it into two cases. 
\smallbreak
\noindent $\bullet$ The case $P(x_1) = P(x_2)$: In such case, we have that $p_\eta(x_1)x_1^{(\eta)}=p_\eta(x_2)x_2^{(\eta)}$ for all $\eta\in\llbracket1,d+1\rrbracket$. Using \eqref{eq:dicotomic_p}, we deduce that $x_{1}^{(k)}=0$. Thus, due to \eqref{eq:dichotomic_b}, we have that $b_{k}\neq 0$. Since $F(x_1)=F(x_2)$ and we have assumed that $P(x_1)=P(x_2)$, then $G(x_1)=G(x_2)$. The last equality implies that $p_\eta(x_1)b_\eta= p_\eta(x_2)b_\eta$ for all $\eta\in\llbracket1,d+1\rrbracket$, therefore, applying \eqref{eq:dicotomic_p}, we conclude that $b_k=0$, which is a contradiction. 
\smallbreak
\noindent$\bullet$ The case $P(x_1) \neq P(x_2)$. When $p_{k}(x_1)x_1^{(k)}\neq p_{k}(x_2)x_2^{(k)}$, due to \eqref{eq:dicotomic_p}, necessarily $x_1^{(k)}\neq 0$. As before, using \eqref{eq:dicotomic_p} and the fact that $F(x_1)=F(x_2)$, we deduce that $x_1^{(k)}=-b^k$. Therefore, considering \eqref{eq:stric_inequality_xk}, we face a contradiction. 
\end{proof}

\begin{proof}[Proof of the Lemma \ref{estimation_norm_psi}] Let us begin by observing that Corollary \ref{coro:estimation_norms} cannot be directly applied since now we have an infinite number of data points. However, we can provide a similar estimation by carefully analyzing the parameters used in the proof of Theorem \ref{UAT_LP}. According to Step 2.3 in the proof of Theorem \ref{UAT_LP}, the map $\phi^{\mathcal{L}}=\phi^L\circ \phi^{2N_E}$ drives the hyperrectangles defined in $\mathcal{H}$ into their respectively labels. Therefore, to estimate the norm of $\phi^\mathcal{L}$, we divide the proof into two parts.
\medbreak

\noindent\textit{ \bf Norm of $\phi^{2N_E}$.} To estimate the norm of $\phi^{2N_E}$, we make the following observations:
\smallbreak

\noindent 1) Due to the fact that the hyperplanes defined in \eqref{hiper_approximation} must belong to $G^h_\delta$, we have
    \begin{align*}
        \|b_\eta\|_{\infty}\leq C_\eta (h+\delta/2)+m_d(\Omega), \quad\text{for every }\eta\in \llbracket1,d+1\rrbracket,
    \end{align*}
    where $m_d(\mathcal{C})$ is the Lebesgue measure of $\mathcal{C}$ and the $C_\eta$'s are positive  uniformly bounded constants. Thus, $\|b^1\|_{\infty}\leq Ch+m_d(\Omega)$. Moreover, by definition, $\|W^1\|_{\infty}=1$.
   \medbreak

\noindent 2) With the parameters derived in Step 2 of the proof of Theorem \ref{UAT_LP}, the hyperrectangles are mapped to a $d+1$-dimensional space. Since $\|W^1\|=1$, the hyperrectangles are mapped according to their distance to the hyperplane, which is less than $\delta/2$. Furthermore, all hyperrectangles are no farther away than $C(h+\delta/2)$. Thus, the parameters $b^2_\eta$ introduced in Step 2 of the proof of Theorem \ref{UAT_LP} satisfy
$
        \|b^2\|_{\infty}\leq C(h+\delta/2).
$
    By definition, again, $\|W^2\|_{\infty}=1$.
\medbreak

\noindent 3) When projecting the hyperrectangles into the $d$-dimensional space, they remain no farther away than $C(h+\delta/2)$. Additionally, the hyperrectangles are contained within the interior of a ball $B_0(C(h+\delta/2))$, centered at zero with radius $C(h+\delta/2)$ (see Figure \ref{fig:complete_compress}).
\medbreak

\noindent 4) Note that if we apply similar parameters, this time to compress $\mathcal{H}_2^2$, 
 from \Cref{fig:complete_compress}, the distance between hyperrectangles becomes $\delta/4$. Generally, the distance between hyperrectangles in step $j$ is $\delta/(2^{\lfloor j/2\rfloor})$.

\begin{figure}
    \centering
    \includegraphics[height=0.33\textwidth]{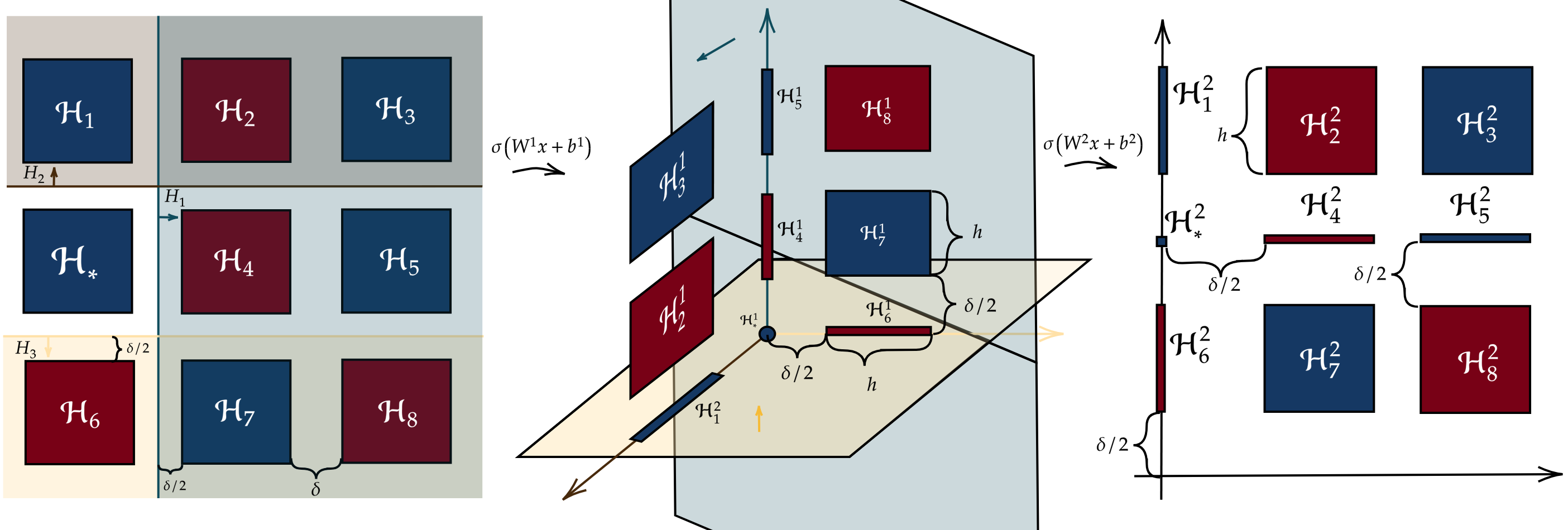}
    \caption{Illustration of the initial steps in the compression process. For a specific $2-$dimensional example, we show how the parameters $(W^1, b^1)$ and $(W^2, b^2)$ affect the hyperrectangles, reducing  distances. The first figure shows hyperrectangles separated by a $\delta$ distance. We choose hyperplanes with normal vectors $(0,1)$, $(0,-1)$, and $(1,0)$ to collapse the hypercube $\mathcal{H}^*$. This action maps the hyperrectangles into a $3-d$ space. Subsequently, using two hyperplanes with normal vectors $(0,0,1)$ and $(1,-1,0)$, we map the hyperrectangles onto a $2-d$ space, where the distance between the hyperrectangles is now $\delta/2$ close to zero, and the farthest hyperrectangle is at most at distance $C(h+\delta/2)$.}
    \label{fig:complete_compress}
\end{figure}

In the compression phase, we apply an iterative process where the parameters are selected based on the same criteria. Therefore, we conclude that
\begin{align*}
    \|b^j\|_{\infty} = \begin{cases}
        C(h+\delta/2) + m_d(\mathcal{C}) &\text{ if } j=1,\\
        C(h+\delta/(2^{\lfloor j/2\rfloor})) &\text{ otherwise,}
    \end{cases} \quad \|W^j\|_{\infty} = 1, \quad \forall j \in \llbracket1,2dN_e\rrbracket.
\end{align*}
Consequently, 
\begin{align}\label{eq:estimation_norm_phi_NE}
    \|\phi^{2N_E}(x)\| \leq  \|x\| + m_d(\mathcal{C}) + 2dN_E C\left( h+\frac{\delta}{2} \right).
\end{align}
\medbreak

\noindent \emph{\bf  Norm of $\phi^{L}$}. We have shown that in the compression process, the data is driven into a ball $B_0(C(h+\delta/2^{2N_E}))$, where $C$ is a constant depending on $m_d(\mathcal{C})$. Moreover, the distance between the points does not exceed $\delta/(2^{2N_E})$. Therefore, the resulting $N$ points from the compression process reside within that ball. Note that the map $\phi^L = (\phi_3^{L_3} \circ \phi_2^{L_2} \circ \phi_1^{L_1} \circ \phi_0^{L_0})$ corresponds to the map constructed in Section \ref{formal_proof}. Thus, starting from the output of $\phi^{2N_E}$, we analyze each map $\phi_i^{L_i}$.
\smallbreak

\noindent {1) Precondition of the data:} In this  phase, $b_1$ is chosen large enough such that  $\vecsigma$ acts as the identity function. Considering as an input data point the output of the map $\phi^{2N_E}$, then it is enough to take $b$ larger than $C(h+\delta/2^{2N_E})$. This implies:
\begin{align*}
    \|W_0\|_{\infty} = 1, \quad \|b_0\|_{\infty} \leq 2C(h+\delta/2^{2N_E}).
\end{align*}
 \smallbreak

\noindent {2) Compression process:} After data preconditioning, all the datasets have been projected to the real line, and the distance between points does not exceed $C\left(\delta/2^{2N_E}\right)$. We place the hyperplanes in the \emph{``Compression process"} depending on 
 the data location. Therefore, we deduce that
\begin{align*}
    \|W_j\|_{\infty} \leq 1, \quad \|b_j\|_{\infty} \leq C\left(\frac{\delta}{2^{2N_E}}\right), \quad \text{ for all } j \in \llbracket1,\dots,2N+1\rrbracket.
\end{align*}
  \smallbreak

\noindent {3) Data sorting:} In this step, we place the hyperplanes depending on the data location. Then, we obtain that 
\begin{align*}
    \|W_j\|_{\infty} \leq 1, \quad \|b_j\|_{\infty} \leq C\left(\frac{\delta}{2^{2N_E}}\right), \quad \text{ for all } j \in \llbracket2N+2,\dots,2N+2M+2\rrbracket.
\end{align*}
  \smallbreak

 \noindent {4) Mapping to the respective label:} In this step, we expand or contract the data to map them to their respective labels. In Theorem \ref{UAT_LP}, the labels are defined by the different values $\{f^h_i\}_i$ that the function $f_h$ takes. Then, we deduce that 
\begin{align*}
\|\phi^{L_3}\|_{L^\infty(\mathcal{C};\R_{+})}\leq\max\{\max_{i}\{f_i^h\},m_d(\mathcal{C})\}.
\end{align*}
Then, analogous to \eqref{eq:estimation_f_ih1}, we obtain that 
\begin{align*}
|f_{i}^h|\leq C(1+\|f\|_{L^p(\mathcal{H}_i;\R_{+})})\leq C(1+\|f\|_{L^p(\Omega;\R_{+})}).
\end{align*}
Consequently, there exists a constant $C > 0$ that only depends on $\|f\|_{L^p(\Omega; \mathbb{R}_+)}$ and $m_d(\mathcal{C})$, such that $\|\phi^{L_3}\|_{L^\infty(\mathcal{C};\R_+)} \leq C$.
  \medbreak
  
On the other hand, given $\hat{L}>0$ and a family of parameters $\mathcal{W}^{\hat{L}}=\{W^{i}\}_{i=1}^{\hat{L}}$ and $\mathcal{B}^{\hat{L}}=\{b^{i}\}_{i=1}^{\hat{L}}$, the norm of $\phi^{\hat{L}}:=\phi^{\hat{L}}(\mathcal{W}^{\hat{L}},\mathcal{B}^{\hat{L}},\cdot)$,  the input-output map of \eqref{discrete_dynamics}, can be bounded by
\begin{align}\label{eq:norm_phi_arbitrary}
    \|\phi^{\hat{L}}\|_{L^{\infty}(\mathcal{C};\R_{+})}\leq \underset{x\in\mathcal{C}}{\esssup}\left\| \prod_{j=1}^{\hat{L}}W^{j} x + \sum_{i=1}^{\hat{L}-1}\left(\prod_{j=i}^{\hat{L}-1}W^{j+1} \right) b^i + b^{\hat{L}}\right\|_{\infty}
\end{align}
Consequently, using the fact that in the compression process and data sorting, we are using $2N_h$ and $2M_h+1$ layers, respectively, and the estimation of the parameters norms, we can apply  \eqref{eq:norm_phi_arbitrary} to deduce that 
\begin{multline*}    \|\phi^{\mathcal{L}}\|_{L^\infty(\mathcal{C};\R_{+})}=\|\phi^{L}\circ \phi^{2N_E}\|\leq C\,\underset{x\in\mathcal{C}}{\esssup}\|\phi^{2N_E}(x)\|\\+ 2N_hC\left( \frac{\delta}{2^{2N_E}}\right) +(2M_h+1)C\left(\frac{\delta}{2^{2N_E}}\right) +2C\left(h+\frac{\delta}{2^{N_E}}\right). 
\end{multline*}
Then, using \eqref{eq:estimation_norm_phi_NE} we have
\begin{align}\label{eq:final_estimation_phi_nh}
\|\phi^{\mathcal{L}}\|_{L^\infty(\mathcal{C};\R_{+})}\leq 2C m_d(\mathcal{C}) + 2dN_E C\left( h+\frac{\delta}{2} \right)+ (2N_h+2M_h+1)C\left( \frac{\delta}{2^{2N_E}}\right) +2C\left(h+\frac{\delta}{2^{N_E}}\right).
\end{align}
Denote by $l_\mathcal{C}$ the largest edge of $\mathcal{C}$. Then, applying \eqref{eq:num_NE} in \eqref{eq:final_estimation_phi_nh} there exists a positive constant $C_1$, independent of $h$ and $\delta$, such that
\begin{align}\label{eq:estimation_inf_phi_1}
\|\phi^{\mathcal{L}}\|_{L^\infty(\mathcal{C};\R_{+})}&\leq C_1\left(1+\delta \frac{2^{\frac{-C_2}{h+\delta}}}{(h+\delta)^d}+\delta2^{\frac{-C_2}{h+\delta}}+h\right),
\end{align}
where $C_2=2l_{\mathcal{C}}$. Now, since $e^{y}\geq \sum_{k=0}^{d+1}y^k/k!$ using the change of variable $y=log(2)C_2/(h+\delta)$ we deduce that
\begin{align*}
    \frac{2^{\frac{-C_2}{h+\delta}}}{{(h+\delta)^{d}}}&\leq \frac{1}{(h+\delta)^{d}}\left( \sum_{k=0}^{d+1}\frac{(C_2\log(2))^k(h+\delta)^{-k}}{k!}\right)^{-1}\\
    &=\frac{1}{(h+\delta)^{d}}\left( \frac{1}{(h+\delta)^{d+1}}\sum_{k=0}^{d+1}\frac{(C_2\log(2))^k(h+\delta)^{(d+1)-k}}{k!}\right)^{-1}\\
    &=(h+\delta)\left(\frac{(C_2\log(2))^{d+1}}{(d+1)!}+\sum_{k=0}^{d}\frac{(C_2\log(2))^k(h+\delta)^{(d+1)-k}}{k!}\right)^{-1}.
\end{align*}
Therefore, for $h<\frac{l_{\mathcal{C}}\log(2)}{(d+1)}$ we have that 
\begin{align*}
    \frac{2^{\frac{-C_2}{h+\delta}}}{{(h+\delta)^{d}}}\leq \frac{(dp+1)!}{(2l_{\mathcal{C}}\log(2))^{d+1}}(h+\delta),
\end{align*}
Similarly, using the inequality $e^{y}\geq 1+y$ and the change of variable $y=log(2)C_2/(h+\delta)$, we deduce that for $h<l_{\mathcal{C}}\log(2)$ we have that $2^{\frac{-C_2}{h+\delta}}\leq (l_{\mathcal{C}}\log(2))^{-1}(h+\delta)$. Consequently, from \eqref{eq:estimation_inf_phi_1} we obtain the inequality 
\begin{align*}
\|\phi^{\mathcal{L}}\|_{L^\infty(\mathcal{C};\R_{+})}\leq C\left(1+\delta (h+\delta)+h\right),
\end{align*}
with $C>0$ a constant depending on $d$, $m_d(\mathcal{C})$, and $\|f\|_{L^p(\Omega;\R_{+})}$. This concludes the proof.
\end{proof}

\section*{Acknowledgments}
The authors wish to express their gratitude to D. Ruiz-Balet and A. Alcalde for their insightful discussions and to A.\'Alvarez-L\'opez and T. Crin-Barat, for taking the time to critically review our manuscript.
\medbreak
M. Hern\'{a}ndez has been funded by the Transregio 154 Project, Mathematical Modelling, Simulation, and Optimization Using the Example of Gas Networks of the DFG, project C07, and the fellowship "ANID-DAAD bilateral agreement". 

E. Zuazua was funded by the European Research Council (ERC) under the European Union's Horizon 2030 research and innovation programme (grant agreement NO: 101096251-CoDeFeL), the Alexander von Humboldt-Professorship program, the ModConFlex Marie Curie Action, HORIZON-MSCA-2021-dN-01, the COST Action MAT-DYNNET, the Transregio 154 Project of the DFG, AFOSR Proposal 24IOE027 and grants PID2020-112617GB-C22/AEI/10.13039/501100011033 and TED2021131390B-I00/ AEI/10.13039/501100011033 of MINECO (Spain), and Madrid GovernmentUAM Agreement for the Excellence of the University Research Staff in the context of the V PRICIT (Regional Programme of Research and Technological Innovation).

 Both authors have been partially supported by the DAAD/CAPES Programs for Project-Related Personal, grant 57703041 'Control and numerical analysis of complex system'.

\bibliographystyle{abbrv} 
\bibliography{Biblio_classification.bib}
\vfill

\end{document}